%% file: main.tex
\documentclass[11pt]{article}
\usepackage[left=1in,right=1in,top=1in,bottom=1in]{geometry}
\usepackage[numbers,sort&compress]{natbib}
\usepackage[utf8]{inputenc} 
\usepackage[T1]{fontenc}    
\usepackage{url}            
\usepackage{booktabs}       
\usepackage{amsfonts}       
\usepackage{nicefrac}       
\usepackage{microtype}
\usepackage{mathtools}
\usepackage{amsmath}
\usepackage{amsthm}
\usepackage{multirow}
\usepackage{algpseudocode}
\usepackage{amssymb}
\usepackage{graphicx}
\usepackage{subcaption}
\usepackage{tikz}
\usepackage{adjustbox}
\usetikzlibrary{calc}
\usepackage{pgfplots}
\usepackage{pgfplotstable}
\usepackage{filecontents}
\usepackage[draft]{hyperref}
\definecolor{shadecolor}{RGB}{190,190,190}

\usepgfplotslibrary{groupplots}
\usetikzlibrary{patterns}
\pgfkeys{/pgf/number format/.cd,1000 sep={\,}}
\pgfplotsset{
  ylabel right/.style={
    after end axis/.append code={
      \node [rotate=90, anchor=north] at (rel axis cs:1,0.5) {#1};
    }
  }
}
\usepackage{caption}
\usepackage{tabularx}
\usepackage{color}
\usepackage{colortbl}
\usepackage{algorithm}
\usepackage{algpseudocode}
\usepackage[capitalize]{cleveref}

\DeclareMathOperator*{\argmax}{argmax}
\DeclareMathOperator*{\argmin}{argmin}

\newtheorem{theorem}{Theorem}
\newtheorem{definition}{Definition}
\newtheorem{corollary}{Corollary}

\newcommand{\norm}[1]{\left\lVert#1\right\rVert}
\newcommand{\eqtext}[1]{\ensuremath{\stackrel{\text{#1}}{=}}}

\newcommand{\test}{\mathrm{test}}
\newcommand{\poisoning}{\mathrm{p}}
\newcommand{\clean}{\mathrm{c}}
\newcommand{\attack}{\mathrm{k}}

\newcommand{\xyen}{\{(x, y)\}^{\epsilon n}} 
\newcommand{\xyent}[1]{\{(x^{#1}, y^{#1})\}^{\epsilon n}} 

\newcommand{\D}{\mathcal{D}} 
\newcommand{\X}{\mathcal{X}} 
\newcommand{\Y}{\mathcal{Y}} 
\renewcommand{\L}{\mathcal{L}} 
\newcommand{\F}{\mathcal{F}} 

\newcommand{\R}{\mathbb{R}} 
\newcommand{\E}{\mathbb{E}} 

\title{On Adversarial Bias and the Robustness of Fair Machine Learning}

\author{
  Hongyan Chang, Ta Duy Nguyen, Sasi Kumar Murakonda, Ehsan Kazemi$^\dagger$, Reza Shokri  \\
  National University of Singapore (NUS), $^\dagger$Google\\
  \texttt{\{hongyan, taduy, murakond, reza\}@comp.nus.edu.sg, ehsankazemi@google.com}
}

\date{}

\begin{document}

\maketitle

\begin{abstract}
Optimizing prediction accuracy can come at the expense of fairness.  Towards minimizing discrimination against a group, fair machine learning algorithms strive to equalize the behavior of a model across different groups, by imposing a fairness constraint on models.  However, we show that giving the same importance to groups of different sizes and distributions, to counteract the effect of bias in training data, can be in conflict with robustness.  We analyze data poisoning attacks against group-based fair machine learning, with the focus on equalized odds.  An adversary who can control sampling or labeling for a fraction of training data, can reduce the test accuracy significantly beyond what he can achieve on unconstrained models.  Adversarial sampling and adversarial labeling attacks can also worsen the model's fairness gap on test data, even though the model satisfies the fairness constraint on training data.  We analyze the robustness of fair machine learning through an empirical evaluation of attacks on multiple algorithms and benchmark datasets.
\end{abstract}

\input{sections/introduction.tex}

\input{sections/problem.tex}
\input{sections/framework.tex}
\input{sections/algorithms.tex}
\input{sections/experiments.tex}

\input{sections/relatedwork.tex}
\input{sections/impact.tex}
\input{sections/conclusion.tex}
\input{sections/acknowledgements.tex}


\input{mainbbl.tex}
\newpage
\appendix
\input{sections/tab-notations.tex}

\newpage
\input{sections/appdx-theory.tex}
\input{sections/appdx-fair-algorithm.tex}
\input{sections/appdx-exp.tex}
\end{document}

%% file: sections/introduction.tex
\section{Introduction}
\label{sec:intro}

Trustworthy algorithms are crucial components of machine learning frameworks in critical systems, as highlighted by many AI regulations and policies as well as technical research papers.  Algorithmic fairness is at the core of the trust requirements for automated decision making in sensitive domains, to avoid systemic discrimination against protected groups.  Many technical notions of fairness are proposed and many algorithms for enforcing such notions are designed~\cite{agarwal2018reductions, calders2009building, dwork2012fairness, hardt2016equality, kamishima2011fairness, kusner2017counterfactual, madras2018learning, zafar2015fairness, zafar2017fairness, zemel2013learning}.  Group fairness measures, such as equalized odds~\cite{hardt2016equality} which is the focus of this paper, suggest equalizing the model's behavior across groups that are identified based on a protected attribute (e.g., race or gender).  Fairness, however, has a cost on the model's performance, as the best decision rules that satisfy a definition of fairness differ from the optimal decision rules~\cite{corbett2017algorithmic}.  In this paper, we ask \emph{how adversarially adding a fraction of the training data can further increase the cost of fairness}.  

A large body of work shows machine learning models are vulnerable to various types of data poisoning attacks that can impose a large test loss on the target models~\cite{biggio2012poisoning, burkard2017analysis, chen2017targeted, gu2017badnets, jagielski2018manipulating, koh2017understanding, li2016data, mei2015security, mei2015using, shafahi2018poison, steinhardt2017certified, suciu2018does}.  Recent work studies the performance of fair machine learning in the presence of noise over a fraction of the training data~\cite{blum2019recovering, calders2013unbiased, de2018learning, jiang2019identifying, kallus2018residual, lamy2019noise}.  These assume a \emph{uniform} distribution of under-representation and labeling bias, or the analysis is limited to having an unlimited number of training data from the underlying distribution.   We present a detailed survey of the related work in Section~\ref{sec:relatedwork}.  To the best of our knowledge, this paper provides the first analysis of the robustness of fair machine learning algorithms in the adversarial setting, against data poisoning attacks.  This paper shows the implications of adversarial bias on fair machine learning, and calls for \emph{robust algorithmic fairness}.

We present a framework for designing data poisoning attack algorithms against models trained with equalized odds as the fairness constraint.  In our algorithms, we assume the attacker who can control the sampling process and (in the stronger case, also) the labeling process for some of the training data.  Our attacks effectively exploit the fact that fair algorithms equalize the importance of different groups (with different sizes and data distributions). This can change the influence of individual data samples in those groups in a disproportionate way, enabling the attacker to place poisoning data where it can impose a large loss on the trained model.  We extensively evaluate the robustness of fair machine learning on multiple fairness algorithms and benchmark datasets.  Here are the key findings in our empirical evaluation:

We show that there is a \emph{significant conflict between fairness and robustness}.  As we tighten the guaranteed fairness gap, we show that the susceptibility of fair models to data poisoning attacks increases.  Notably, enforcing exact fairness results in the largest drop in test accuracy under attack, much beyond what an adversary can achieve in unconstrained models.  We can observe this effect even for the case of the most limited adversary who can only control data sampling for a small fraction of the training data, without being able to change the features and labels.  We observe that the adversary achieves this by placing the poisoning data in the smallest group with the least frequent label.  To satisfy the fairness constraint, the model ends up sacrificing its generalizability over the majority group to equalize its prediction performance. 

The impact of our data poisoning attacks is not limited to reducing the test accuracy.  We show that the attack also results in a significant loss of fairness over test data.  Adversarial manipulation of the training data prevents the model to generalize its fairness to clean test data, even though it \emph{is} guaranteed on training data.  The attacker can influence the models to become even more discriminatory than unconstrained models, according to the fairness measure.

%% file: sections/problem.tex

\section{Background and the Problem Statement}
\label{sec:backgrd}
\label{backgrd:problem_setting}
\label{backgrd:fairness_definition}
\label{backgrd:threat_model}

{\bf Machine learning.} Consider a binary classification model ${f_\theta: \X \to \Y}$, that maps the feature space $\X$ to binary labels ${\Y=\{+,-\}}$.  The model is represented by its parameters $\theta$ taken from a parameter space $\Theta$.  The model is trained to minimize a loss function ${\ell: \Theta \times \X \times \Y \to \R_+}$ over its training set $\D$.  We let $X$ and $Y$ to denote the random variables associated with the features and the labels, and $(X, Y)$ to denote the underlying distribution of the data.  We obtain the optimal parameters as ${\hat{\theta} = \argmin_{\theta \in \Theta} \frac{1}{|\D|} \L(\theta;\D)}$, where ${\L(\theta;\D) = \sum_{(x, y) \in \D} \ell(\theta; x, y)}$ is the cumulative loss of the model over the training set.  We quantify the accuracy of the trained model on a test dataset $\D_{\test}$.

{\bf Fairness.} We assume all data points are split into two groups based on a binary attribute $S \in \{0, 1\}$ (e.g., gender), referred to as the protected/sensitive attribute. This attribute could be part of the feature set $\X$.  In fair machine learning, our objective is to train a model such that its predictions are non-discriminatory and fair with respect to $S$.  To this end, the training process needs to be adjusted to equalize the prediction behavior of the model across the two groups~\cite{dwork2012fairness, hardt2016equality, keinberg2017inherent, zafar2017fairness, dwork2018decoupled}.  In this work, we focus on {\em equalized odds}, which is a widely-used definition for group fairness~\cite{hardt2016equality}.  A model is fair if, given the true label for a data point, the model's prediction on a data point and its sensitive attribute are conditionally independent.  We use a relaxed notion of equalized odds.
\begin{definition}[Equalized odds] \label{def:eo}
	A binary classifier $f_\theta$ is $\delta$-fair under equalized odds if
	\begin{align} \label{eq:eo-fairness}
		\Delta(\theta, \D) \coloneqq \max_{y\in\{0,1\}}|\Pr_\D[f_\theta(X)\neq y|S=0, Y=y] - \Pr_\D[f_\theta(X)\neq y|S=1, Y=y] | \le \delta ,
	\end{align}
	where, the probabilities are computed empirically over the training data set $\D$.  We refer to $\Delta$ as the model's empirical {\bf fairness gap}.  A model satisfies exact fairness when $\delta=0$.
\end{definition}

Fairness is achieved by ensuring $\delta$-fairness empirically on the model's training set, e.g., through minimizing the model's empirical loss under $\delta$-fairness as a constraint~\cite{agarwal2018reductions} or post-processing~\cite{hardt2016equality}.  We define the constraint ${C(\theta, \D) \coloneqq  \Delta(\theta, \D) - \delta \le 0}$ as a {\bf fairness constraint} for the model.

{\bf Data poisoning.} An adversary might be able to contaminate the training dataset in order to degrade the test accuracy of a classification model.  In this setting, we assume the training set is composed of the clean dataset $\D_\clean$ of size $n$, and the poisoning dataset $\D_\poisoning$ of size $\epsilon n$, which is contributed by the attacker.
The level of contamination is determined by $\epsilon$ (the ratio of the size of the poisoning data over the clean data in the training set).  The attacker's objective is to maximize the loss of the classifier over the data distribution (evaluated using the test dataset). This objective can be stated as a bi-level optimization problem subject to the fairness constraint $C(\theta, \D) \leq 0$ :
\begin{align} \label{eq:poisoning_attack_fair}
	\max_{\D_\poisoning}  \E_{(X, Y)}[\ell(\hat{\theta}; X, Y)], \mbox{ where } \hat{\theta} \coloneqq \argmin_{\theta \in \Theta} \frac{\mathcal{L}(\theta;\D_\clean \cup \D_\poisoning)}{|\D_\clean \cup \D_\poisoning|}, \mbox{ such that } C(\theta, \D_\clean \cup \D_\poisoning) \le 0,
\end{align}
where the expectation is taken over the underlying distribution of the (clean) data.

{\bf Problem statement.} The primary research question that we investigate in this paper is whether, how, and why training a model with (equalized odds) fairness can compromise its robustness to data poisoning attacks, compared with unconstrained models (without any fairness constraint).

We assume that the attacker has access to $\D_\clean$, and knows the learning task, the structure of the classification model, the learning hyper-parameters, and the fairness constraint on the target model.  A strong threat model involves the attacker that can craft any arbitrary poisoning data $\D_\poisoning$.  In this paper, however, we focus on a more restricted yet more realistic attack scenario, where the attacker is restricted to select the feature vector of the poisoning data from an attack dataset $\D_\attack$, which is sampled from the same  underlying distribution of the clean dataset.  Two variations of the attack are {\bf adversarial~sampling}, where $\D_\poisoning \subset \D_\attack$, and {\bf adversarial~labeling}, where ${\D_\poisoning \subset \D_\attack \cup \{ (x, 1-y): (x, y) \in \D_\attack\}}$.  Adversarial labeling is a more powerful attack as the attacker is also allowed to craft the labels (via label flipping) for generating the poisoning data.

The rationale for considering adversarial sampling and adversarial labeling attacks is that, in a realistic scenario, the attacker might not be able to craft feature vectors (i.e., generate fake loan applications).  However, he could be part of a system that can introduce a bias in the sampling process for the training data, by ignoring some samples and including the others.  In addition to this, the attacker could also be capable of influencing the decision making for some data points (i.e., producing labels) which later will be used as part of the target model's training set.  Our objective is to design these forms of poisoning attacks that introduce an {\bf adversarial bias} into the training set of fair models.

%% file: sections/framework.tex

\section{Optimization Framework for Adversarial Bias}
\label{sec:attack}
\label{sec:framework}

The bi-level optimization problem~\eqref{eq:poisoning_attack_fair} is non-convex and intractable~\cite{bard1991some, hansen1992new, deng1998complexity}. The fairness constraint in the inner optimization makes the problem even more difficult.\footnote{We would like to point out that, for the unconstrained model, under the convex assumption of the loss function, it is possible for the attacker to find the approximate solution  by replacing the inner optimization with its stationarity (KKT) condition~\cite{biggio2012poisoning, koh2017understanding, mei2015using}.}
In this section, we present a number of approximations for problem~\eqref{eq:poisoning_attack_fair} which enables us to design effective attack algorithms.

We first approximate the loss function (which is maximized in the outer optimization of \cref{eq:poisoning_attack_fair}) by following the same techniques used for designing poisoning attacks against unconstrained models~\cite{steinhardt2017certified}.  For this reason, let $\hat{\theta}$ be the solution to the inner optimization in \eqref{eq:poisoning_attack_fair}.  We use the loss on the clean training data as an approximation for the loss over the underlying data distribution (of test data).
\begin{align} \label{eq:approximations}
	\E_{(X, Y)}[\ell(\hat{\theta}; X, Y)] \approx \frac{\L(\hat{\theta};\D_\test)}{|\D_\test|} \approx \frac{\L(\hat{\theta};\D_\clean)}{|\D_\clean|} \leq \frac{\L(\hat{\theta};\D_\clean \cup \D_\poisoning)}{|\D_\clean|}
\end{align}
The inequality provides a valid upper bound, as the loss function $\ell$ is non-negative.  Note that this bound becomes tighter if the fair model $f_\theta$ fits the poisoned training dataset well.
With the same line of reasoning, we replace the objective of the inner minimization in \cref{eq:poisoning_attack_fair} with $\nicefrac{\mathcal{L}(\theta;\D_\clean \cup \D_\poisoning)}{n}$, where $n$ is the size of the clean training dataset $\D_\clean$.

The fairness constraint in the inner optimization of \cref{eq:poisoning_attack_fair} makes it hard to track the influence of $\D_\poisoning$ on the training loss.
We use a Lagrange multiplier $\lambda \in \R_+$ to replace the constraint for the inner optimization problem with a penalizing term:
\begin{align} \label{eq:lagrange_dual_problem}
	 \min_{\theta \in \Theta} \left[ \frac{\mathcal{L}(  \theta;\D)}{n}, \mbox{ s.t. } C(\theta, \D) \le 0 \right]
	  & = \min_{\theta \in \Theta}   \max_{\lambda \in \R_+}  \left( \frac{\L(\theta;\D)}{n}  +\lambda C(\theta, \D)\right)  \nonumber \\
&	 \geq  \max_{\lambda \in \R_+}  \min_{\theta \in \Theta}  \left( \frac{\L(\theta;\D)}{n}  +\lambda C(\theta, \D)\right)  \enspace,
\end{align}
where ${\D = \D_\clean \cup \D_\poisoning}$. The last inequality in \cref{eq:lagrange_dual_problem} follows from the weak duality theorem~\cite{boyd2004convex}.

Based on the dual problem, the attacker can try to find a poisoning dataset $\D_\poisoning$ by maximizing a lower bound provided by the Lagrangian function $\min_{\theta \in \Theta}  \left( \nicefrac{\L(\theta;\D)}{n}  +\lambda C(\theta, \D)\right)$ for a fixed $\lambda \in \R_+$.  Indeed, maximizing the lower bound provided by the Lagrangian function would result in a solution with a high loss (which is guaranteed to be at least equal to the loss for the lower bound) for the original problem.

In this optimization procedure, we can also replace the fairness constraint $C(\theta, \D) \coloneqq  \Delta(\theta, \D) -  \delta$ with the fairness gap $\Delta(\theta, \D)$, because the constant value $\delta \geq 0$ does not affect the solution for the Lagrangian.
Finally, by considering all the above-mentioned steps, the new attacker's objective is:
\begin{align} \label{eq:simplified_dual}
	\max_{\D_\poisoning}  \min_{\theta \in \Theta} \left( \frac{\L(\theta; \D_\clean \cup \D_\poisoning)}{n} +  \lambda \Delta(\theta, \D_\clean \cup \D_\poisoning) \right) .
\end{align}
Thus, the goal is to find a poisoning dataset that maximizes a linear combination of the training loss and the model's violation from the fairness constraint, where $\lambda$ controls the penalty for the violation.

%% file: sections/algorithms.tex

\section{Attack Algorithms}
\label{sec:algorithms}

In \cref{sec:framework}, we explained how the objective of an optimal attacker could be interpreted as  \eqref{eq:simplified_dual}.
Towards solving \eqref{eq:simplified_dual}, the attacker needs to overcome a number of subtle challenges: the loss function and the constraint are non-convex functions of the model parameter $\theta \in \Theta$, and the fairness gap is not an additive function of the training data points $(x,y) \in \D$.  These two keep little hope to solve \cref{eq:simplified_dual} without further assumptions.
To overcome these issues, in ~\cref{sec:certified}, we find an approximation for the fairness gap which is additive in the training data points.
By using this additivity property,  we design an online algorithm for the data poisoning attack.
We further prove the optimality of our algorithm for this modified objective under some reasonable conditions.
In \cref{sec:regular-attack}, we present another variant of such online algorithms, which under the Lagrange multiplier $\lambda = 0$ is equivalent to the prior poisoning attack algorithms against unconstrained models~\cite{steinhardt2017certified, koh2018stronger}.

\subsection{An Approximation for the Fairness Gap}
\label{sec:certified}

To design an additive approximation of the fairness gap $\Delta$, we consider the contribution of each training data point to the fairness gap independently.
Let $\{(x, y)\}^{k}$ be the $k$ repetition of a single data point $(x, y)$. Thus $\D  \cup  \{(x, y)\}^{k}$ is equivalent to adding $k$ copies of the data point $(x, y)$ to set $\D$.
In this setting, for any data point $(x,y) \in \D_\poisoning$, ${\Delta\left(\theta,  \D_\clean  \cup  \{(x, y)\}^{\epsilon n} \right)}$ is a proxy for measuring the contribution of that data point to the fairness gap ${\Delta(\theta,  \D_\clean  \cup  \D_\poisoning)}$, and its maximum over all data points in $ \D_\poisoning$ provides an upper bound on the fairness gap of the model when the size of the poisoning set is $\epsilon n$.  Thus, we get an approximation for the fairness gap as follows:
\begin{align}\label{eq:contr_proxy}
	\Delta\left(\theta, \D_\clean  \cup \D_\poisoning \right) \approx  \frac{1}{\epsilon n}  \sum_{(x,y) \in \D_\poisoning}  \Delta\left(\theta, \D_\clean  \cup \xyen \right) .
\end{align}

By substituting the fairness gap with its surrogate function, the objective of the attacker is to solve:
\begin{align} \label{eq:surrogate-optimization}
	M^* = \max_{\D_\poisoning}  \,\, \underbrace{\min_{\theta \in \Theta}
	\left(	\frac{1}{n} \L \left( \theta ; \D_{\clean} \cup \D_\poisoning \right)  + \frac{\lambda}{\epsilon n} \cdot \sum_{(x,y) \in \D_\poisoning} \Delta \left( \theta, \D_\clean  \cup \xyen \right) \right)}_{M(\D_\poisoning)} ,
\end{align}
where $M(\D_\poisoning)$ is the  loss which is incurred  by any poisoning dataset $\D_\poisoning$ on the fair model, and $M^*$ is the maximum loss under the optimal attack.

Algorithm~\ref{algorithm:OGD},  a variant of the online gradient descent methods~\cite{hazan2016introduction}, is our solution to the problem \eqref{eq:surrogate-optimization}.
 It initializes a model $\theta^0 \in \Theta$, and identifies $\epsilon n$ poisoning data points iteratively. The feasible set of poisoning points $\F(\D_\attack)$ is determined by the capabilities of the attacker.  For adversarial sampling attacks, we have ${\F(\D_\attack) = \D_\attack}$, and for adversarial labeling attacks, we have ${\F(\D_\attack) = \D_\attack \cup \{ (x, 1-y): (x, y) \in \D_\attack\}}$.  The algorithm iteratively performs the following steps:

{\bf Data point selection.} (Algorithm~\ref{algorithm:OGD}, line~\ref{line:OGD:pick}): It selects a data point with the highest impact on a weighted sum of the loss function and the fairness gap with respect to the model parameter $\theta^{t-1}$.

{\bf Parameter update.} (Algorithm~\ref{algorithm:OGD}, line~\ref{line:OGD-update}): The parameters are updated to minimize the penalized loss function based on the selected data point $(x^t, y^t)$.  In this way, the algorithm (through the approximations made by the Lagrange multiplier and the surrogate function) keeps track of the fair model under the set of already selected poisoning data points by the attack.

\begin{algorithm}[t]
	\caption{Online Gradient Descent Algorithm for Generating Poisoning Data for Fair Models}
	\begin{algorithmic}[1]
		\State {\bfseries Input:} Clean data $\D_\clean$, $n = |\D_\clean|$, feasible poisoning set $\F(\D_\attack)$, number of poisoning data $\epsilon n$, penalty parameter (Lagrange multiplier) $\lambda$, learning rate $\eta$.
		\State {\bfseries Output:} Poisoning dataset $\D_\poisoning$.
		\State Initialize $\theta^0 \in \Theta$
		\For{$ t = 1,\cdots,\epsilon n$}
		\State  $ (x^t, y^t) \gets \argmax_{(x,y) \in \F(\D_\attack)}
		\left[ \epsilon \cdot \ell (\theta^{t-1}; x,y) + \lambda \cdot \Delta\left(\theta^{t-1}, \D_\clean \cup \xyen \right)  \right] $ \label{line:OGD:pick}
		\State $\D_\poisoning \leftarrow \D_\poisoning \cup \{(x^t, y^t)\}$
		\State $\theta^t \leftarrow \theta^{t-1} -
		\eta \left( \frac{\nabla \L (\theta^{t-1}; \D_\clean )}{n}   +
		\nabla \left[ \epsilon \cdot \ell (\theta^{t-1};  x^t,y^t)
		+\lambda \cdot \Delta(\theta^{t-1}, \D_\clean  \cup \xyent{t}   ) \right] \right)$
		\label{line:OGD-update} 
		\EndFor
	\end{algorithmic}\label{algorithm:OGD}
\end{algorithm}

In \cref{theorem:bound-regret}, by following the approach proposed by \cite{steinhardt2017certified}, we relate the performance of Algorithm~\ref{algorithm:OGD} with the loss of the optimal attack for  \cref{eq:surrogate-optimization}.
Moreover, in \cref{sec:regret-approximation}, we prove that under some reasonable conditions (e.g., by using similar assumptions made by \cite{donini2018empirical} to approximate the fairness gap), our algorithm finds the (nearly) optimal solution for \cref{eq:surrogate-optimization}.

\begin{theorem} \label{theorem:bound-regret}
	Let $ \D^*_\poisoning$ be the data poisoning set produced by Algorithm~\ref{algorithm:OGD}.  Let $\mathrm{Regret}(\epsilon n)$ be the regret of this online learning algorithm after $\epsilon n$ steps.
	The performance  of the algorithm is guaranteed by
	\begin{align}
		M^* - M(\D^*_\poisoning)
		\leq \frac{\mathrm{Regret}(\epsilon n )}{\epsilon n } ,
	\end{align}
	where, $M^*$ and $M(\D^*_\poisoning)$ represent the loss of the fair model under the optimal data poisoning attack and the poisoning set $\D^*_\poisoning$, respectively.
\end{theorem}

\subsection{A Surrogate Function for the Target Model}
\label{sec:regular-attack}
In this section, we present a second algorithm that differs from Algorithm~\ref{algorithm:OGD} in its parameter update step, but is the same in its data point selection approach.
Indeed, in Algorithm~\ref{algorithm:OGD}, the parameter update step provides an approximation (through adding the fairness constraint as a penalizing term and approximating the fairness gap) for the target fair model.
 An alternative strategy, for an attacker, could be iteratively adding data points that maximize a combination of the loss and the fairness gap, however, over the parameters of the unconstrained model. In this case, $\theta^t$ would represent the parameters of an unconstrained model, without considering the fairness constraint.  Thus we update the parameters of the model in the direction of decreasing the loss for the unconstrained model:
\begin{align} \label{eq:update-alg2}
	\theta^t \leftarrow \theta^{t-1} - \eta \left( \frac{\nabla \L (\theta^{t-1}; \D_\clean )}{n}  + \epsilon \cdot  \nabla \ell (\theta^{t-1};  x^{t},y^{t}) \right) .
\end{align}

The pseudo-code of this algorithm is presented in Algorithm~\ref{algorithm:minmax-fair} in \cref{appendix-alg2}.
An intuitive explanation of Algorithm~\ref{algorithm:minmax-fair} is as follows: the attacker, from the result of the parameter update step~\eqref{eq:update-alg2}, would be able to estimate the unconstrained model (as a surrogate for the target fair model) over the set of currently selected points.  Then, a point with the largest value for a weighted sum of the loss and the fairness gap for the estimated surrogate model, potentially has a large impact on reducing the accuracy of the fair model.  We should point out that Algorithm~\ref{algorithm:minmax-fair} reduces the chance of getting stuck in local minima because, in each parameter update, it makes a step towards the negative gradient of the exact unconstrained loss.
This is in contrast with Algorithm~\ref{algorithm:OGD}, where due to the difficulty of approximating a constrained max-min problem, it might converge to some parameters not so close to the actual fair model or even not converge.

Note that, in our algorithms, if we set $\lambda = 0$, the adversarial bias attacks and their objectives are similar to the data poisoning attacks against the unconstrained models, e.g., the work of \citet{steinhardt2017certified}.  However, without taking into account the influence of the poisoning data on the fairness gap (as we do in the data point selection step), the attacker would not be as effective.  In Section~\ref{sec:eval}, we empirically investigate to what extend considering both the loss and fairness gap in designing attacks affects the accuracy of a fair model (which is trained over the union of clean data and poisoning data).

%% file: sections/experiments.tex

\section{Evaluation}\label{sec:eval}

In this section, we present the main findings of our experiments.  See details and results in Appendix~\ref{appdx:experiments}. 

\subsection{Evaluation Setup} \label{sec:eval:setup}

{\bf Datasets and models.} We train logistic regression models on the COMPAS dataset~\cite{compasdataset} and the Adult dataset~\cite{adultdataset}, which are benchmark datasets in the fairness literature.  We use race (white/black) in the COMPAS dataset, and gender in the Adult dataset, as the sensitive attribute $S$ (which is part of the feature vector $\X$).  The accuracy of classification models on these two datasets is low and close to predicting the most frequent label in the set.  This does not help understanding the behavior of models in the presence of poisoning data.  Hence, we perform a data pre-processing to separate {\em hard examples} from the data that we use for the clean training data $\D_\clean$, test data $\D_\test$, and the attack dataset $\D_\attack$.  Hard examples are data points with large loss on a trained model on the entire dataset.  We will use hard examples as one of our baselines.  We also add hard examples to the attack dataset.  

{\bf Fair machine learning algorithms.}  We train logistic regression models with an \emph{equalized odds} fairness constraint, by using the post-processing approach (the original exact equalized odds algorithm)~\cite{hardt2016equality} and the reductions approach (the relaxed equalized odds algorithm)~\cite{agarwal2018reductions}.  See Appendix~\ref{appdx:fair_impl}.

{\bf Adversarial bias.} Attacker adds poisoning data selected from the attack dataset $\D_\attack$ using Algorithm~\ref{algorithm:OGD} (with $\lambda =\epsilon$ on COMPAS, and $\lambda =0.1\epsilon$ on Adult), and Algorithm~\ref{algorithm:minmax-fair} (with $\lambda =100\epsilon$ for both datasets).  We use Algorithm~\ref{algorithm:minmax-fair} with $\lambda=0$ to attack unconstrained models. We use the same learning rate $\eta = 0.001$ in both algorithms.  See Appendix~\ref{sec:lambda-effect} for a discussion on choosing $\lambda$. 

{\bf Baseline algorithms.} In addition to comparing with prior data poisoning attacks against unconstrained models~\cite{steinhardt2017certified}, we also consider the following baselines.  \textit{Random sampling}: Attacker randomly selects data points from $\D_\attack$.  \textit{Label flipping}:  Attacker randomly selects data points from $\D_\attack$ and flips their labels.  \textit{Hard examples}: Attacker randomly selects data points from the set of hard examples.  

\input{plot/accuracy_COMPAS_paper.tex}

\input{plot/fairness-measure_COMPAS_sampling.tex}

\subsection{Evaluation Results} \label{sec:eval:results}

In this section, we present the experimental results on the COMPAS dataset.  We run each experiment $100$ times, with randomizing the datasets and random seeds in the algorithms, and report the average and standard deviation values.  See Appendix~\ref{appdx:experiments} for the full results on COMPAS dataset, as well as the same evaluations on the Adult dataset (in which we observe the same patterns). 

{\bf Conflict between fairness and robustness.} Figure~\ref{fig:accuracy-compas-paper} compares the test accuracy of unconstrained models and fair models under data poisoning attacks.  We attack the unconstrained models using Algorithm~\ref{algorithm:minmax-fair} with $\lambda=0$ (no fairness constraint), which is equivalent to the optimal attack~\cite{steinhardt2017certified}.  We attack the fair models using Algorithm~\ref{algorithm:minmax-fair} with $\lambda=100 \epsilon$, and Algorithm~\ref{algorithm:OGD} with $\lambda=\epsilon$.   The baselines, e.g., adding randomly selected hard examples, do not have much effect on test accuracy.  However, under attacks with the same capability, the \emph{fair models are noticeably less robust than unconstrained models}.  At $\epsilon =0.2$, the test accuracy of fair models approaches what can be achieved even by a constant classifier.  The significant outcomes can be evidently observed in the plots on adversarial sampling, where the adversary cannot change data labels (so effectively he is adding clean data, but in an adversarially biased manner). 

In the benign setting ($\epsilon = 0$), the reduction approach~\cite{agarwal2018reductions} (relaxed equalized odds with $\delta=0.01$) has a visibly better test accuracy compared to post-processing approach~\cite{hardt2016equality} (exact equalized odds, $\delta=0$), which reflects the cost of fairness on model accuracy. Figure~\ref{fig:fairness-measure-compas-sampling}(a) shows that this cost is significantly amplified for fair models under attack, as we train models with larger levels of fairness (i.e., smaller~$\delta$).  Notably, comparing $\delta=0.1$ (weaker fairness) with $\delta=0.01$ (stronger fairness) on the reduction approach~\cite{agarwal2018reductions}, as $\epsilon$ increases, shows that robustness against adversarial bias decreases as we increase the enforced fairness level. Thus, \emph{robustness and fairness are in conflict}.

{\bf Implication of adversarial bias for majority vs. minority groups.} Figures~\ref{fig:fairness-measure-compas-sampling}(b)~and~\ref{fig:fairness-measure-compas-sampling}(c) show the test accuracy for the majority and minority groups.  We observe that the impact of adversarial bias is not homogeneous across different groups in test data.  To understand the implications of the attack, observe the relation between test accuracy of two groups on the unconstrained model, and then compare it with the same in fair models (all under the same poisoning data).  We observe that, on fair models the \emph{attack is significantly more impactful on the majority group}.  However, on the unconstrained model, the minority group is the one that incurs a larger loss.  

We also compute the fairness gap on the unconstrained model with respect to the training data (see Figure~\ref{fig:EOgap} for all results).  On training data poisoned with Algorithm~\ref{algorithm:minmax-fair} ($\lambda=100\epsilon$) at $\epsilon = 0.1$, the fairness gap is $0.54$, which is much larger than that of the random sampling baseline with $0.3$ fairness gap.  This indicates that the data poisoning attack for \emph{adversarial bias increases the fairness gap}.  This explains the underlying strategy of the attacker against fair machine learning: to increase fairness gap and distort the data distribution of mainly the minority group in order to force the fair algorithm to lower the accuracy over the whole distribution when trying to equalize its behavior across the groups. 

{\bf Distribution and importance of poisoning data.} We investigate where exactly the poisoning data are placed under attacks designed for adversarial bias.  We observe that the data samples generated by our most effective attack, Algorithm~\ref{algorithm:minmax-fair} with $\lambda = 100\epsilon$, mostly belong to the smallest subgroup, i.e., the smallest sensitive group $s$ with the least frequent label $y$ (in this case it is $+$ label with race as Black).  Thus, \emph{the attack algorithms effectively exploit the fact that fair models give a higher weight to points from the under-represented areas of the distribution to satisfy the constraints}.  See Figure~\ref{fig:subgroups} for the distribution of poisoning points in all experiments.  We also compute the accuracy of trained models on their poisoning data.  The accuracy of fair models with $\delta = 0.01$ on the poisoning data from Algorithm~\ref{algorithm:minmax-fair} is approximately $0.2$ at $\epsilon = 0.1$, which increases to almost $0.4$ at $\epsilon = 0.2$. Whereas, the accuracy of unconstrained models on the poisoning data remains more or less zero, implying that the unconstrained models ignore these points.  See Figure~\ref{fig:poisoning-acc} for the accuracy of models on their clean and poisoning data.  We observe that by increasing~$\epsilon$, the accuracy of fair models on their poisoning data $\D_\poisoning$ increases, whereas it decreases on their clean training data $\D_\clean$.  This reduces their ability to learn the underlying data distribution and generalize to the (clean) test data.  

\input{sections/tab-fairness-paper.tex}

{\bf Fairness gap on test data.} The ultimate objective of fair machine learning is to extend fairness to test data.  Table~\ref{table:robustness-fairness-paper} shows how adversarial bias can jeopardize the fairness generalizability of fair models. We observe that, for $\lambda>0$, the lower the fairness level $\delta$ on training data is, the higher the fairness gap on test data becomes.  Note that e.g., for a fair model with $\delta = 0.01$ under adversarial sampling attack using Algorithm~\ref{algorithm:minmax-fair} with $\lambda =100\epsilon$, the fairness gap on the test data is about $0.37$.  This fairness gap is even larger than the fairness gap of an unconstrained model ($0.21$ in the benign setting and $0.26$ under data poisoning).  This shows that even by just controlling the sampling process for a small fraction of the training set, without affecting the labels, \emph{attacker can influence models trained with fairness constraints, to become more discriminatory than unconstrained models}.

%% file: plot/accuracy_COMPAS_paper.tex

\begin{figure*}[t!]
 \centering
 \resizebox{\columnwidth}{!}{%
{ \large
\begin{tabular}{l}
      \begin{tikzpicture}
     \begin{axis}
       [name=plot11,title= {Unconstrained Model - Adv. Sampling},xlabel={$\epsilon$ },ylabel={Test Accuracy}, ymin = 0.5, ymax =1, xtick={0,0.05,0.1,0.15,0.2}, xticklabels={0,0.05,0.1,0.15,0.2},grid = major]
       		\addplot[solid, blue, mark = *, mark options={solid,}, error bars/.cd, y dir=both, y explicit] table[skip first n=1,x index=1, y index=2, y error index=12, col sep=comma] {"plot2/reduction_csv/attacker-nfminmax-0.01.csv"};
          	\addplot[solid, gray, mark = triangle, mark options={solid,}] table[skip first n=1,x index=1, y index=2, col sep=comma] {"plot2/reduction_csv/attacker-usel-0.01.csv"};
          	\addplot[solid, brown, mark = diamond, mark options={solid,}] table[skip first n=1,x index=1, y index=2, col sep=comma] {"plot2/reduction_csv/noise-usel-0.01.csv"};
			\addplot[no marks] table[skip first n=1,x index=1, y index=2, col sep=comma] {"plot2/baseline.csv"};
     \end{axis}

      \begin{axis}
        [name=plot12, at=(plot11.south east), anchor=south west, xshift=0.8cm,title= {Fair~\cite{hardt2016equality} ($\delta = 0$) - Adv. Sampling},xlabel={$\epsilon$},ylabel={}, ymin = 0.5, ymax =1,xtick={0,0.05,0.1,0.15,0.2}, xticklabels={0,0.05,0.1,0.15,0.2},  grid = major, legend entries = {},legend style={at={(0.5, -0.15)}, anchor=north, draw=none, legend columns=6}]

         \addplot[solid, blue, mark = *, mark options={solid,}] table[skip first n=1,x index=1, y index=3, col sep=comma] {"plot2/PP_csv/attacker-nfminmax-0.01.csv"};

       \addplot[violet, mark = o, error bars/.cd, y dir=both, y explicit] table[skip first n=1,x index=1, y index=3, y error index=13, col sep=comma] {"plot2/PP_csv/attacker-eominmaxlossnf100-0.01.csv"};
        \addplot[teal, mark = x, error bars/.cd, y dir=both, y explicit] table[skip first n=1,x index=1, y index=3, y error index=13, col sep=comma] {"plot2/PP_csv/attacker-newregnf-1.0-0.01.csv"};

        \addplot[solid, gray, mark = triangle, mark options={solid,}] table[skip first n=1,x index=1, y index=3, col sep=comma] {"plot2/PP_csv/attacker-usel-0.01.csv"};
         \addplot[solid, brown, mark = diamond, mark options={solid,}] table[skip first n=1,x index=1, y index=3, col sep=comma] {"plot2/PP_csv/noise-usel-0.01.csv"};

		\addplot[no marks] table[skip first n=1,x index=1, y index=2, col sep=comma] {"plot2/baseline.csv"};

      \end{axis}

      \begin{axis}
        [name=plot13,
    at=(plot12.south east), anchor=south west, xshift=0.8cm,title= {Fair~\cite{agarwal2018reductions} ($\delta = 0.01$) - Adv. Sampling },xlabel={$\epsilon$},ylabel={}, ymin = 0.5, ymax =1,xtick={0,0.05,0.1,0.15,0.2}, xticklabels={0,0.05,0.1,0.15,0.2},  grid = major]

         \addplot[solid, blue, mark = *, mark options={solid,}] table[skip first n=1,x index=1, y index=3, col sep=comma] {"plot2/reduction_csv/attacker-nfminmax-0.01.csv"};

        \addplot[violet, mark = o, error bars/.cd, y dir=both, y explicit] table[skip first n=1,x index=1, y index=3, y error index=13, col sep=comma] {"plot2/reduction_csv/attacker-eominmaxlossnf100-0.01.csv"};
         \addplot[teal, mark = x, error bars/.cd, y dir=both, y explicit] table[skip first n=1,x index=1, y index=3, y error index=13, col sep=comma] {"plot2/reduction_csv/attacker-newregnf-1.0-0.01.csv"};

        \addplot[solid, gray, mark = triangle, mark options={solid,}] table[skip first n=1,x index=1, y index=3, col sep=comma] {"plot2/reduction_csv/attacker-usel-0.01.csv"};
         \addplot[solid, brown, mark = diamond, mark options={solid,}] table[skip first n=1,x index=1, y index=3, col sep=comma] {"plot2/reduction_csv/noise-usel-0.01.csv"};

		\addplot[no marks] table[skip first n=1,x index=1, y index=2, col sep=comma] {"plot2/baseline.csv"};

      \end{axis}
      \end{tikzpicture} \\
      \begin{tikzpicture}
     \begin{axis}
       [name=plot1, ,title= {Unconstrained Model - Adv. Labeling},xlabel={$\epsilon$ },ylabel={Test Accuracy}, ymin = 0.5, ymax =1, xtick={0,0.05,0.1,0.15,0.2}, xticklabels={0,0.05,0.1,0.15,0.2}, grid = major]

       \addplot[solid, blue, mark = *, mark options={solid,}, error bars/.cd, y dir=both, y explicit] table[skip first n=1,x index=1, y index=2, y error index=12, col sep=comma] {"plot2/reduction_csv/attacker-minmax2-0.01.csv"};
          \addplot[solid, red, mark = square, mark options={solid,},error bars/.cd, y dir=both, y explicit] table[skip first n=1,x index=1, y index=2, y error index=12, col sep=comma] {"plot2/reduction_csv/attacker-uflip-0.01.csv"};
          \addplot[solid, gray, mark = triangle, mark options={solid,}, error bars/.cd, y dir=both, y explicit] table[skip first n=1,x index=1, y index=2, y error index=12, col sep=comma] {"plot2/reduction_csv/attacker-usel-0.01.csv"};
          \addplot[solid, brown, mark = diamond, mark options={solid,}, error bars/.cd, y dir=both, y explicit] table[skip first n=1,x index=1, y index=2, y error index=12, col sep=comma] {"plot2/reduction_csv/noise-usel-0.01.csv"};
		      \addplot[no marks] table[skip first n=1,x index=1, y index=2, col sep=comma] {"plot2/baseline.csv"};

     \end{axis}

      \begin{axis}
        [name=plot2, at=(plot1.south east), anchor=south west, xshift=0.8cm,title= {Fair \cite{hardt2016equality} ($\delta = 0$)  - Adv. Labeling},xlabel={$\epsilon$},ylabel={}, ymin = 0.5, ymax =1,xtick={0,0.05,0.1,0.15,0.2}, xticklabels={0,0.05,0.1,0.15,0.2},  grid = major, legend entries = {{Alg.~\ref{algorithm:minmax-fair} ($\lambda = 0$)~\cite{steinhardt2017certified}, Alg.~\ref{algorithm:minmax-fair} ($\lambda = 100\epsilon$), Alg.~\ref{algorithm:OGD} ($\lambda = \epsilon$) ,  Label flipping,  Random sampling, Hard examples, Constant prediction}},legend style={at={(0., -0.20)}, anchor=north, draw=none, legend columns=4}]

		\addplot[solid, blue, mark = *, mark options={solid,}, error bars/.cd, y dir=both, y explicit] table[skip first n=1,x index=1, y index=3,y error index=13,  col sep=comma] {"plot2/PP_csv/attacker-minmax2-0.01.csv"};

		\addplot[violet, mark = o, error bars/.cd, y dir=both, y explicit] table[skip first n=1,x index=1, y index=3, y error index=13, col sep=comma] {"plot2/PP_csv/attacker-eominmaxloss-0.01.csv"};

		\addplot[teal, mark = x, error bars/.cd, y dir=both, y explicit] table[skip first n=1,x index=1, y index=3, y error index=13, col sep=comma] {"plot2/PP_csv/attacker-newreg-1.0-0.01.csv"};

		\addplot[solid, red, mark = square, mark options={solid,}, error bars/.cd, y dir=both, y explicit] table[skip first n=1,x index=1, y index=3, y error index=13, col sep=comma] {"plot2/PP_csv/attacker-uflip-0.01.csv"};

		\addplot[solid, gray, mark = triangle, mark options={solid,}, error bars/.cd, y dir=both, y explicit]table[skip first n=1,x index=1, y index=3, y error index=13, col sep=comma] {"plot2/PP_csv/attacker-usel-0.01.csv"};
      	 \addplot[solid, brown, mark = diamond, mark options={solid,}, error bars/.cd, y dir=both, y explicit] table[skip first n=1,x index=1, y index=3, y error index=13, col sep=comma] {"plot2/PP_csv/noise-usel-0.01.csv"};

		\addplot[no marks] table[skip first n=1,x index=1, y index=2, col sep=comma] {"plot2/baseline.csv"};

      \end{axis}

      \begin{axis}
        [name=plot3,
    at=(plot2.south east), anchor=south west, xshift=0.8cm,title= {Fair \cite{agarwal2018reductions} ($\delta = 0.01$) - Adv. Labeling}, xlabel={ $\epsilon$},ylabel={}, ymin = 0.5, ymax =1,xtick={0,0.05,0.1,0.15,0.2}, xticklabels={0,0.05,0.1,0.15,0.2},  grid = major]

		 \addplot[solid, blue, mark = *, mark options={solid,}, error bars/.cd, y dir=both, y explicit] table[skip first n=1,x index=1, y index=3, y error index=13, col sep=comma] {"plot2/reduction_csv/attacker-minmax2-0.01.csv"};

        \addplot[violet, mark = o, error bars/.cd, y dir=both, y explicit] table[skip first n=1,x index=1, y index=3, y error index=13, col sep=comma] {"plot2/reduction_csv/attacker-eominmaxloss-0.01.csv"};
       \addplot[teal, mark = x, error bars/.cd, y dir=both, y explicit] table[skip first n=1,x index=1, y index=3, y error index=13, col sep=comma] {"plot2/reduction_csv/attacker-newreg-1.0-0.01.csv"};
         \addplot[solid, red, mark = square, mark options={solid,}, error bars/.cd, y dir=both, y explicit] table[skip first n=1,x index=1, y index=3, y error index=13, col sep=comma] {"plot2/reduction_csv/attacker-uflip-0.01.csv"};
         \addplot[solid, gray, mark = triangle, mark options={solid,}, error bars/.cd, y dir=both, y explicit] table[skip first n=1,x index=1, y index=3, y error index=13, col sep=comma] {"plot2/reduction_csv/attacker-usel-0.01.csv"};
         \addplot[solid, brown, mark = diamond, mark options={solid,}, error bars/.cd, y dir=both, y explicit] table[skip first n=1,x index=1, y index=3, y error index=13, col sep=comma] {"plot2/reduction_csv/noise-usel-0.01.csv"};

		\addplot[no marks] table[skip first n=1,x index=1, y index=2, col sep=comma] {"plot2/baseline.csv"};

      \end{axis}
      \end{tikzpicture} 
      \end{tabular}
}}
\caption{Test accuracy of unconstrained and fair models under data poisoning attacks -- COMPAS dataset.
		The x-axis $\epsilon$ is the ratio of the size of poisoning dataset $\D_\poisoning$ to the size of clean dataset $\D_\clean$, and reflects the contamination level of training set.  We compare the impact of adversarial bias with baselines and poisoning attacks against unconstrained models, for various $\epsilon$.  The difference between test accuracy at $\epsilon=0$ (benign setting) and larger $\epsilon$ values reflects the impact of the attack.  Constant prediction always outputs the majority label in clean dataset.
}
\label{fig:accuracy-compas-paper}
\end{figure*}
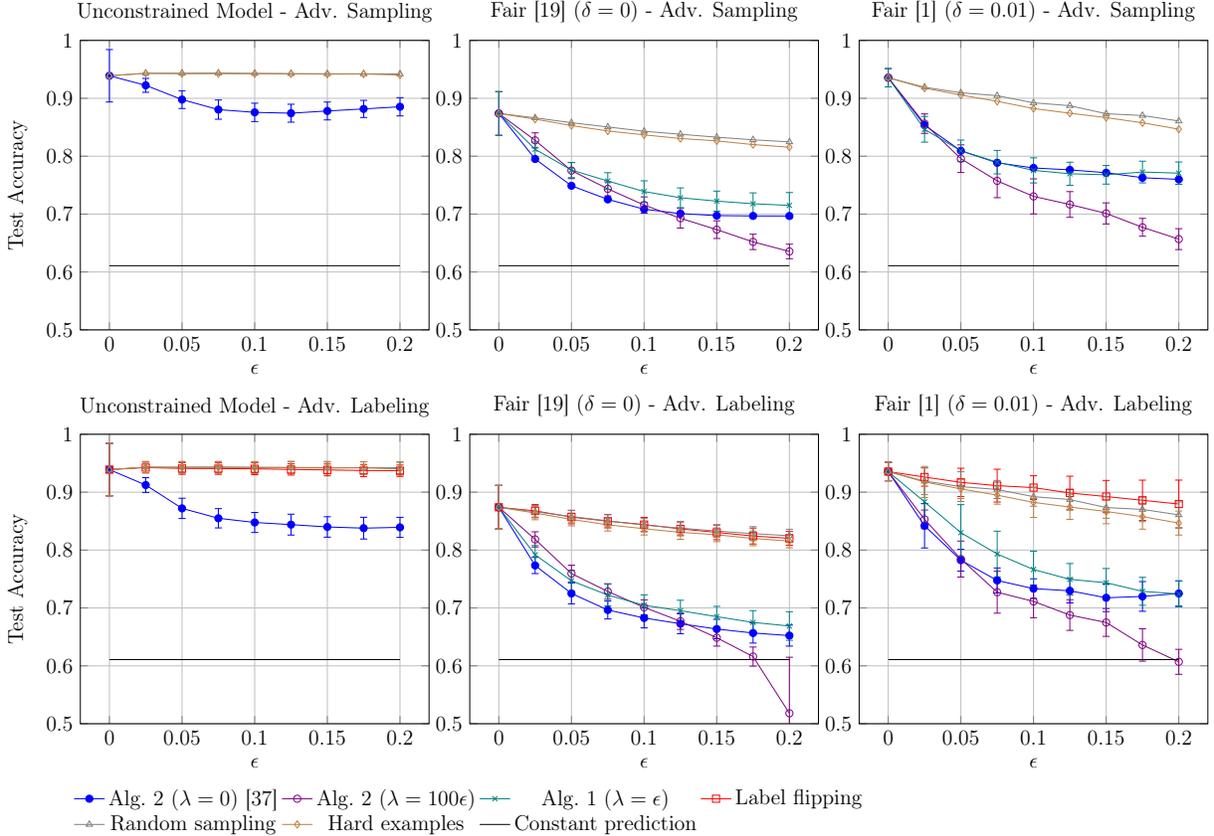

%% file: plot/fairness-measure_COMPAS_sampling.tex
\begin{figure*}[t!]
  \centering
 \resizebox{\columnwidth}{!}{%
       \begin{tikzpicture}
         \begin{axis}
         [name=plot1, legend style={font=\small},title= {(a) Overall},xlabel={$\epsilon$},ylabel={}, ymin = 0.5, ymax =1, ylabel={Test Accuracy}, grid = major , xtick={0,0.05,0.1,0.15,0.2}, xticklabels={0,0.05,0.1,0.15,0.2},legend entries = {Unconstrained model, Fair~\cite{agarwal2018reductions} ($\delta = 0.1$), Fair~\cite{agarwal2018reductions} ($\delta = 0.01$), Fair~\cite{hardt2016equality}  ($\delta = 0$)}, legend pos=south west]
           \addplot[teal,mark = asterisk,error bars/.cd, y dir=both, y explicit]  table[skip first n=1,x index=1, y index=2, col sep=comma, y error index=12] {"plot2/reduction_csv/attacker-eominmaxlossnf100-0.01.csv"};
           \addplot[red,mark = square*,error bars/.cd, y dir=both, y explicit]  table[skip first n=1,x index=1, y index=3,  col sep=comma, y error index=13] {"plot2/reduction_csv/attacker-eominmaxlossnf100-0.1.csv"};
           \addplot[blue, mark = diamond*,error bars/.cd, y dir=both, y explicit] table[skip first n=1,x index=1, y index=3, col sep=comma, y error index=13] {"plot2/reduction_csv/attacker-eominmaxlossnf100-0.01.csv"};
       	  \addplot[brown, mark = *,error bars/.cd, y dir=both, y explicit] table[skip first n=1,x index=1, y index=3,  col sep=comma, y error index=13] {"plot2/PP_csv/attacker-eominmaxlossnf100-0.01.csv"};
         \end{axis}
         \begin{axis}
         [name=plot2, at=(plot1.south east), anchor=south west, xshift=0.8cm, legend style={font=\small},title= {(b) Majority group},xlabel={$\epsilon$},ylabel={}, ymin = 0.5, ymax =1, xtick={0,0.05,0.1,0.15,0.2}, xticklabels={0,0.05,0.1,0.15,0.2},grid = major]

           \addplot[teal,mark = asterisk,error bars/.cd, y dir=both, y explicit] table[skip first n=1,x index=1, y index=4, col sep=comma, y error index=14] {"plot2/reduction_csv/attacker-eominmaxlossnf100-0.01.csv"};
           \addplot[red,mark = square*,error bars/.cd, y dir=both, y explicit] table[skip first n=1,x index=1, y index=6,  col sep=comma, y error index=16] {"plot2/reduction_csv/attacker-eominmaxlossnf100-0.1.csv"};
           \addplot[blue, mark = diamond*,error bars/.cd, y dir=both, y explicit] table[skip first n=1,x index=1, y index=6, col sep=comma, y error index=16] {"plot2/reduction_csv/attacker-eominmaxlossnf100-0.01.csv"};
      	  \addplot[brown, mark = *,error bars/.cd, y dir=both, y explicit] table[skip first n=1,x index=1, y index=6,  col sep=comma, y error index=16] {"plot2/PP_csv/attacker-eominmaxlossnf100-0.01.csv"};

         \end{axis}
           \begin{axis}
           [name=plot3, at=(plot2.south east), anchor=south west, xshift=0.8cm, legend style={font=\small},title= {(c) Minority group}, xlabel={$\epsilon$}, ylabel={}, ymin = 0.5, ymax =1, xtick={0,0.05,0.1,0.15,0.2}, xticklabels={0,0.05,0.1,0.15,0.2}, grid = major]

             \addplot[teal,mark = asterisk,error bars/.cd, y dir=both, y explicit] table[skip first n=1,x index=1, y index=5, col sep=comma, y error index=15] {"plot2/reduction_csv/attacker-eominmaxlossnf100-0.01.csv"};
             \addplot[red,mark = square*,error bars/.cd, y dir=both, y explicit] table[skip first n=1,x index=1, y index=7,  col sep=comma, y error index=17] {"plot2/reduction_csv/attacker-eominmaxlossnf100-0.1.csv"};
             \addplot[blue, mark = diamond*,error bars/.cd, y dir=both, y explicit] table[skip first n=1,x index=1, y index=7, col sep=comma, y error index=17] {"plot2/reduction_csv/attacker-eominmaxlossnf100-0.01.csv"};
        	  \addplot[brown, mark = *,error bars/.cd, y dir=both, y explicit] table[skip first n=1,x index=1, y index=7,  col sep=comma, y error index=17] {"plot2/PP_csv/attacker-eominmaxlossnf100-0.01.csv"};

           \end{axis}
         \end{tikzpicture}

 }
 \caption{Effect of fairness level $\delta$ on robustness across groups under adversarial sampling attack -- COMPAS dataset.  For a given $\epsilon$ the poisoning data is the same for all algorithms (generated using Alg.~\ref{algorithm:minmax-fair} with $\lambda=100\epsilon$).  The majority group (whites) contributes $61\%$ of the training data.  The curve for accuracy of the minority group under the unconstrained model overlaps with that of the model with exact fairness ($\lambda = 0$), thus is not visible.  See Figure~\ref{fig:fairness-measure} for adversarial labeling results. 
 }
 \label{fig:fairness-measure-compas-sampling}
\end{figure*}
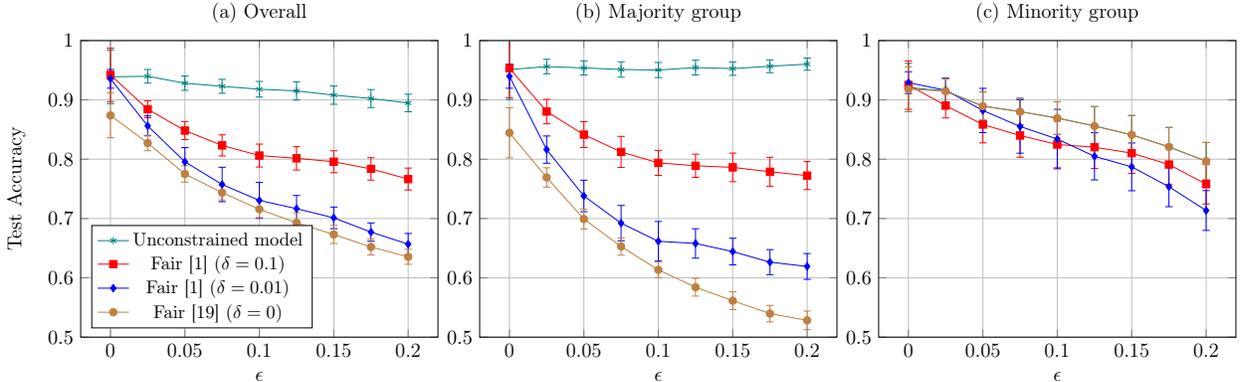

%% file: sections/tab-fairness-paper.tex

{\renewcommand{\arraystretch}{1.15}
\begin{table}[t]
\caption{{\textbf{Fairness gap of attacked models on test data} -- COMPAS dataset, for $\epsilon=0.1$.  The fairness gap $\Delta$ is defined in \eqref{eq:eo-fairness}.  The numbers reflect how unfair the model is with respect to the protected group in the test data.  For fair models, compare numbers with $\delta$ (the guaranteed fairness gap on training data).  The farther apart $\Delta$ and $\delta$ are, the less the fairness generalization is on test data.}
}
\label{table:robustness-fairness-paper}
\centering
{\small
\begin{tabularx}{\columnwidth}{l >{\centering\arraybackslash}X >{\centering\arraybackslash}X >{\centering\arraybackslash}X >{\centering\arraybackslash}X >{\centering\arraybackslash}X}
\toprule
Attacks& Unconstrained Model & Fair~\cite{agarwal2018reductions} ($\delta=0.1$) & Fair~\cite{agarwal2018reductions} $(\delta=0.01)$ & Fair~\cite{hardt2016equality} ($\delta=0$)\\ \hline \hline
Benign  &0.21$\pm$0.07&0.11$\pm$0.06&0.06$\pm$0.04&0.07$\pm$0.04 \\ \cline{1-5}
Random sampling&0.19$\pm$0.07&0.08$\pm$0.03&0.11$\pm$0.05&0.13$\pm$0.07 \\
Hard examples&0.19$\pm$0.08&0.09$\pm$0.03&0.13$\pm$0.05&0.15$\pm$0.07 \\
Label flipping&0.23$\pm$0.07&0.09$\pm$0.04&0.07$\pm$0.04&0.1$\pm$0.06\\ \cline{1-5}

Adv. sampling~(Alg.~\ref{algorithm:minmax-fair}, $\lambda = 0$)~\cite{steinhardt2017certified}&0.26$\pm$0.08&0.19$\pm$0.07&0.30$\pm$0.07&0.27$\pm$0.08 \\
Adv. sampling~(Alg.~\ref{algorithm:minmax-fair}, $\lambda = 100\epsilon$)&-&0.29$\pm$0.06&0.37$\pm$0.09&0.53$\pm$0.05 \\ 
Adv. sampling~(Alg.~\ref{algorithm:OGD}, $\lambda = \epsilon$)&-&0.12$\pm$0.07&0.21$\pm$0.10&0.25$\pm$0.13 \\ \cline{1-5}

Adv. labeling~(Alg.~\ref{algorithm:minmax-fair}, $\lambda = 0$)~\cite{steinhardt2017certified}&0.28$\pm$0.08&0.13$\pm$0.05&0.19$\pm$0.08&0.25$\pm$0.08\\
Adv. labeling~(Alg.~\ref{algorithm:minmax-fair}, $\lambda = 100\epsilon$)&-&0.28$\pm$0.05&0.39$\pm$0.08&0.55$\pm$0.04\\
Adv. labeling~(Alg.~\ref{algorithm:OGD}, $\lambda = \epsilon$)&-&0.11$\pm$0.06&0.12$\pm$0.04&0.13$\pm$0.09\\
\bottomrule \\
\end{tabularx}
}
\end{table}
}

%% file: sections/relatedwork.tex

\section{Related Work} \label{sec:relatedwork}

\subsection{Fairness in Machine Learning}

A classifier that is learned by minimizing the overall cumulative loss might not perform well on one sensitive group (usually the minority group), when the distribution of features per each class is different across groups. In order to address this problem, multiple definitions of fairness are proposed in the literature.  Examples include metric equality across sensitive groups~\cite{hardt2016equality, calders2009building}, individual fairness~\cite{dwork2012fairness}, causality~\cite{kusner2017counterfactual}, and many techniques to satisfy group-based fairness (which is the focus of this paper) such as pre-processing methods~\cite{zemel2013learning, madras2018learning}, in-processing methods~\cite{kamishima2011fairness, zafar2015fairness, zafar2017fairness, agarwal2018reductions}, and post-processing methods~\cite{hardt2016equality}. Pre-processing methods aim at finding a new representation of data such that it retains as much information of input features as possible, except those which can lead to bias. In-processing methods enforce fairness during the training process, for example, by incorporating the fairness constraints into the objective function as a regularization term. Post-processing methods correct the predictions of a given trained model, without modifying the training data or the training process. Please refer to~\cite{mehrabi2019survey} for a recent survey on methods to achieve fairness. In this work, we focus on the notion of Equalized odds~\cite{hardt2016equality} and use the reductions approach~\cite{agarwal2018reductions} (in-processing) and post-processing approach~\cite{hardt2016equality} to train fair models.

Imposing fairness constraints might come at a cost of the model's performance. The effect of fair classification on accuracy and the compatibility of various definitions with each other have been studied in some related works~\cite{corbett2017algorithmic, keinberg2017inherent}. \citet{corbett2017algorithmic} show that the optimal decision rule is different from the fair decision rules that satisfy fairness definitions (statistical parity, conditional statistical parity, predictive equality).  Thus, imposing fairness constraints has a cost on the model accuracy. \citet{corbett2017algorithmic} then evaluate the cost of fairness empirically. \citet{keinberg2017inherent} show that it is {\em impossible} to achieve equal calibration, false positive rate and false negative rate, if the fraction of positive labeled examples is different across sensitive groups.

\subsection{Data Poisoning Attacks}

Machine learning systems are susceptible to data poisoning attacks. In indiscriminate attacks, which is the focus of this paper, the adversary's objective is to degrade the test accuracy of the model~\cite{biggio2012poisoning, mei2015security, jagielski2018manipulating, li2016data, mei2015using,koh2018stronger,koh2017understanding}.  In targeted attacks, the adversary seeks to impose the loss on specific test data points or small sub-populations \cite{gu2017badnets, chen2017targeted, burkard2017analysis, shafahi2018poison, koh2017understanding, suciu2018does}.

\citet{steinhardt2017certified} propose an optimal algorithm for poisoning attacks on (unconstrained) convex models, given a set of feasible poisoning data points. The algorithm relies on the assumption that test loss of the target model can be approximated as training loss of the model on clean data (assuming $\D_\test$ is drawn from the same distribution as the clean training data $\D_\clean$).  Our attack algorithm is inspired by this work and uses the same online learning framework. Note that, when $\lambda = 0$ in Algorithm~\ref{algorithm:minmax-fair}, it is equivalent to the algorithm in~\cite{steinhardt2017certified}.

In our setting of adversarial sampling bias, the attacker is not allowed to modify the label $y$. In {\em clean-label} data poisoning attacks~\cite{shafahi2018poison}, the attacker manages to reduce the accuracy of target examples via injecting the correctly labeled data with {\em modified features}. Compared with this work, the attacker in adversarial sampling bias is not permitted to change the features. Furthermore, the objective of our attack is to reduce the accuracy of the model over the entire test data.

When the attacker is allowed to change both features and labels of the poisoning data, a typical poisoning attack algorithm is gradient ascent, in which the attacker iteratively modifies each attack point in the poisoning dataset by following the gradient of the test loss with respect to poisoning training data. This kind of attack is first studied in the context of SVMs by~\citet{biggio2012poisoning}, and has subsequently been extended to linear and logistic regression~\cite{mei2015using}, topic modeling~\cite{mei2015security}, collaborative filtering~\cite{li2016data}, and neural networks~\cite{koh2017understanding}.  In our setting, we assume the attacker is not allowed to modify the features, as we focus on the most practical scenario in decision making processes that move toward automation. An interesting future direction would be to allow changes of features and design poisoning attacks using the gradient-based algorithm. Given more power to the attacker, it is likely that the attacker could reduce the test accuracy more significantly.

\subsection{Learning Fair Models from Noisy Training Data}

In most practical scenarios, the training data used for learning models might be biased (under-representation bias) and/or noisy (with mis-labeling). The mis-labeling phenomenon can be random or adversarial. Mislabeling can be seen as a specific case of adversarial labeling bias, where the attacker flips labels of data points only from a certain part of the population. Similarly, under-representation bias can be considered as a specific case of adversarial sampling bias. Multiple works in the literature study the impact of noisy and biased data on machine learning.

\citet{calders2013unbiased} show that learning a regular unconstrained model from training data with under-representation and mislabeling bias results in biased predictions on test data. \citet{kallus2018residual} consider the case of systematic censoring in training data. For example, a model for predicting whether an individual defaults a loan can be trained only on individuals who were already granted a loan. Individuals who were not even granted loan cannot be present in the dataset. Such systematic censoring can be seen as a form of sampling bias. This work shows that even after using a fair classifier, that seeks to achieve fairness by equalizing accuracy metrics across sensitive groups, the classifier can still be unfair on the population due to the systematic censoring in training data. We also show a similar result in Table~\ref{table:robustness-fairness} that learning models on training data with adversarial bias increases their fairness gap on the test data.

Under varying assumptions, multiple works~\cite{blum2019recovering, jiang2019identifying, de2018learning, lamy2019noise} have proposed strategies to account for under-representation bias and mislabeling while learning models. \citet{de2018learning} study selective label bias, where true outcomes corresponding to a certain label cannot be learned (for example in predicting recidivism risk), as such examples cannot be added to the training data. Selective label bias can be considered as a form of sampling bias. The authors propose a method for augmenting the dataset with human expert predictions to mitigate selective label bias. Assuming that examples in certain sensitive groups are randomly mislabeled, \citet{jiang2019identifying} propose a re-weighting strategy for recovering the optimal classifier on unbiased data from training data with labeling bias.
When uniform random noise is present in the sensitive attribute, it is shown that demographic parity gap of a fair classifier on test data increases~\cite{lamy2019noise}. The authors quantify the increase in DP gap on test data at any level of noise in the label. Given the level of noise in sensitive attribute, this is used to compute the exact level of DP gap that needs to be imposed on training data, for achieving a target DP gap on the test data. \citet{blum2019recovering} consider a training set corrupted by under-representation/labeling bias (or both). Assuming access to an infinite number of samples and learning different classifiers for different sensitive groups, this work shows that ERM with Equal Opportunity constraint on the biased data can recover the Bayes-optimal classifier for the true data distribution.

All the above works~\cite{blum2019recovering, jiang2019identifying, de2018learning, lamy2019noise} assume that the noise and bias in training data is an uniform distribution of under-representation/mislabeling over a subspace of points and study the consequences of learning from training data with such bias. These results cannot translate to our case of adversarial bias as we consider non-uniform bias over the input space, and our attacker introduces bias with the specific intention of reducing test accuracy.

%% file: sections/impact.tex
\section{Broader Impact}\label{sec:impact}

AI governance frameworks released by multiple organizations such as the European union\footnote{EU guidelines on ethics in artificial intelligence: Context and implementation \url{https://www.europarl.europa.eu/RegData/etudes/BRIE/2019/640163/EPRS_BRI(2019)640163_EN.pdf}} and Google AI\footnote{Perspectives on Issues in AI Governance \url{https://ai.google/static/documents/perspectives-on-issues-in-ai-governance.pdf}} state that fairness and robustness are two key requirements for building trustworthy automated systems. In fact, the AI governance document by Google AI mentions a library for training with equalized odds constraints as a tool for {\bf fairness} appraisal (page 14) and {\bf data poisoning} as a possible risk for AI safety (page 17).

In this work, we show that imposing group-fairness constraints on learning algorithms decreases their robustness to poisoning attacks.  This is a significant obstacle towards implementing trustworthy machine learning systems.  We specifically provide evidence that an attacker that can only control the sampling and labeling process for a fraction of the training data can significantly degrade the test accuracy of the models learned with fairness constraints. In fact, from a practical perspective, the attack algorithms for adversarial bias, introduced in this paper, can easily and stealthily be perpetrated in many existing systems, as similar to historical discrimination and/or selection bias.  This calls for an immediate attention to a theoretical study of the robustness properties for any fair machine learning algorithm and the potential consequences of using such algorithms in presence of adversarially biased data.  Moreover, this calls for designing models which are not only fair, but also robust.  We suspect that there might be a fundamental trade-off between these two aspects of trustworthy machine learning. We also show that learning with fairness constraints in presence of adversarial bias results in a classifier that does not only have a poor test accuracy but is also potentially more discriminatory on test data.  Hence, machine learning system designers must be cautious when deploying FairML in real world applications, as they might be building a system that is both unfair and less robust.

%% file: sections/conclusion.tex

\section{Conclusions}

We have introduced adversarial bias as a framework for data poisoning attacks against fair machine learning. Our attack exploits the existing tension between fairness constraint and model accuracy, and the fact that the fair models try to achieve equality on groups with different sensitive attributes even though they do not have the same weight in the loss function of the model. Thus, our experiments show that by adding a small percentage of adversarially sampled/labeled data points to the training set, the attacker can significantly reduce the model accuracy beyond what he can achieve in unconstrained models.  Adversarial bias also increases the fairness gap on test data. 

%% file: sections/acknowledgements.tex

\section*{Acknowledgments}

This  work  is supported  by  the  NUS  Early  Career Research Award (NUS ECRA) by the Office of the Deputy President, Research \& Technology (ODPRT), award number NUS-ECRA-FY19-P16.

%% file: sections/tab-notations.tex

\section{Table of Notations}

\begin{table}[h!]
	\centering
	\caption{List of Notations}
	  \begin{adjustbox}{max width=\columnwidth}
		\begin{tabular}{lll}
			\hline
			\textbf{Symbol} & \textbf{Description} & \textbf{Where it is defined}\\ \hline
			 $\mathcal{X}$ & Features space & Section~\ref{backgrd:problem_setting} \\
			 $\mathcal{Y}$ & Label space & Section~\ref{backgrd:problem_setting} \\
			$X$ & Random variable associated with features & Section~\ref{backgrd:problem_setting} \\
			 $Y$ & Random variable associated with lables & Section~\ref{backgrd:problem_setting} \\
			 $(X,Y)$ & Underlying distribution of the data & Section~\ref{backgrd:problem_setting} \\
			 $\D$ & Training dataset & Section~\ref{backgrd:problem_setting} \\
			$\D_\clean$ & Clean training dataset & Section~\ref{backgrd:problem_setting} \\
			$n$ & Size of the clean training dataset & Section~\ref{backgrd:problem_setting} \\
			$\D_\poisoning$ & Poisoning training dataset & Section~\ref{backgrd:problem_setting} \\
			$\epsilon$ & The ratio of the size of poisoning data over the size of clean data in the training set  & Section~\ref{backgrd:problem_setting} \\
			$\Pr_\D$ & Computing a probability empirically over a dataset $\D$ & Section~\ref{backgrd:problem_setting} \\
			$\D_\attack$ & Attack dataset   & Section~\ref{backgrd:problem_setting} \\
			$\D_\test$ & Test dataset & Section~\ref{backgrd:problem_setting} \\
			$S$ & Sensitive/protected attribute & Section~\ref{backgrd:problem_setting} \\
			$\delta$ & Guaranteed fairness level on training data& Section~\ref{backgrd:problem_setting} \\
			$\Delta$ & Fairness gap & Section~\ref{backgrd:problem_setting} \\
			$C(\theta,\D)$ & Fairness constraint of $f_{\theta}$ on dataset $\D$ &  Section~\ref{backgrd:fairness_definition}\\
			$\theta$ & Model parameters & Section~\ref{backgrd:problem_setting} \\
			$\Theta$  & Parameter space & Section~\ref{backgrd:problem_setting} \\
			$f_{\theta}$ & Classification model parameterized by $\theta$  & Section~\ref{backgrd:problem_setting} \\
			$\ell(\theta;x,y)$ &  Loss of $f_\theta$ on data $(x,y)$ & Section~\ref{backgrd:problem_setting} \\
			$\L(\theta;\D)$ & Cumulative loss of $f_\theta$ on dataset $\D$ & Section~\ref{backgrd:problem_setting} \\
			$\hat{\theta}$ & Optimal model parameters trained on $\D$ with fairness constrained & \cref{eq:poisoning_attack_fair} \\
			$\lambda$ & Lagrange multiplier (penalty parameter)& Section~\ref{sec:framework} \\
			$\D \cup \{(x,y)\}^k$ & Adding $k$ copies of the data point $(x,y)$ to set $\D$ & Section~\ref{sec:certified} \\
			$M^*$ & Maximum loss under the optimal attack & \cref{eq:surrogate-optimization} \\
			 $M(\D_\poisoning)$ & The imposed loss by poisoning dataset $\D_\poisoning$ on the fair model & \cref{eq:surrogate-optimization}\\
			 $(x^t, y^t)$ & The data point selected by Algorithm~\ref{algorithm:OGD} at step $t$ &  Section~\ref{sec:certified} \\
			 $\theta^t$ & The model parameter chosen by Algorithm~\ref{algorithm:OGD} at step $t$ &  Section~\ref{sec:certified} \\
			 $\F(\D_\attack)$ & Feasible set of poisoning data points & Section~\ref{sec:certified} \\
			$\eta$ & Learning rate & Algorithm~\ref{algorithm:OGD} \\
			$\D_\poisoning^*$ & Poisoning dataset produced by Algorithm~\ref{algorithm:OGD}  & Section~\ref{sec:certified}\\
			 $\text{Regret(T)}$ & Regret of Algorithm~\ref{algorithm:OGD} after $T$ steps & Section~\ref{sec:certified} \\
			 $U(\theta)$ & Loss of model $f_\theta$ under the optimal attack  & Appendix~\ref{appdx:proof} \\
			 $U^*$ & Minimum loss under the optimal attack $\min_{t=1}^T U(\theta^t)$ & Appendix~\ref{appdx:proof} \\
			 $g_t(\theta)$ & The loss function for online learning algorithm at step $t$ &  Appendix~\ref{appdx:proof} \\
			 $\tilde{\theta}$ & Optimal model parameters for minimizing the cumulative loss $\sum_{t=1}^T g_t(\theta)$ & Appendix~\ref{appdx:proof} \\
			  $\eta_t$ & Learning rate at time $t$ in Algorithm~\ref{algorithm:OGD} used in the proof of Corollary~\ref{corollary:optimality} & Appendix~\ref{appdx:proof} \\
			  $d$ & Upper bound on the diameter of $\Theta$ & Appendix~\ref{appdx:proof} \\
			  $G$ & Upper bound on the norm of the subgradients of $g_t$ over $\Theta$ & Appendix~\ref{appdx:proof} \\
			  $\tilde{\Delta}(\theta,\D)$ & The convex relaxation for the fairness gap of equalized odds  & Appendix~\ref{sec:regret-approximation} \\
			  $\ell_{l}(\theta;x,y)$ & linear loss of model $f_\theta$ on data point $(x,y)$ & Appendix~\ref{sec:regret-approximation} \\
			  $\D^{y,s}$ &  A set of data points from group $s$ with label $y$ in $\D$  & Appendix~\ref{sec:regret-approximation} \\
			  $n^{y,s}$ &  The number of data points in $\D^{y,s}$ & Appendix~\ref{sec:regret-approximation} \\
			  $R^{y,s}(\theta,\D)$ & Average linear loss of $f_\theta$ for data points in $\D^{y,s}$ & Appendix~\ref{sec:regret-approximation} \\
			\hline
		\end{tabular}
	\end{adjustbox}
	\label{table:notations}
\end{table}

%% file: sections/appdx-theory.tex

\section{Supplementary Theoretical Results}

\subsection{Proof for Theorem~\ref{theorem:bound-regret}} \label{appdx:proof}

\begin{proof}
We should point out that in this proof we follow the approach of \cite{steinhardt2017certified}.
	Assume $T = \epsilon n$ is the time horizon.
	We have $\D^*_\poisoning = \{  (x^1,y^1) , \cdots\ (x^T, y^T) \}$ as the data poisoning set produced by Algorithm~\ref{algorithm:OGD}. Also, $\theta^t $ is the parameter chosen by the algorithm at the $t$-th step.
	First, from max–min inequality we have:
	\begin{align*}
		M^* \eqtext{def} \max_{\D_\poisoning} \min_{\theta} & \left[  \frac{1}{n} \mathcal{L} (\theta; \mathcal{D}_\clean \cup \mathcal{D}_\poisoning) + \frac{\lambda}{\epsilon n} \cdot \sum_{(x,y) \in \mathcal{D}_\poisoning}   \Delta \left(\theta, \mathcal{D}_\clean  \cup  \xyen \right) \right]  \\
		&	 \leq
		\min_{\theta} \max_{\D_\poisoning}  \left[  \frac{1}{n} \mathcal{L} (\theta; \mathcal{D}_\clean \cup \mathcal{D}_\poisoning) +
	\frac{\lambda}{\epsilon n} \cdot \sum_{(x,y) \in \mathcal{D}_\poisoning}   \Delta \left(\theta, \mathcal{D}_\clean  \cup  \xyen \right) \right] \enspace .
	\end{align*}

	Furthermore, for a given $\theta$ we define:
	\begin{align*}
		U(\theta) \eqtext{def}  \max_{\D_\poisoning}  & \left[  \frac{1}{n} \mathcal{L} (\theta; \mathcal{D}_\clean \cup \mathcal{D}_\poisoning) +
	\frac{\lambda}{\epsilon n}  \cdot \sum_{(x,y) \in \mathcal{D}_\poisoning}   \Delta \left(\theta, \mathcal{D}_\clean  \cup \xyen \right) \right]  \\
		& =  \frac{1}{n} \mathcal{L } (\theta; \mathcal{D}_\clean) +   \max_{(x,y) \in \F(\D_\attack)} \left[ \epsilon \cdot  \ell (\theta; x,y) + \lambda \cdot  \Delta \left(\theta, \mathcal{D}_\clean  \cup \xyen \right)  \right] \enspace .
	\end{align*}
	We define $U^{*} = \min_{t=1}^{T} U(\theta^t)$.
	Note that for any given $\theta$, we have $M^* \leq U(\theta)$. More specifically, we have $M^* \leq U^{*}$.

	From the definition of $M^*$, for any set, including $\D^*_\poisoning$, we have
	\begin{align*}
		\min_{\theta} \left[  \frac{1}{n} \mathcal{L } (\theta; \mathcal{D}_\clean \cup \D^*_\poisoning) +
		\frac{\lambda}{\epsilon n}  \cdot \sum_{(x,y) \in \D^*_\poisoning}
		\Delta \left(\theta, \mathcal{D}_\clean  \cup \xyen \right) \right]  = M(\D^*_\poisoning) \leq M^*
	\end{align*}

	Let us define $T$ different functions
	\begin{align} \label{eq:function-f}
	g_t(\theta) = \frac{1}{n} \mathcal{L} (\theta; \mathcal{D}_\clean ) +   \epsilon \cdot \ell (\theta; x^{t+1},y^{t+1})  +\lambda \cdot
	\Delta (\theta, \mathcal{D}_\clean  \cup \xyent{t+1} ) \enspace,
	\end{align}
	for $0 \leq t \leq T$. Let us define
	\[ \tilde{\theta} = \argmin_{\theta \in \Theta} \sum_{t=1}^{T} g_t(\theta) \enspace . \]

	Note that we have
	\[\frac{\sum_{t=1}^{T} g_t(\tilde{\theta})}{T} = M(\D^*_\poisoning) \leq
	M^*
	\leq U^* \leq  \frac{\sum_{t=1}^{T} g_t(\theta^t)}{T}
	\]
	Finally, from the definition of regret we have:
	\[ \frac{\sum_{t=1}^{T} g_t(\theta^t)}{T}   -  \frac{\sum_{t=1}^{T} g_t(\tilde{\theta})}{T}  = \frac{\mathrm{Regret}(T)}{T} \enspace, \]
	which consequently completes the proof of the theorem.
\end{proof}

\subsection{The Conditions for the Optimality of the Attack} \label{sec:regret-approximation}

In this section, we prove that under what conditions, our algorithm finds the (nearly) optimal solution for the attack.
We first state a direct consequence of \cref{theorem:bound-regret} for a no-regret algorithm which results from a convexity assumption for functions $g_t(\theta)$. We then explain under what conditions this convexity assumption is valid.

\begin{corollary}\label{corollary:optimality}
	Under the assumption that (i) loss function $\ell$ is convex in $\theta$, (ii) $\Delta(\theta, \D)$ is convex in~$\theta$ and (iii) $\eta_t = \frac{d}{G\sqrt{t}}$ for $1 \leq t \leq T$,
	Algorithm~\ref{algorithm:OGD} produces the near optimal poisoning dataset $D^*_\poisoning$, such that
	\begin{equation}
	M^* - M(D^*_\poisoning) \leq 	\frac{3Gd}{\sqrt{\epsilon n}}
	\end{equation}
	where $\eta_t$ is step size at time $t$, $d$ is an upper bond on the diameter of $\Theta$, and $G$ is an upper bound on the norm of the subgradients of $g_t$ over $\Theta$, i.e., $\norm{\nabla g_t(\theta)} \leq G$.
\end{corollary}

\begin{proof}
First note that Algorithm~\ref{algorithm:OGD} exactly runs as online gradient descent algorithm for $g_t(\theta)$ functions. 
From the assumptions (i) and (ii), we conclude that functions $g_t(\theta)$ are convex.
The theoretical guarantee of the online gradient descent algorithm for convex functions  \cite{hazan2016introduction} allows us to bound the average regret
\[  \frac{\mathrm{Regret}(T)}{T}  \leq  \frac{3Gd}{\sqrt T} \enspace, \]
where $d$ is an upper bond on the diameter of $\Theta$, and $G$ is an upper bound on the norm of the subgradients of $g_t$ over $\Theta$, i.e., $\norm{\nabla g_t(\theta)} \leq G$.
Finally, the proof is concluded from this bound for the regret and the result of \cref{theorem:bound-regret}.
\end{proof}

Next, we discuss the optimality conditions for linear classifiers with a convex loss, e.g., $\ell(\theta ; x, y)=\max (0,1-y\langle\theta, x\rangle)$ for SVM.
In our attack, we use \emph{equalized odds} as our fairness constraint which is non-convex.
We adopt simplification proposed by \citet{donini2018empirical} to reach convex relaxations of loss and fairness constraint.
Instead of balancing prediction error, \citet{donini2018empirical} propose a fairness definition as balancing the risk among two sensitive groups.
Following the same idea, we define the linear loss as $\ell_{l}$ (e.g., $\ell_{l} = (1 - f_\theta(x))/2$ for SVM). Based on the linear loss, the convex relaxation for the fairness gap of equalized odds is defined as follows:
\begin{align}\label{eq:loss_diff}
	\tilde{\Delta}(\theta, \D)  \coloneqq \frac{|R^{+,a}(\theta,\D) - R^{+,b}(\theta, \D)| + |R^{-,a}(\theta, \D) - R^{-,b}(\theta, \D)|}{2} \enspace ,
\end{align}
for $R^{y,s}(\theta, \D) = \frac{1}{n^{y,s}}\sum_{(x,y) \in \D^{y,s}}\ell_{l}(\theta;x,y)$ where $\D^{y,s}$  is the set of data points from group $s$ with label $y$ in $\D$ and $n_{y,s} = |\D^{y,s}|$.  
To find the optimal attack for the EO fair model, in \cref{eq:surrogate-optimization}, we replace loss $\ell$ with a convex loss (e.g. Hinge loss) $\ell_c$ and replace $\Delta(f_\theta;\D)$ with the convex relaxation $\tilde{\Delta}(\theta;\D)$. Hence, Algorithm~\ref{algorithm:OGD} produces the nearly optimal poisoning set $\D^*_\poisoning$ such that it has the maximal damage on the fair model under our approximations.

As a future research direction, one could try to design new online algorithms that achieve small regrets in the non-convex setting or under better approximations of the fairness constraint. 
Our framework can then utilize such online algorithms to further investigate the effect of data poisoning attacks on the robustness of models with fairness constraints.

\subsection{Pseudocode of the Algorithm from \cref{sec:regular-attack}} \label{appendix-alg2}

For the sake of completeness, we present the full pseudo-code for our data poisoning algorithm proposed in \cref{sec:regular-attack}.

\begin{algorithm}[h]
	\caption{}
	\begin{algorithmic}[1]
		\State {\bfseries Input:} Clean data $\D_\clean$, $n = |\D_\clean|$, feasible poisoning set $\F(\D_\attack)$, number of poisoning data $\epsilon n$, penalty parameter (Lagrange multiplier) $\lambda$, learning rate $\eta$.
\State {\bfseries Output:} Poisoning dataset $\D_\poisoning$.
		\State Initialize $\theta^0$
		\For{$ t = 1,\cdots,\epsilon n$}
		\State  $ (x^t, y^t) \gets \argmax_{(x,y) \in \F(\D_\attack)} 
		\left[ \epsilon \cdot \ell (\theta^{t-1}; x,y) + \lambda \cdot  \Delta\left(\theta^{t-1}, \D_\clean  \cup \xyen \right)  \right] $ \label{line:SA:pick}
		\State $\D_\poisoning \leftarrow \D_\poisoning \cup \{(x^t, y^t)\}$
		\State $\theta^t \leftarrow \theta^{t-1} - 
		\eta \left( \frac{\nabla \L (\theta^{t-1}; \D_\clean )}{n}   + 
		\epsilon \cdot  \nabla \ell(\theta^{t-1}; x^{t},y^{t})  
		\right)$ \label{line:SA-update}
		\EndFor
	\end{algorithmic}\label{algorithm:minmax-fair}
\end{algorithm}

%% file: sections/appdx-fair-algorithm.tex

\section{Fair Machine Learning Algorithms} \label{appdx:fair_impl}

The {\bf post-processing approach} is the first proposed algorithm to achieve equalized odds~\cite{hardt2016equality}. The fair model is obtained by adjusting a trained unconstrained model so as to remove the discrimination according to equalized odds. The outcome of this approach is a randomized classifier that assigns to each data point a probability of changing the prediction output by the unconstrained model, conditional on its protected attribute, and predicted label. These probabilities are computed by a linear program that optimizes the expected loss.

Many methods have been proposed to achieve fairness in machine learning (see~\cite{mehrabi2019survey} for a recent survey).  The {\bf reductions approach} proposed by ~\cite{agarwal2018reductions}  trains a fair randomized classifier over a hypothesis class by reducing the constrained optimization problem to learning a sequence of cost-sensitive classification models. Cost-sensitive classification is used in this as an oracle to solve classification problems resulted from a two-player game: one player (primal variables) minimizes the loss function; the other player (dual variables) maximizes the fairness violation (constraints).

%% file: sections/appdx-exp.tex

\section{Supplementary Experimental Results}\label{appdx:experiments}

For the following section, we present the detailed experimental results on COMPAS and Adult dataset.
All the results on COMPAS dataset are averaged over 100 runs with different random seeds. On Adult dataset, all the results are averaged over 50 runs with different random seeds.

\subsection{Details of datasets and models}\label{appdx:setup}

We use two datasets in our evaluation, their details are described below.

\begin{table}[t!]
\caption{Distribution of data points in clean training dataset $\D_\clean$ and attack dataset $\D_\attack$ - COMPAS dataset.}
\label{table:distributtion-compas}
\centering
\begin{tabular}{ccccccc}
\cline{1-3} \cline{5-7}
\multicolumn{3}{c}{$\D_\clean$} &  &\multicolumn{3}{c}{ $\D_\attack$}\\ \cline{1-3} \cline{5-7}
&$y=-$& $y=+$& &&$y=-$&$y=+$  \\ \cline{1-3} \cline{5-7}
$s=0$&28.5\% &31.8\% & &$s=0$& 29.0\% & 31.1\% \\ \cline{1-3} \cline{5-7}
$s=1$&32.5\% &7.2\% & &$s=1$& 16.0\% & 23.9\% \\ \cline{1-3} \cline{5-7}
\end{tabular}
\end{table}

\begin{table}[t!]
\caption{Distribution of data points in clean training dataset $\D_\clean$ and attack dataset $\D_\attack$ - Adult dataset.}
\label{table:distributtion-adult}
\centering
\begin{tabular}{ccccccc}
\cline{1-3} \cline{5-7}
\multicolumn{3}{c}{$\D_\clean$} &  &\multicolumn{3}{c}{$\D_\attack$}\\ \cline{1-3} \cline{5-7}
&$y=-$& $y=+$& &&$y=-$&$y=+$  \\ \cline{1-3} \cline{5-7}
$s=0$& 48.5\% &16.5\% & &$s=0$& 45.0\% & 23.4\% \\ \cline{1-3} \cline{5-7}
$s=1$& 32.3\% &2.6\% & &$s=1$& 27.2\% & 4.4\% \\ \cline{1-3} \cline{5-7}
\end{tabular}
\end{table}

\paragraph{COMPAS. }
COMPAS \cite{compasdataset} dataset  contains 5278 data samples. The classification task is to predict recidivism risk from criminal history and demographics.
We consider race as the sensitive attribute and include records only with white/black as race. There are 3175 records (60.2\%) for the sensitive attribute as white. Among the white group, 52.3\% have positive labels while among the black group, this number is 41.9\%. Overall, there are 2483 records (47\%) are labeled positive.

\textbf{\textit{Pre-processing}}
A model trained with the original dataset can only achieve low accuracy (66.6\% for the Logistic regression model, compared to constant prediction classifier that can achieve 53\% accuracy), which does not help the understanding of the model's behavior in the presence of data poisoning attacks. To get rid of the noise that exists in the dataset, we pre-process the dataset as follows: we train an SVM model with RBF kernel on the entire dataset and only keep 60\% of the data points which have the smallest loss. To create the training data $\D_\clean$, test data $\D_{\test}$ and attack dataset $\D_\attack$, we randomly split the clean data in the corresponding ratio 4:1:1. Hard examples (the left out data points) are added to the attack dataset.

\textbf{\textit{Data distribution}}
The data distribution of points in clean training dataset $\D_\clean$ and attack dataset $\D_\attack$ after pre-processing are presented in
Table~\ref{table:distributtion-compas}. The numbers are the average values over all the datasets we evaluated on. On average, $\D_\clean$ contains 2111 samples, $\D_\test$ 528 samples. $\D_\attack$ consists of 2639 samples out of which 2112 are hard examples.. A Logistic regression model trained on  $\D_\clean$ achieves on average 94\% accuracy on test data.

\textbf{\textit{Model}}
We use Logistic regression for classification.

\paragraph{UCI Adult (Census Income). }
Adult dataset \cite{adultdataset} includes 48,842 records with 14 attributes such as age, gender, education, marital status, occupation, working hours, and native country. The (binary) classification task is to predict if a person makes over \$50K a year based on the census attributes. We consider gender (male and female) as the sensitive attribute. In this dataset, 66.8\% are males, and 23.9\% are labeled one, i.e having an income over \$50K  a year. Among male samples, 30.4\% are positive samples; for the females, this number is 10.9\%.

\textbf{\textit{Pre-processing}}
A model trained on this dataset generally achieves below 90\% accuracy (Logistic regression: 85.3\%, 2-layer fully connected neural network with 32 hidden units each layer: 85.3\% on training data, compared to a constant prediction classifier that can achieve 76.1\% accuracy). To enhance the model accuracy, we apply similar pre-processing steps as on COMPAS:  we train an SVM model with Linear kernel on the entire dataset and keep 90\% of the data points which have the smallest loss. The number of females with income above \$50K is small; hence we randomly split 1/2 of the data for $\D_\attack$. Of the remaining data, 70\% are use for $\D_\clean$ and 30\% for $\D_{\test}$. Hard examples (the left out data points) are added to the attack dataset.

\textbf{\textit{Data distribution}}
The data distribution of the points in clean training dataset $\D_\clean$ and attack dataset $\D_\attack$ after pre-processing are presented in Table~\ref{table:distributtion-adult}. The numbers are the average values over all the datasets we evaluated on. On average, $\D_\clean$ contains 15385 samples, $\D_\test$ 6594 samples. $\D_\attack$ consists of 26863 samples.  $\D_\clean$ maintains approximately the same fractions of males and females as in the original dataset. A Logistic regression model trained on  $\D_\clean$ achieves on average 94\% accuracy on test data.

\textbf{\textit{Model}}
We use Logistic regression for classification.

\subsection{Implementation and Parameters Selection} \label{sec:lambda-effect}

To generate poisoning data points, we implement Algorithm~\ref{algorithm:OGD} and Algorithm~\ref{algorithm:minmax-fair}.

As discussed in \ref{sec:regret-approximation}, in Algorithm~\ref{algorithm:OGD}, the attacker uses an SVM model (due to the linear approximation of the fairness gap in Equation~\eqref{eq:loss_diff}). For this SVM model, we use Hinge loss for the classification loss and linear loss for evaluating the fairness gap as mentioned in~\ref{sec:regret-approximation}. Note that the linear function used to approximate $\Delta$ can fall out of range ($[0,1]$). Having large $\lambda$ implies assigning more weight to this term and can result in a bad approximation. We therefore test with small $\lambda$, with $\lambda \in \{0.1 \epsilon, \epsilon, 10 \epsilon\}$ and show  the results when $\lambda =\epsilon $ for COMPAS and  $\lambda =0.1 \epsilon$ for Adult.

In Algorithm 2, we use Logistic regression models. Since we measure the exact $\Delta$~of the model and want that $\Delta$~to have a large impact on finding a new poisoning data point in each iteration. This leads to the selection of a larger $\lambda$. We choose $\lambda \in \{\epsilon, 10 \epsilon, 100 \epsilon\}$ and show  $\lambda =100 \epsilon $ in the evaluation for both datasets.

For both algorithms, we use $\eta=0.001$ as the learning rate.

To train a fair model, we use the post-processing method~\cite{hardt2016equality} and reductions approach~\citep{agarwal2018reductions}. We use the implementation of these algorithms provided in~\citep{agarwal2018reductions}\footnote{See https://github.com/fairlearn/fairlearn}. Note that while the post-processing approach allows achieving exact fairness on the training data, the implementation of the reductions approach requires a strictly positive $\delta$. We use default values for all hyper-parameters from the available implementation.

 It is important to note that, the output of these approaches is a {\em randomized classifier}. We, therefore, use the expected accuracy to measure the classification performance, given by
 \begin{align} \label{eq:accuracy}
\mathrm{Acc}(\theta;\D) = 1-\frac{1}{|\D|}\sum_{(x,y) \in \D}|f_\theta(x) - y| ,
 \end{align}
where $f_\theta(x)$ is the expected prediction of randomized classifier $f_\theta$. For the unconstrained models, $f_\theta$ is the deterministic prediction.

\input{sections/appdx-exp-robustness.tex}
\input{sections/appdx-exp-acc.tex}
\input{sections/appdx-exp-fairness-level.tex}
\input{sections/appdx-adult-exp.tex}

\input{sections/appdx-exp-train-acc.tex}
\input{sections/appdx-exp-fairness-gap.tex}
\input{sections/appdx-exp-subgroup.tex}

%% file: sections/appdx-exp-robustness.tex

\subsection{Robustness evaluation} \label{appdx:exp-robustness}
In this section, we provide the detailed results about the test accuracy and fairness gap of the target models for both COMPAS and Adult datasets, as discussed in Section~\ref{sec:eval:results}. We show that the fair models are more vulnerable to the poisoning attack compared with the unconstrained model. In addition, the test accuracy and the fairness property of fair models are both compromised.

\input{sections/tab-robustness}
\input{sections/tab-fairness}

\paragraph{Test accuracy}
In the Table~\ref{table:robustness}, we compare the effect of poisoning attacks on unconstrained models (without fairness constraint) with different fair models at different desired fairness levels $\delta$ on two datasets (COMPAS and Adult), when the attacker controls 10\% of the training data i.e., $\epsilon = 0.1$.
The ``Unconstrained Model'' column corresponds to the test accuracy of the unconstrained model.
The ``Fair~\cite{agarwal2018reductions} ($\delta=0.1$)'' and the ``Fair~\cite{agarwal2018reductions} ($\delta=0.01$)'' columns respectively correspond to the test accuracy of the fair models trained with the Reductions approach~\cite{agarwal2018reductions} at $\delta=0.1$ and $\delta=0.01$.
The ``Fair~\cite{hardt2016equality} ($\delta=0$)'' column corresponds to the fair model trained with the Post-processing method~\cite{hardt2016equality} with exact fairness, i.e $\delta=0$.
For each dataset, ``Benign'' row shows the test accuracy for models trained on data without any poisoning attack i.e., $\epsilon = 0$. We compare these with the test accuracy of corresponding models that are learned from poisoned data.

We observe that when the models are trained with fairness, the drop in test accuracy are more significant than when the constraints are absent in both adversarial labeling and adversarial sampling setting.
We notice that our attacks outperform the baseline attacks on both datasets in both adversarial bias setting (adversarial labeling and adversarial sampling).
Even in the adversarial sampling setting, using our proposed attack strategies, the attacker manages to reduce the test accuracy of the target model more than labeling flipping attack. This shows the effectiveness of our attack strategies.
In addition, an increase in the desired fairness level, i.e when $\delta$ decreases, correlates with an increase in the accuracy drop. This shows fair models are more vulnerable to poisoning attacks than unconstrained models.
Note that on Adult, the drops are not as large as they are on COMPAS. However, the constant prediction classifier trained on Adult dataset can achieve 81\% accuracy. In other words,  the fair models trained on poisoned datasets can only perform barely better than constant prediction classifier. Our attacks are still effective on Adult dataset.

The overall results in the Table~\ref{table:robustness} reflect the effectiveness of our strategies and provide evidence that the fair model is more vulnerable than the unconstrained model.

\paragraph{Fairness gap}
In Table~\ref{table:robustness-fairness}, we compare the effect of poisoning attacks on unconstrained model (without fairness constraint) with different fair models at different $\delta$ on two datasets (COMPAS and Adult), when $\epsilon = 0.1$. Similar to the Table~\ref{table:robustness}, the columns 4-7 (``Unconstrained Model'', ``Fair~\cite{agarwal2018reductions} ($\delta=0.1$)'', ``Fair~\cite{agarwal2018reductions} ($\delta=0.01$)'', ``Fair~\cite{hardt2016equality} ($\delta=0$))  show the fairness gap on test datasets $\Delta(\theta;\D_\test)$ of the unconstrained models, of the fair models trained with the Reductions approach~\cite{agarwal2018reductions} at $\delta=0.1$ and $\delta=0.01$ and of fair models trained with the Post-processing method~\cite{hardt2016equality} with exact fairness, i.e $\delta=0$ , respectively.
For each dataset, ``Benign'' row shows the fairness gap on the test dataset for models trained on data without any poisoning attack i.e., $\epsilon = 0$.
We compare these with the fairness gap on test dataset of corresponding models learned from poisoned data.

We notice that the fairness gap of fair models trained on the poisoned data is much larger than those of fair models trained on clean data (as shown in ``Benign'' row). This implies fair models trained on poisoned data become less fair on test dataset when attacks are present in both adversarial sampling bias and adversarial labeling bias setting.

Interestingly, the fairness gap of the fair models is larger than that of the unconstrained model when  adversarial bias are present in the training dataset. In addition, an increase in the desired fairness level, i.e when $\delta$ decreases, is associated with an increase in the fairness gap on the test dataset. This shows that not only does poisoning attacks can cause accuracy drops, but they are also able to make fair models more discriminatory on test data than unconstrained models.

%% file: sections/tab-robustness.tex

{\renewcommand{\arraystretch}{1.15}
\begin{table}[t!]

\caption{{\textbf{Test accuracy of attacked models under adversarial bias} -- COMPAS and Adult datasets, for $\epsilon=0.1$.   The numbers reflect test accuracy (defined in \eqref{eq:accuracy}) of attacked models. The lower the numbers are, the more effective the attacks are and the less robust the models are against the attacks. The numbers in bold are the \textit{smallest} accuracies of the target models against different attack algorithms in the adversarial bias or adversarial labeling setting. }
}
\label{table:robustness}
\centering
{\footnotesize
\begin{tabularx}{\columnwidth}{l l >{\centering\arraybackslash}X >{\centering\arraybackslash}X >{\centering\arraybackslash}X >{\centering\arraybackslash}X}
\toprule
Dataset & Attacks& Unconstrained Model & Fair~\cite{agarwal2018reductions} ($\delta=0.1$) & Fair~\cite{agarwal2018reductions} $(\delta=0.01)$ & Fair~\cite{hardt2016equality} ($\delta=0$)\\ \hline \hline
\multirow{10}{*}{COMPAS }

&Benign	 &93.7$\pm$5.6&94.0$\pm$5.6&93.5$\pm$1.7&87.4$\pm$3.8\\ \cline{2-6}

&Random sampling &94.3$\pm$0.9&92.1$\pm$1.6&89.2$\pm$1.6&84.3$\pm$1.1  \\

&Hard examples&94.2$\pm$1.0&91.5$\pm$1.7&88.2$\pm$2.6&83.7$\pm$1.1 \\

&Label flipping &94.0$\pm$1.0&92.8$\pm$1.4&90.8$\pm$2.0&84.4$\pm$1.2  \\ \cline{2-6}

&Adv. sampling~(Alg.~\ref{algorithm:minmax-fair}, $\lambda = 0$) \cite{steinhardt2017certified}   &\textbf{87.6$\pm$1.6}&81.5$\pm$1.7&78.0$\pm$2.3&\textbf{70.8$\pm$1.6}\\

&Adv. sampling~(Alg.~\ref{algorithm:minmax-fair}, $\lambda = 100\epsilon$)&-&\textbf{80.6$\pm$1.9}&\textbf{73.1$\pm$3.0}&71.6$\pm$1.4\\

&Adv. sampling~(Alg.~\ref{algorithm:OGD}, $\lambda = \epsilon$)&-&81.2$\pm$1.4&77.6$\pm$2.2&73.9$\pm$1.8\\

\cline{2-6}
&Adv. labeling~(Alg.~\ref{algorithm:minmax-fair}, $\lambda = 0$) \cite{steinhardt2017certified}   &\textbf{84.8$\pm$1.7}&\textbf{76.3$\pm$1.8}&73.3$\pm$1.7&\textbf{68.3$\pm$1.7} \\

&Adv. labeling~(Alg.~\ref{algorithm:minmax-fair}, $\lambda = 100\epsilon$)&-&77.6$\pm$1.3&\textbf{71.0$\pm$2.8}&70.1$\pm$1.3  \\

&Adv. labeling~(Alg.~\ref{algorithm:OGD}, $\lambda = \epsilon$)& -&80.1$\pm$3.3&76.6$\pm$3.2&70.5$\pm$1.8 \\

\hline \hline

\multirow{10}{*}{Adult }
&Benign&94.3$\pm$0.3&94.3$\pm$0.3&93.8$\pm$0.3&92.7$\pm$0.4\\\cline{2-6}

&Random sampling&94.3$\pm$0.3&94.3$\pm$0.3&93.7$\pm$0.3&92.3$\pm$0.3 \\

&Hard examples&94.2$\pm$0.3&94.1$\pm$0.3&92.6$\pm$0.4&90.8$\pm$0.4 \\

&Label flipping&93.3$\pm$0.4&91.0$\pm$0.5&89.2$\pm$0.4&88.2$\pm$0.4 \\ \cline{2-6}

&Adv. sampling~(Alg.~\ref{algorithm:minmax-fair}, $\lambda = 0$) \cite{steinhardt2017certified}   &\textbf{94.0$\pm$0.3}&93.1$\pm$0.5&\textbf{91.7$\pm$0.6}&89.6$\pm$0.5\\

&Adv. sampling~(Alg.~\ref{algorithm:minmax-fair}, $\lambda = 100\epsilon$) &-&92.5$\pm$0.5&92.2$\pm$0.4&90.1$\pm$0.5\\

&Adv. sampling~(Alg.~\ref{algorithm:OGD}, $\lambda = 0.1\epsilon$)&-&\textbf{92.3$\pm$0.5}&92.3$\pm$0.5&\textbf{89.3$\pm$0.4}\\
 \cline{2-6}

&Adv. labeling~(Alg.~\ref{algorithm:minmax-fair}, $\lambda = 0$) \cite{steinhardt2017certified}   &\textbf{89.3$\pm$0.9}&87.2$\pm$0.6&83.9$\pm$0.5&84.6$\pm$0.6 \\

&Adv. labeling~(Alg.~\ref{algorithm:minmax-fair}, $\lambda =100\epsilon$)&-&\textbf{85.5$\pm$1.2}&\textbf{80.9$\pm$1.7}&\textbf{81.1$\pm$1.6}\\

&Adv. labeling~(Alg.~\ref{algorithm:OGD}, $\lambda = 0.1\epsilon$)&-&87.5$\pm$0.6&83.9$\pm$0.8&84.6$\pm$0.7\\

\bottomrule \\
\end{tabularx}
}

\end{table}
}

%% file: sections/tab-fairness.tex

{\renewcommand{\arraystretch}{1.15}
\begin{table}[t!]
\caption{{\textbf{Fairness gap of attacked models on test data} -- COMPAS and Adult datasets, for $\epsilon=0.1$.  The fairness gap $\Delta$ is defined in \eqref{eq:eo-fairness}.  The numbers reflect how unfair the model is with respect to the protected group in the test data.  For fair models, compare numbers with $\delta$ (the guaranteed fairness gap on training data).  The farther apart $\Delta$ and $\delta$ are, the less the fairness generalization is on test data. The numbers in bold are the \textit{largest} $\Delta$ of the target models against different attack algorithms in the adversarial bias or adversarial labeling setting. }
}
\label{table:robustness-fairness}
\centering
{\footnotesize
\begin{tabularx}{\columnwidth}{l l >{\centering\arraybackslash}X >{\centering\arraybackslash}X >{\centering\arraybackslash}X >{\centering\arraybackslash}X}
\toprule
Dataset & Attacks& Unconstrained Model & Fair~\cite{agarwal2018reductions} ($\delta=0.1$) & Fair~\cite{agarwal2018reductions} $(\delta=0.01)$ & Fair~\cite{hardt2016equality} ($\delta=0$)\\ \hline \hline
\multirow{10}{*}{COMPAS }
&Benign&0.21$\pm$0.07&0.11$\pm$0.06&0.06$\pm$0.04&0.07$\pm$0.04\\ \cline{2-6}
&Random Sampling&0.19$\pm$0.07&0.08$\pm$0.03&0.11$\pm$0.05&0.13$\pm$0.07\\
&Hard examples&0.19$\pm$0.08&0.09$\pm$0.03&0.13$\pm$0.05&0.15$\pm$0.07 \\
&Label flipping &0.23$\pm$0.07&0.09$\pm$0.04&0.07$\pm$0.04&0.10$\pm$0.06\\ \cline{2-6}

&Adv. sampling~(Alg.~\ref{algorithm:minmax-fair}, $\lambda = 0$) \cite{steinhardt2017certified}          &0.26$\pm$0.08    &0.19$\pm$0.07    &0.30$\pm$0.07    &0.27$\pm$0.08 \\
&Adv. sampling~(Alg.~\ref{algorithm:minmax-fair}, $\lambda = 100\epsilon$)&-               &0.29$\pm$0.06    &0.37$\pm$0.09    &0.53$\pm$0.05 \\
&Adv. sampling~(Alg.~\ref{algorithm:OGD}, $\lambda = \epsilon$)   &-               &0.12$\pm$0.07    &0.21$\pm$0.10    &0.25$\pm$0.13 \\
\cline{2-6}

&Adv. labeling~(Alg.~\ref{algorithm:minmax-fair}, $\lambda = 0$)  \cite{steinhardt2017certified}        &\textbf{0.28$\pm$0.08}  &0.13$\pm$0.05  &0.19$\pm$0.08    &0.25$\pm$0.08\\
&Adv. labeling~(Alg.~\ref{algorithm:minmax-fair}, $\lambda = 100\epsilon$)&-                 &\textbf{0.28$\pm$0.05}  &\textbf{0.39$\pm$0.08}   &\textbf{0.55$\pm$0.04}\\
&Adv. labeling~(Alg.~\ref{algorithm:OGD}, $\lambda = \epsilon$)&-    &0.11$\pm$0.06    &0.12$\pm$0.04      &0.13$\pm$0.09\\
\hline \hline

\multirow{9}{*}{Adult }
&Benign&0.07$\pm$0.03&0.07$\pm$0.03&0.04$\pm$0.02&0.03$\pm$0.02 \\ \cline{2-6}
&Random sampling&0.07$\pm$0.03&\textbf{0.07$\pm$0.03}&0.03$\pm$0.02&0.03$\pm$0.02\\
&Hard examples&\textbf{0.08$\pm$0.03}&0.06$\pm$0.03&0.04$\pm$0.02&0.06$\pm$0.03\\
&Label flipping&\textbf{0.08$\pm$0.04}&0.10$\pm$0.04&\textbf{0.24$\pm$0.04}&0.24$\pm$0.04\\ \cline{2-6}

&Adv. sampling~(Alg.~\ref{algorithm:minmax-fair}, $\lambda = 0$) \cite{steinhardt2017certified}   &0.06$\pm$0.03&0.03$\pm$0.02&\textbf{0.12$\pm$0.03}&0.17$\pm$0.04 \\

&Adv. sampling~(Alg.~\ref{algorithm:minmax-fair}, $\lambda = 100\epsilon$&-&0.06$\pm$0.03&0.07$\pm$0.02&\textbf{0.19$\pm$0.04} \\

&Adv. sampling~(Alg.~\ref{algorithm:OGD}, $\lambda = 0.1\epsilon$)&-&0.05$\pm$0.03&0.05$\pm$0.02&0.14$\pm$0.05\\

\cline{2-6}
&Adv. labeling~(Alg.~\ref{algorithm:minmax-fair}, $\lambda = 0$) \cite{steinhardt2017certified}   &0.06$\pm$0.04&0.07$\pm$0.03&0.21$\pm$0.04&0.27$\pm$0.04\\

&Adv. labeling~(Alg.~\ref{algorithm:minmax-fair}, $\lambda = 100\epsilon$)&-&\textbf{0.18$\pm$0.06}&0.09$\pm$0.08&0.09$\pm$0.19\\

&Adv. labeling~(Alg.~\ref{algorithm:OGD}, $\lambda = 0.1\epsilon$)&-&0.11$\pm$0.04&0.23$\pm$0.05&\textbf{0.34$\pm$0.04}\\

\bottomrule \\
\end{tabularx}
}
\end{table}
}

%% file: sections/appdx-exp-acc.tex

\subsection{Conflict between fairness and robustness.} \label{appdx:exp-attack}

\input{plot/accuracy_COMPAS_appendix.tex}

\input{plot/acc_labeling_ADULT.tex}

In Figure~\ref{fig:accuracy-compas} and Figure~\ref{fig:labeling-attacks-adult}, we compare the test accuracy of target model at different fractions of poisoning data selected using all the attack strategies on COMPAS and Adult dataset respectively.

On COMPAS dataset, for unconstrained models, only Algorithm~\ref{algorithm:minmax-fair}  ($\lambda = 0$) has an effect, causing a 10\% drop in accuracy when $\epsilon = 0.2$ and the accuracy hardly decreases when $\epsilon$ increases from $0.1$ to $0.2$ in the adversarial labeling bias setting. We observe a similar result for the adversarial sampling bias.
On fair models, we can observe that both Algorithm~\ref{algorithm:OGD} and Algorithm~\ref{algorithm:minmax-fair} (with $\lambda = 100\epsilon$ and 0) have a significantly better performance than Label flipping attack and adding Hard examples.
The performance of Algorithm~\ref{algorithm:minmax-fair} is better than that of Algorithm~\ref{algorithm:OGD}, as Algorithm~\ref{algorithm:OGD} uses a surrogate linear loss for evaluating the fairness gap $\Delta(\theta; \D_\clean \cup \D_\poisoning)$, whereas Algorithm~\ref{algorithm:minmax-fair} computes the exact fairness gap.  
Algorithm~\ref{algorithm:minmax-fair} with $\lambda = 100\epsilon$ and 0 show similar results at smaller fractions of poisoning data ($\epsilon < 0.1$) and start to diverge at higher values of $\epsilon$ due to increase in contribution of fairness gap term, with the former approaching the Constant classifier baseline at $\epsilon = 0.2$.

For the Adult dataset, notice that the Constant prediction baseline has good accuracy ($>$80\%).
Hence, the relative accuracy drop on the Adult dataset is not as significant as that on the COMPAS dataset.
However, we can still observe similar results that compared to unconstrained models, fair models witness a greater accuracy drop, with our proposed attacks perform significantly better than the baselines. The three algorithms have similar results both when the fair models are trained with~\cite{hardt2016equality} and~\cite{agarwal2018reductions}.
We observe that the plots for Algorithm~\ref{algorithm:minmax-fair} fluctuate in both adversarial sampling and adversarial labeling settings when $\epsilon$ increases. The detailed explanations are presented in Appendix~\ref{appdx:adult-exp}.

%% file: plot/accuracy_COMPAS_appendix.tex
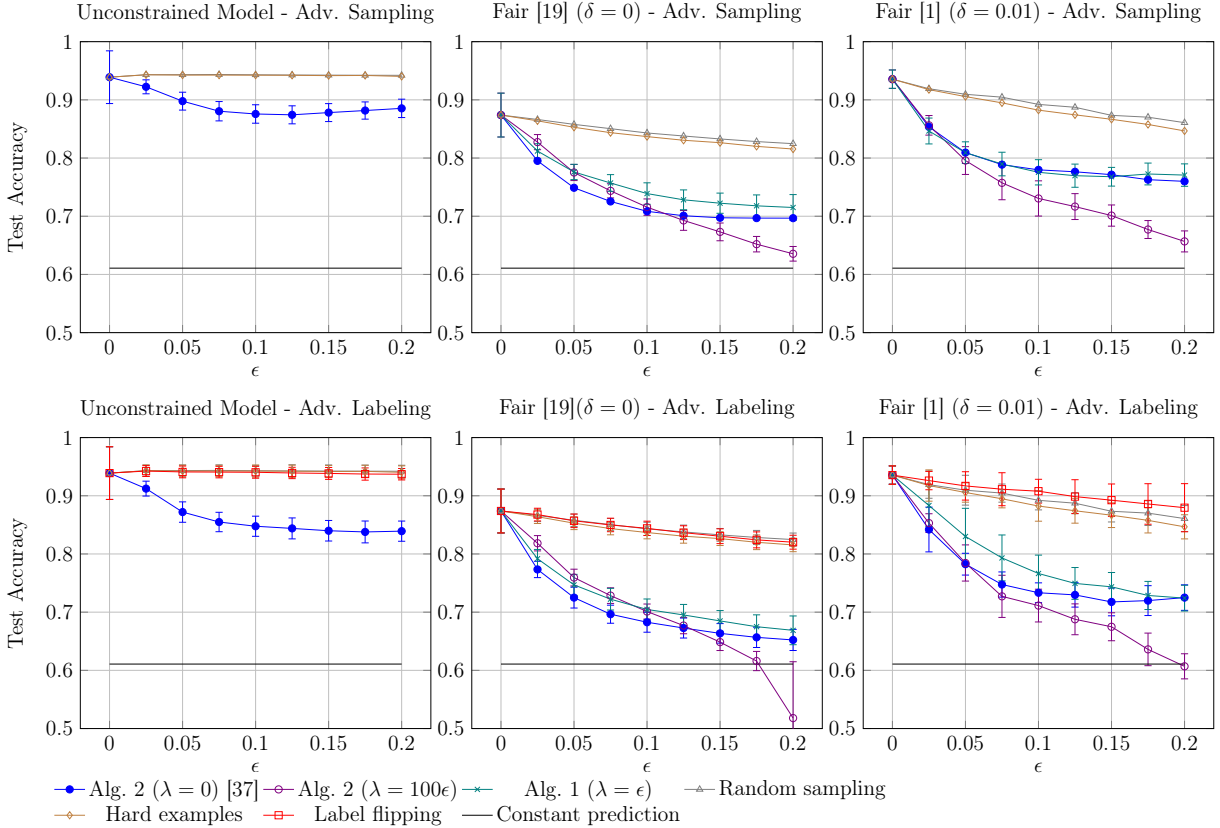
\begin{figure*}[t!]
 \centering
 \resizebox{\columnwidth}{!}{%
{\large
\begin{tabular}{l}
      \begin{tikzpicture}
     \begin{axis}
       [name=plot11, legend style={font=\small},title= { Unconstrained Model - Adv. Sampling},xlabel={$\epsilon$ },ylabel={Test Accuracy}, ymin = 0.5, ymax =1, xtick={0,0.05,0.1,0.15,0.2}, xticklabels={0,0.05,0.1,0.15,0.2},grid = major]
       		\addplot[solid, blue, mark = *, mark options={solid,}, error bars/.cd, y dir=both, y explicit] table[skip first n=1,x index=1, y index=2, y error index=12, col sep=comma] {"plot2/reduction_csv/attacker-nfminmax-0.01.csv"};
          	\addplot[solid, gray, mark = triangle, mark options={solid,}] table[skip first n=1,x index=1, y index=2, col sep=comma] {"plot2/reduction_csv/attacker-usel-0.01.csv"};
          	\addplot[solid, brown, mark = diamond, mark options={solid,}] table[skip first n=1,x index=1, y index=2, col sep=comma] {"plot2/reduction_csv/noise-usel-0.01.csv"};
			\addplot[no marks] table[skip first n=1,x index=1, y index=2, col sep=comma] {"plot2/baseline.csv"};
     \end{axis}

      \begin{axis}
        [name=plot12, at=(plot11.south east), anchor=south west, xshift=0.8cm,title= {Fair~\cite{hardt2016equality} ($\delta = 0$) - Adv. Sampling},xlabel={$\epsilon$},ylabel={}, ymin = 0.5, ymax =1,xtick={0,0.05,0.1,0.15,0.2}, xticklabels={0,0.05,0.1,0.15,0.2},  grid = major,
        ]

         \addplot[solid, blue, mark = *, mark options={solid,}] table[skip first n=1,x index=1, y index=3, col sep=comma] {"plot2/PP_csv/attacker-nfminmax-0.01.csv"};

       \addplot[violet, mark = o, error bars/.cd, y dir=both, y explicit] table[skip first n=1,x index=1, y index=3, y error index=13, col sep=comma] {"plot2/PP_csv/attacker-eominmaxlossnf100-0.01.csv"};
        \addplot[teal, mark = x, error bars/.cd, y dir=both, y explicit] table[skip first n=1,x index=1, y index=3, y error index=13, col sep=comma] {"plot2/PP_csv/attacker-newregnf-1.0-0.01.csv"};

        \addplot[solid, gray, mark = triangle, mark options={solid,}] table[skip first n=1,x index=1, y index=3, col sep=comma] {"plot2/PP_csv/attacker-usel-0.01.csv"};
         \addplot[solid, brown, mark = diamond, mark options={solid,}] table[skip first n=1,x index=1, y index=3, col sep=comma] {"plot2/PP_csv/noise-usel-0.01.csv"};

		\addplot[no marks] table[skip first n=1,x index=1, y index=2, col sep=comma] {"plot2/baseline.csv"};

      \end{axis}

      \begin{axis}
        [name=plot13,
    at=(plot12.south east), anchor=south west, xshift=0.8cm, legend style={font=\small},title= {Fair~\cite{agarwal2018reductions} ($\delta = 0.01$) - Adv. Sampling},xlabel={$\epsilon$},ylabel={}, ymin = 0.5, ymax =1,xtick={0,0.05,0.1,0.15,0.2}, xticklabels={0,0.05,0.1,0.15,0.2},  grid = major]

         \addplot[solid, blue, mark = *, mark options={solid,}] table[skip first n=1,x index=1, y index=3, col sep=comma] {"plot2/reduction_csv/attacker-nfminmax-0.01.csv"};

        \addplot[violet, mark = o, error bars/.cd, y dir=both, y explicit] table[skip first n=1,x index=1, y index=3, y error index=13, col sep=comma] {"plot2/reduction_csv/attacker-eominmaxlossnf100-0.01.csv"};
         \addplot[teal, mark = x, error bars/.cd, y dir=both, y explicit] table[skip first n=1,x index=1, y index=3, y error index=13, col sep=comma] {"plot2/reduction_csv/attacker-newregnf-1.0-0.01.csv"};

        \addplot[solid, gray, mark = triangle, mark options={solid,}] table[skip first n=1,x index=1, y index=3, col sep=comma] {"plot2/reduction_csv/attacker-usel-0.01.csv"};
         \addplot[solid, brown, mark = diamond, mark options={solid,}] table[skip first n=1,x index=1, y index=3, col sep=comma] {"plot2/reduction_csv/noise-usel-0.01.csv"};

		\addplot[no marks] table[skip first n=1,x index=1, y index=2, col sep=comma] {"plot2/baseline.csv"};

      \end{axis}
      \end{tikzpicture}\\
      \begin{tikzpicture}
     \begin{axis}
       [name=plot1,  legend style={font=\small},title= {Unconstrained Model - Adv. Labeling},xlabel={$\epsilon$ },ylabel={Test Accuracy}, ymin = 0.5, ymax =1, xtick={0,0.05,0.1,0.15,0.2}, xticklabels={0,0.05,0.1,0.15,0.2}, grid = major]

       \addplot[solid, blue, mark = *, mark options={solid,}, error bars/.cd, y dir=both, y explicit] table[skip first n=1,x index=1, y index=2, y error index=12, col sep=comma] {"plot2/reduction_csv/attacker-minmax2-0.01.csv"};

          \addplot[solid, gray, mark = triangle, mark options={solid,}, error bars/.cd, y dir=both, y explicit] table[skip first n=1,x index=1, y index=2, y error index=12, col sep=comma] {"plot2/reduction_csv/attacker-usel-0.01.csv"};
          \addplot[solid, brown, mark = diamond, mark options={solid,}, error bars/.cd, y dir=both, y explicit] table[skip first n=1,x index=1, y index=2, y error index=12, col sep=comma] {"plot2/reduction_csv/noise-usel-0.01.csv"};
          \addplot[solid, red, mark = square, mark options={solid,},error bars/.cd, y dir=both, y explicit] table[skip first n=1,x index=1, y index=2, y error index=12, col sep=comma] {"plot2/reduction_csv/attacker-uflip-0.01.csv"};
		      \addplot[no marks] table[skip first n=1,x index=1, y index=2, col sep=comma] {"plot2/baseline.csv"};

     \end{axis}

      \begin{axis}
        [name=plot2, at=(plot1.south east), anchor=south west, xshift=0.8cm,title= {Fair \cite{hardt2016equality}($\delta = 0$) - Adv. Labeling },xlabel={$\epsilon$},ylabel={}, ymin = 0.5, ymax =1,xtick={0,0.05,0.1,0.15,0.2}, xticklabels={0,0.05,0.1,0.15,0.2},  grid = major, legend entries = {Alg.~\ref{algorithm:minmax-fair} ($\lambda = 0$) \cite{steinhardt2017certified}   ,Alg.~\ref{algorithm:minmax-fair} ($\lambda = 100\epsilon$), Alg.~\ref{algorithm:OGD} ($\lambda = \epsilon$),  Random sampling, Hard examples, Label flipping,  Constant prediction},legend style={at={(0., -0.15)}, anchor=north, draw=none, legend columns=4}]

		\addplot[solid, blue, mark = *, mark options={solid,}, error bars/.cd, y dir=both, y explicit] table[skip first n=1,x index=1, y index=3,y error index=13,  col sep=comma] {"plot2/PP_csv/attacker-minmax2-0.01.csv"};

		\addplot[violet, mark = o, error bars/.cd, y dir=both, y explicit] table[skip first n=1,x index=1, y index=3, y error index=13, col sep=comma] {"plot2/PP_csv/attacker-eominmaxloss-0.01.csv"};

		\addplot[teal, mark = x, error bars/.cd, y dir=both, y explicit] table[skip first n=1,x index=1, y index=3, y error index=13, col sep=comma] {"plot2/PP_csv/attacker-newreg-1.0-0.01.csv"};

		\addplot[solid, gray, mark = triangle, mark options={solid,}, error bars/.cd, y dir=both, y explicit]table[skip first n=1,x index=1, y index=3, y error index=13, col sep=comma] {"plot2/PP_csv/attacker-usel-0.01.csv"};
  	 \addplot[solid, brown, mark = diamond, mark options={solid,}, error bars/.cd, y dir=both, y explicit] table[skip first n=1,x index=1, y index=3, y error index=13, col sep=comma] {"plot2/PP_csv/noise-usel-0.01.csv"};
   \addplot[solid, red, mark = square, mark options={solid,}, error bars/.cd, y dir=both, y explicit] table[skip first n=1,x index=1, y index=3, y error index=13, col sep=comma] {"plot2/PP_csv/attacker-uflip-0.01.csv"};
		\addplot[no marks] table[skip first n=1,x index=1, y index=2, col sep=comma] {"plot2/baseline.csv"};

      \end{axis}

      \begin{axis}
        [name=plot3,
    at=(plot2.south east), anchor=south west, xshift=0.8cm, legend style={font=\small},title= {Fair \cite{agarwal2018reductions} ($\delta = 0.01$) - Adv. Labeling}, xlabel={ $\epsilon$},ylabel={}, ymin = 0.5, ymax =1,xtick={0,0.05,0.1,0.15,0.2}, xticklabels={0,0.05,0.1,0.15,0.2},  grid = major]

		 \addplot[solid, blue, mark = *, mark options={solid,}, error bars/.cd, y dir=both, y explicit] table[skip first n=1,x index=1, y index=3, y error index=13, col sep=comma] {"plot2/reduction_csv/attacker-minmax2-0.01.csv"};

        \addplot[violet, mark = o, error bars/.cd, y dir=both, y explicit] table[skip first n=1,x index=1, y index=3, y error index=13, col sep=comma] {"plot2/reduction_csv/attacker-eominmaxloss-0.01.csv"};
       \addplot[teal, mark = x, error bars/.cd, y dir=both, y explicit] table[skip first n=1,x index=1, y index=3, y error index=13, col sep=comma] {"plot2/reduction_csv/attacker-newreg-1.0-0.01.csv"};

         \addplot[solid, gray, mark = triangle, mark options={solid,}, error bars/.cd, y dir=both, y explicit] table[skip first n=1,x index=1, y index=3, y error index=13, col sep=comma] {"plot2/reduction_csv/attacker-usel-0.01.csv"};
         \addplot[solid, brown, mark = diamond, mark options={solid,}, error bars/.cd, y dir=both, y explicit] table[skip first n=1,x index=1, y index=3, y error index=13, col sep=comma] {"plot2/reduction_csv/noise-usel-0.01.csv"};
       \addplot[solid, red, mark = square, mark options={solid,}, error bars/.cd, y dir=both, y explicit] table[skip first n=1,x index=1, y index=3, y error index=13, col sep=comma] {"plot2/reduction_csv/attacker-uflip-0.01.csv"};
		\addplot[no marks] table[skip first n=1,x index=1, y index=2, col sep=comma] {"plot2/baseline.csv"};

      \end{axis}
      \end{tikzpicture}
      \end{tabular}
}}
\caption{Test accuracy of unconstrained and fair models under data poisoning attacks -- COMPAS dataset.
		The x-axis $\epsilon$ is the ratio of the size of poisoning dataset $\D_\poisoning$ to the size of clean dataset $\D_\clean$, and reflects the contamination level of training set.  We compare the impact of adversarial bias with baselines and poisoning attacks against unconstrained models, for various $\epsilon$.  The difference between test accuracy at $\epsilon=0$ (benign setting) and larger $\epsilon$ values reflects the impact of the attack.  Constant prediction always outputs the majority label in clean dataset.
}
\label{fig:accuracy-compas}
\end{figure*}

%% file: plot/acc_labeling_ADULT.tex
\begin{figure*}[t!]
 \centering
 \resizebox{\columnwidth}{!}{%

{\large
\begin{tabular}{l}
\begin{tikzpicture}
     \begin{axis}
       [name=plot1, legend style={font=\small},title= {Unconstrained Model - Adv. Sampling},xlabel={$\epsilon$},ylabel={Test Accuracy}, ymin = 0.7, ymax =1, xtick={0,0.05,0.1,0.15,0.2}, xticklabels={0,0.05,0.1,0.15,0.2}, grid = major]

       \addplot[solid, blue, mark = *, mark options={solid,}, error bars/.cd, y dir=both, y explicit] table[skip first n=1,x index=1, y index=2, y error index=12, col sep=comma] {"plot2/adults_reduction_csv/attacker-nfminmax-0.01.csv"};

          \addplot[solid, gray, mark = triangle, mark options={solid,}, error bars/.cd, y dir=both, y explicit] table[skip first n=1,x index=1, y index=2, y error index=12, col sep=comma] {"plot2/adults_reduction_csv/attacker-usel-0.01.csv"};
		\addplot[solid, brown, mark = diamond, mark options={solid,}] table[skip first n=1,x index=1, y index=2, col sep=comma] {"plot2/adults_reduction_csv/noise-usel-0.01.csv"};

          \addplot[no marks] table[skip first n=1,x index=1, y index=2, col sep=comma] {"plot2/adult_baseline.csv"};

     \end{axis}

      \begin{axis}
        [name=plot2, at=(plot1.south east), anchor=south west, xshift=0.8cm,title= {Fair \cite{hardt2016equality} ($\delta = 0$) - Adv. Sampling},ylabel={}, ymin = 0.7, ymax =1,xtick={0,0.05,0.1,0.15,0.2}, xlabel={ $\epsilon$},xticklabels={0,0.05,0.1,0.15,0.2},  grid = major,]

        \addplot[solid, blue, mark = *, mark options={solid,},  error bars/.cd, y dir=both, y explicit] table[skip first n=1,x index=1, y index=3, y error index=13, col sep=comma] {"plot2/adults_PP_csv/attacker-nfminmax-0.01.csv"};

        \addplot[violet, mark = o, error bars/.cd, y dir=both, y explicit] table[skip first n=1,x index=1, y index=3, y error index=13, col sep=comma] {"plot2/adults_PP_csv/attacker-eominmaxlossnf-0.01.csv"};
		\addplot[teal, mark = x, error bars/.cd, y dir=both, y explicit] table[skip first n=1,x index=1, y index=3, y error index=13, col sep=comma] {"plot2/adults_PP_csv/attacker-newregnf-0.1-0.01.csv"};
		\addplot[solid, gray, mark = triangle, mark options={solid,},error bars/.cd, y dir=both, y explicit] table[skip first n=1,x index=1, y index=3,y error index=13,  col sep=comma] {"plot2/adults_PP_csv/attacker-usel-0.01.csv"};
       	\addplot[solid, brown, mark = diamond, mark options={solid,}] table[skip first n=1,x index=1, y index=3, col sep=comma] {"plot2/adults_PP_csv/noise-usel-0.01.csv"};
		\addplot[no marks] table[skip first n=1,x index=1, y index=2, col sep=comma] {"plot2/adult_baseline.csv"};

      \end{axis}

      \begin{axis}
        [name=plot3,
    at=(plot2.south east), anchor=south west, xshift=0.8cm, legend style={font=\small},title= { Fair \cite{agarwal2018reductions} ($\delta = 0.01$) - Adv. Sampling}, xlabel={ $\epsilon$ },ylabel={}, ymin = 0.7, ymax =1,xtick={0,0.05,0.1,0.15,0.2}, xticklabels={0,0.05,0.1,0.15,0.2},  grid = major]

		 \addplot[solid, blue, mark = *, mark options={solid,}, error bars/.cd, y dir=both, y explicit] table[skip first n=1,x index=1, y index=3, y error index=13, col sep=comma] {"plot2/adults_reduction_csv/attacker-nfminmax-0.01.csv"};

        \addplot[violet, mark = o, error bars/.cd, y dir=both, y explicit] table[skip first n=1,x index=1, y index=3, y error index=13, col sep=comma] {"plot2/adults_reduction_csv/attacker-eominmaxlossnf-0.01.csv"};

       \addplot[teal, mark = x, error bars/.cd, y dir=both, y explicit] table[skip first n=1,x index=1, y index=3, y error index=13, col sep=comma] {"plot2/adults_reduction_csv/attacker-newregnf-0.1-0.01.csv"};
         \addplot[solid, gray, mark = triangle, mark options={solid,}, error bars/.cd, y dir=both, y explicit] table[skip first n=1,x index=1, y index=3, y error index=13, col sep=comma] {"plot2/adults_reduction_csv/attacker-usel-0.01.csv"};
        \addplot[solid, brown, mark = diamond, mark options={solid,}] table[skip first n=1,x index=1, y index=3, col sep=comma] {"plot2/adults_reduction_csv/noise-usel-0.01.csv"};
\addplot[no marks] table[skip first n=1,x index=1, y index=2, col sep=comma] {"plot2/adult_baseline.csv"};

      \end{axis}
      \end{tikzpicture} \\
     \begin{tikzpicture}
     \begin{axis}
       [name=plot1, legend style={font=\small},title= {Unconstrained Model - Adv. Labeling},xlabel={$\epsilon$},ylabel={Test Accuracy}, ymin = 0.7, ymax =1, xtick={0,0.05,0.1,0.15,0.2}, xticklabels={0,0.05,0.1,0.15,0.2}, grid = major]

       \addplot[solid, blue, mark = *, mark options={solid,}, error bars/.cd, y dir=both, y explicit] table[skip first n=1,x index=1, y index=2, y error index=12, col sep=comma] {"plot2/adults_reduction_csv/attacker-minmax2-0.01.csv"};

          \addplot[solid, gray, mark = triangle, mark options={solid,}, error bars/.cd, y dir=both, y explicit] table[skip first n=1,x index=1, y index=2, y error index=12, col sep=comma] {"plot2/adults_reduction_csv/attacker-usel-0.01.csv"};
		 \addplot[solid, brown, mark = diamond, mark options={solid,}, error bars/.cd, y dir=both, y explicit] table[skip first n=1,x index=1, y index=2, y error index=12, col sep=comma] {"plot2/adults_reduction_csv/noise-usel-0.01.csv"};
         \addplot[solid, red, mark = square, mark options={solid,},error bars/.cd, y dir=both, y explicit] table[skip first n=1,x index=1, y index=2, y error index=12, col sep=comma] {"plot2/adults_reduction_csv/attacker-uflip-0.01.csv"};
          \addplot[no marks] table[skip first n=1,x index=1, y index=2, col sep=comma] {"plot2/adult_baseline.csv"};

     \end{axis}

      \begin{axis}
        [name=plot2, at=(plot1.south east), anchor=south west, xshift=0.8cm,title= {Fair \cite{hardt2016equality} ($\delta = 0$) - Adv. Labeling},ylabel={}, ymin = 0.7, ymax =1,xtick={0,0.05,0.1,0.15,0.2}, xlabel={ $\epsilon$},xticklabels={0,0.05,0.1,0.15,0.2},  grid = major, legend entries = {Alg.~\ref{algorithm:minmax-fair}  ($\lambda = 0$) \cite{steinhardt2017certified}   , Alg.~\ref{algorithm:minmax-fair} ($\lambda = 100\epsilon$), Alg.~\ref{algorithm:OGD} ($\lambda = 0.1 \epsilon$),  Random sampling, Hard examples, Label flipping,
        Constant prediction}, legend style={at={(0., -0.15)}, anchor=north, draw=none, legend columns=4}]

        \addplot[solid, blue, mark = *, mark options={solid,},  error bars/.cd, y dir=both, y explicit] table[skip first n=1,x index=1, y index=3, y error index=13, col sep=comma] {"plot2/adults_PP_csv/attacker-minmax2-0.01.csv"};

        \addplot[violet, mark = o, error bars/.cd, y dir=both, y explicit] table[skip first n=1,x index=1, y index=3, y error index=13, col sep=comma] {"plot2/adults_PP_csv/attacker-eominmaxloss-0.01.csv"};
		\addplot[teal, mark = x, error bars/.cd, y dir=both, y explicit] table[skip first n=1,x index=1, y index=3, y error index=13, col sep=comma] {"plot2/adults_PP_csv/attacker-newreg-0.1-0.01.csv"};

		\addplot[solid, gray, mark = triangle, mark options={solid,},error bars/.cd, y dir=both, y explicit] table[skip first n=1,x index=1, y index=3,y error index=13,  col sep=comma] {"plot2/adults_PP_csv/attacker-usel-0.01.csv"};
       	 \addplot[solid, brown, mark = diamond, mark options={solid,}, error bars/.cd, y dir=both, y explicit] table[skip first n=1,x index=1, y index=3, y error index=13, col sep=comma] {"plot2/adults_PP_csv/noise-usel-0.01.csv"};
      \addplot[solid, red, mark = square, mark options={solid,}, error bars/.cd, y dir=both, y explicit] table[skip first n=1,x index=1, y index=3,y error index=13,  col sep=comma] {"plot2/adults_PP_csv/attacker-uflip-0.01.csv"};
		\addplot[no marks] table[skip first n=1,x index=1, y index=2, col sep=comma] {"plot2/adult_baseline.csv"};

      \end{axis}

      \begin{axis}
        [name=plot3,
    at=(plot2.south east), anchor=south west, xshift=0.8cm, legend style={font=\small},title= {Fair \cite{agarwal2018reductions} ($\delta = 0.01$) - Adv. Labeling}, xlabel={ $\epsilon$ },ylabel={}, ymin = 0.7, ymax =1,xtick={0,0.05,0.1,0.15,0.2}, xticklabels={0,0.05,0.1,0.15,0.2},  grid = major]

		 \addplot[solid, blue, mark = *, mark options={solid,}, error bars/.cd, y dir=both, y explicit] table[skip first n=1,x index=1, y index=3, y error index=13, col sep=comma] {"plot2/adults_reduction_csv/attacker-minmax2-0.01.csv"};

        \addplot[violet, mark = o, error bars/.cd, y dir=both, y explicit] table[skip first n=1,x index=1, y index=3, y error index=13, col sep=comma] {"plot2/adults_reduction_csv/attacker-eominmaxloss-0.01.csv"};

       \addplot[teal, mark = x, error bars/.cd, y dir=both, y explicit] table[skip first n=1,x index=1, y index=3, y error index=13, col sep=comma] {"plot2/adults_reduction_csv/attacker-newreg-0.1-0.01.csv"};

         \addplot[solid, gray, mark = triangle, mark options={solid,}, error bars/.cd, y dir=both, y explicit] table[skip first n=1,x index=1, y index=3, y error index=13, col sep=comma] {"plot2/adults_reduction_csv/attacker-usel-0.01.csv"};
      \addplot[solid, brown, mark = diamond, mark options={solid,}, error bars/.cd, y dir=both, y explicit] table[skip first n=1,x index=1, y index=3, y error index=13, col sep=comma] {"plot2/adults_reduction_csv/noise-usel-0.01.csv"};
        \addplot[solid, red, mark = square, mark options={solid,}, error bars/.cd, y dir=both, y explicit] table[skip first n=1,x index=1, y index=3, y error index=13, col sep=comma] {"plot2/adults_reduction_csv/attacker-uflip-0.01.csv"};
\addplot[no marks] table[skip first n=1,x index=1, y index=2, col sep=comma] {"plot2/adult_baseline.csv"};

      \end{axis}
      \end{tikzpicture}
      \end{tabular}
}}
\caption{Test accuracy of unconstrained and fair models under data poisoning attacks -- Adult dataset.
		The x-axis $\epsilon$ is the ratio of the size of poisoning dataset $\D_\poisoning$ to the size of clean dataset $\D_\clean$, and reflects the contamination level of training set.  We compare the impact of adversarial bias with baselines and poisoning attacks against unconstrained models, for various $\epsilon$.  The difference between test accuracy at $\epsilon=0$ (benign setting) and larger $\epsilon$ values reflects the impact of the attack.  Constant prediction always outputs the majority label in clean dataset.
}
\label{fig:labeling-attacks-adult}
\end{figure*}
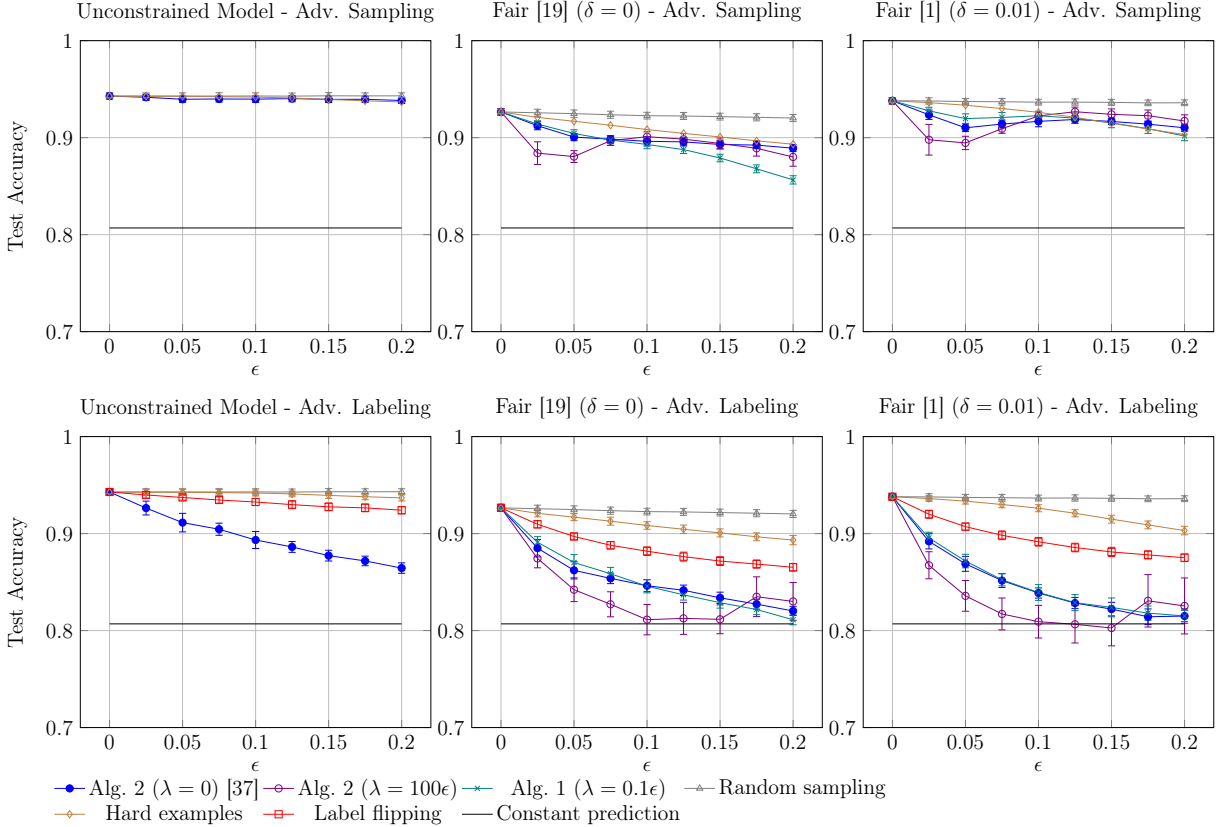

%% file: sections/appdx-exp-fairness-level.tex

\subsection{Effect of fairness level on impact of adversarial bias}\label{appdx:exp-fairness-level}

In Figure~\ref{fig:fairness-measure} and Figure~\ref{fig:fairness-measure-adult-labeling}, we show the effect of fairness level $\delta$ on impact of adversarial bias for COMPAS and Adult dataset respectively.
To measure the influence of fairness level $\delta$, we generate poisoning data using Algorithm~\ref{algorithm:minmax-fair} with $\lambda = 100\epsilon$ and Algorithm~\ref{algorithm:OGD}  with $\lambda = \epsilon$  for both adversarial labeling and adversarial sampling settings on COMPAS dataset. On Adult dataset, we generate poisoning data using Algorithm~\ref{algorithm:minmax-fair} with $\lambda = 100\epsilon$ and Algorithm~\ref{algorithm:OGD}  with $\lambda = 0.1 \epsilon$.
We measure the test accuracy of models learned with different values of fairness level $\delta$.
We can observe that the drop in accuracy for the same fraction of poisoning data is higher for models with stricter fairness constraints (smaller $\delta$). This shows that the more fair a model tries to be, the more vulnerable it becomes to poisoning attacks.
In Figure~\ref{fig:fairness-measure} and Figure~\ref{fig:fairness-measure-adult-labeling}, we also present the majority (the protected group with a larger number of samples) accuracy and minority accuracy. It is clear that the accuracy drop for the majorities is more significant than that for minorities for all the cases. In the Appendix~\ref{appdx:exp-subgroups}, we show that the algorithms choose the points with large loss from the smallest subgroup (subgroups are determined by the protected attribute and the label). As a result, in order to achieve fairness on the poisoned dataset, fair models are more likely to reduce the accuracy of the majority group.

We notice that, in Figure~\ref{fig:fairness-measure-adult-labeling}(a), accuracy plots fluctuate when $\epsilon$ increases, which is not observed in Figure~\ref{fig:fairness-measure-adult-labeling}(b) and Figure~\ref{fig:fairness-measure}. In Appendix~\ref{appdx:adult-exp}, we present the detailed explanations.

\input{plot/fairness-measure.tex}
\input{plot/fairness-measure_ADULT_labeling.tex}

%% file: plot/fairness-measure.tex
\begin{figure*}
  \centering
{ \large
 \begin{subfigure}[b]{\columnwidth}
     \centering
      \resizebox{\columnwidth}{!}{%
         \begin{tikzpicture}
         \begin{axis}
         [name=plot11,title= {Overall - Adv. Sampling},xlabel={$\epsilon$},ylabel={}, ymin = 0.5, ymax =1, ylabel={Test Accuracy}, grid = major , xtick={0,0.05,0.1,0.15,0.2}, xticklabels={0,0.05,0.1,0.15,0.2},legend entries = {Unconstrained model, Fair \cite{agarwal2018reductions} ($\delta =0.1$),Fair \cite{agarwal2018reductions} ($\delta =0.01$), Fair \cite{hardt2016equality}  ($\delta = 0$)}, legend pos=south west
         	    ]
           \addplot[teal,mark = asterisk,error bars/.cd, y dir=both, y explicit]  table[skip first n=1,x index=1, y index=2, col sep=comma, y error index=12] {"plot2/reduction_csv/attacker-eominmaxlossnf100-0.01.csv"};
           \addplot[red,mark = square*,error bars/.cd, y dir=both, y explicit]  table[skip first n=1,x index=1, y index=3,  col sep=comma, y error index=13] {"plot2/reduction_csv/attacker-eominmaxlossnf100-0.1.csv"};
           \addplot[blue, mark = diamond*,error bars/.cd, y dir=both, y explicit] table[skip first n=1,x index=1, y index=3, col sep=comma, y error index=13] {"plot2/reduction_csv/attacker-eominmaxlossnf100-0.01.csv"};
       	  \addplot[brown, mark = *,error bars/.cd, y dir=both, y explicit] table[skip first n=1,x index=1, y index=3,  col sep=comma, y error index=14] {"plot2/PP_csv/attacker-eominmaxlossnf100-0.01.csv"};
         \end{axis}
         \begin{axis}
         [name=plot12, at=(plot11.south east), anchor=south west, xshift=0.8cm, legend style={font=\small},title= {Majority group - Adv. Sampling},xlabel={$\epsilon$},ylabel={}, ymin = 0.5, ymax =1, xtick={0,0.05,0.1,0.15,0.2}, xticklabels={0,0.05,0.1,0.15,0.2},grid = major]

           \addplot[teal,mark = asterisk,error bars/.cd, y dir=both, y explicit] table[skip first n=1,x index=1, y index=4, col sep=comma, y error index=14] {"plot2/reduction_csv/attacker-eominmaxlossnf100-0.01.csv"};
           \addplot[red,mark = square*,error bars/.cd, y dir=both, y explicit] table[skip first n=1,x index=1, y index=6,  col sep=comma, y error index=16] {"plot2/reduction_csv/attacker-eominmaxlossnf100-0.1.csv"};
           \addplot[blue, mark = diamond*,error bars/.cd, y dir=both, y explicit] table[skip first n=1,x index=1, y index=6, col sep=comma, y error index=16] {"plot2/reduction_csv/attacker-eominmaxlossnf100-0.01.csv"};
      	  \addplot[brown, mark = *,error bars/.cd, y dir=both, y explicit] table[skip first n=1,x index=1, y index=6,  col sep=comma, y error index=16] {"plot2/PP_csv/attacker-eominmaxlossnf100-0.01.csv"};

         \end{axis}
           \begin{axis}
           [name=plot13,
    at=(plot12.south east), anchor=south west, xshift=0.8cm,legend style={font=\small},title= { Minority group - Adv. Sampling},xlabel={$\epsilon$},ylabel={}, ymin = 0.5, ymax =1, xtick={0,0.05,0.1,0.15,0.2}, xticklabels={0,0.05,0.1,0.15,0.2},grid = major]

             \addplot[teal,mark = asterisk,error bars/.cd, y dir=both, y explicit] table[skip first n=1,x index=1, y index=5, col sep=comma, y error index=15] {"plot2/reduction_csv/attacker-eominmaxlossnf100-0.01.csv"};
             \addplot[red,mark = square*,error bars/.cd, y dir=both, y explicit] table[skip first n=1,x index=1, y index=7,  col sep=comma, y error index=17] {"plot2/reduction_csv/attacker-eominmaxlossnf100-0.1.csv"};
             \addplot[blue, mark = diamond*,error bars/.cd, y dir=both, y explicit] table[skip first n=1,x index=1, y index=7, col sep=comma, y error index=17] {"plot2/reduction_csv/attacker-eominmaxlossnf100-0.01.csv"};
        	  \addplot[brown, mark = *,error bars/.cd, y dir=both, y explicit] table[skip first n=1,x index=1, y index=7,  col sep=comma, y error index=17] {"plot2/PP_csv/attacker-eominmaxlossnf100-0.01.csv"};

           \end{axis}
         \begin{axis}
         [name=plot4, at=(plot11.south west), anchor=north west, yshift=-1.7cm, title= {Overall - Adv. Labeling},xlabel={$\epsilon$},ylabel={}, ymin = 0.5, ymax =1, ylabel={Test Accuracy}, grid = major , xtick={0,0.05,0.1,0.15,0.2}, xticklabels={0,0.05,0.1,0.15,0.2}]
           \addplot[teal,mark = asterisk,error bars/.cd, y dir=both, y explicit]  table[skip first n=1,x index=1, y index=2, col sep=comma, y error index=12] {"plot2/reduction_csv/attacker-eominmaxloss-0.01.csv"};
           \addplot[red,mark = square*,error bars/.cd, y dir=both, y explicit]  table[skip first n=1,x index=1, y index=3,  col sep=comma, y error index=13] {"plot2/reduction_csv/attacker-eominmaxloss-0.1.csv"};
           \addplot[blue, mark = diamond*,error bars/.cd, y dir=both, y explicit] table[skip first n=1,x index=1, y index=3, col sep=comma, y error index=13] {"plot2/reduction_csv/attacker-eominmaxloss-0.01.csv"};
       	  \addplot[brown, mark = *,error bars/.cd, y dir=both, y explicit] table[skip first n=1,x index=1, y index=3,  col sep=comma, y error index=13] {"plot2/PP_csv/attacker-eominmaxloss-0.01.csv"};
         \end{axis}

         \begin{axis}
         [name=plot5, at=(plot4.south east), anchor=south west, xshift=0.8cm, legend style={font=\small},title= {Majority group - Adv. Labeling},xlabel={$\epsilon$},ylabel={}, ymin = 0.5, ymax =1, xtick={0,0.05,0.1,0.15,0.2}, xticklabels={0,0.05,0.1,0.15,0.2},grid = major]

           \addplot[teal,mark = asterisk,error bars/.cd, y dir=both, y explicit] table[skip first n=1,x index=1, y index=4, col sep=comma, y error index=14] {"plot2/reduction_csv/attacker-eominmaxloss-0.01.csv"};
           \addplot[red,mark = square*,error bars/.cd, y dir=both, y explicit] table[skip first n=1,x index=1, y index=6,  col sep=comma, y error index=16] {"plot2/reduction_csv/attacker-eominmaxloss-0.1.csv"};
           \addplot[blue, mark = diamond*,error bars/.cd, y dir=both, y explicit] table[skip first n=1,x index=1, y index=6, col sep=comma, y error index=16] {"plot2/reduction_csv/attacker-eominmaxloss-0.01.csv"};
      	  \addplot[brown, mark = *,error bars/.cd, y dir=both, y explicit] table[skip first n=1,x index=1, y index=6,  col sep=comma, y error index=16] {"plot2/PP_csv/attacker-eominmaxloss-0.01.csv"};

         \end{axis}
           \begin{axis}
           [name=plot6,
    at=(plot5.south east), anchor=south west, xshift=0.8cm,legend style={font=\small},title= {Minority group - Adv. Labeling},xlabel={$\epsilon$},ylabel={}, ymin = 0.5, ymax =1, xtick={0,0.05,0.1,0.15,0.2}, xticklabels={0,0.05,0.1,0.15,0.2},grid = major]

             \addplot[teal,mark = asterisk,error bars/.cd, y dir=both, y explicit] table[skip first n=1,x index=1, y index=5, col sep=comma, y error index=15] {"plot2/reduction_csv/attacker-eominmaxloss-0.01.csv"};
             \addplot[red,mark = square*,error bars/.cd, y dir=both, y explicit] table[skip first n=1,x index=1, y index=7,  col sep=comma, y error index=17] {"plot2/reduction_csv/attacker-eominmaxloss-0.1.csv"};
             \addplot[blue, mark = diamond*,error bars/.cd, y dir=both, y explicit] table[skip first n=1,x index=1, y index=7, col sep=comma, y error index=17] {"plot2/reduction_csv/attacker-eominmaxloss-0.01.csv"};
        	  \addplot[brown, mark = *,error bars/.cd, y dir=both, y explicit] table[skip first n=1,x index=1, y index=7,  col sep=comma, y error index=17] {"plot2/PP_csv/attacker-eominmaxloss-0.01.csv"};

           \end{axis}
         \end{tikzpicture}}
         \caption{Alg.~\ref{algorithm:minmax-fair} ($\lambda=100\epsilon$)}
         \end{subfigure}

           \vspace{0.5em}

         \begin{subfigure}[b]{\columnwidth}
     \centering
      \resizebox{\columnwidth}{!}{%

         \begin{tikzpicture}
         \begin{axis}
         [name=plot11,title= {Overall - Adv. Sampling},xlabel={$\epsilon$},ylabel={}, ymin = 0.5, ymax =1, ylabel={Test Accuracy}, grid = major , xtick={0,0.05,0.1,0.15,0.2}, xticklabels={0,0.05,0.1,0.15,0.2} ]
           \addplot[teal,mark = asterisk,error bars/.cd, y dir=both, y explicit]  table[skip first n=1,x index=1, y index=2, col sep=comma, y error index=12] {"plot2/reduction_csv/attacker-newregnf-1.0-0.01.csv"};
           \addplot[red,mark = square*,error bars/.cd, y dir=both, y explicit]  table[skip first n=1,x index=1, y index=3,  col sep=comma, y error index=13] {"plot2/reduction_csv/attacker-newregnf-1.0-0.1.csv"};
           \addplot[blue, mark = diamond*,error bars/.cd, y dir=both, y explicit] table[skip first n=1,x index=1, y index=3, col sep=comma, y error index=13] {"plot2/reduction_csv/attacker-newregnf-1.0-0.01.csv"};
       	  \addplot[brown, mark = *,error bars/.cd, y dir=both, y explicit] table[skip first n=1,x index=1, y index=3,  col sep=comma, y error index=14] {"plot2/PP_csv/attacker-newregnf-1.0-0.01.csv"};
         \end{axis}
         \begin{axis}
         [name=plot12, at=(plot11.south east), anchor=south west, xshift=0.8cm, legend style={font=\small},title= {Majority group - Adv. Sampling},xlabel={$\epsilon$},ylabel={}, ymin = 0.5, ymax =1, xtick={0,0.05,0.1,0.15,0.2}, xticklabels={0,0.05,0.1,0.15,0.2},grid = major]

           \addplot[teal,mark = asterisk,error bars/.cd, y dir=both, y explicit] table[skip first n=1,x index=1, y index=4, col sep=comma, y error index=14] {"plot2/reduction_csv/attacker-newregnf-1.0-0.01.csv"};
           \addplot[red,mark = square*,error bars/.cd, y dir=both, y explicit] table[skip first n=1,x index=1, y index=6,  col sep=comma, y error index=16] {"plot2/reduction_csv/attacker-newregnf-1.0-0.1.csv"};
           \addplot[blue, mark = diamond*,error bars/.cd, y dir=both, y explicit] table[skip first n=1,x index=1, y index=6, col sep=comma, y error index=16] {"plot2/reduction_csv/attacker-newregnf-1.0-0.01.csv"};
      	  \addplot[brown, mark = *,error bars/.cd, y dir=both, y explicit] table[skip first n=1,x index=1, y index=6,  col sep=comma, y error index=16] {"plot2/PP_csv/attacker-newregnf-1.0-0.01.csv"};

         \end{axis}
           \begin{axis}
           [name=plot13,
    at=(plot12.south east), anchor=south west, xshift=0.8cm,legend style={font=\small},title= {Minority group - Adv. Sampling},xlabel={$\epsilon$},ylabel={}, ymin = 0.5, ymax =1, xtick={0,0.05,0.1,0.15,0.2}, xticklabels={0,0.05,0.1,0.15,0.2},grid = major]

             \addplot[teal,mark = asterisk,error bars/.cd, y dir=both, y explicit] table[skip first n=1,x index=1, y index=5, col sep=comma, y error index=15] {"plot2/reduction_csv/attacker-newregnf-1.0-0.01.csv"};
             \addplot[red,mark = square*,error bars/.cd, y dir=both, y explicit] table[skip first n=1,x index=1, y index=7,  col sep=comma, y error index=17] {"plot2/reduction_csv/attacker-newregnf-1.0-0.1.csv"};
             \addplot[blue, mark = diamond*,error bars/.cd, y dir=both, y explicit] table[skip first n=1,x index=1, y index=7, col sep=comma, y error index=17] {"plot2/reduction_csv/attacker-newregnf-1.0-0.01.csv"};
        	  \addplot[brown, mark = *,error bars/.cd, y dir=both, y explicit] table[skip first n=1,x index=1, y index=7,  col sep=comma, y error index=17] {"plot2/PP_csv/attacker-newregnf-1.0-0.01.csv"};

           \end{axis}
         \begin{axis}
         [name=plot4, at=(plot11.south west), anchor=north west, yshift=-1.7cm, title= {Overall - Adv. Labeling},xlabel={$\epsilon$},ylabel={}, ymin = 0.5, ymax =1, ylabel={Test Accuracy}, grid = major , xtick={0,0.05,0.1,0.15,0.2}, xticklabels={0,0.05,0.1,0.15,0.2}]
           \addplot[teal,mark = asterisk,error bars/.cd, y dir=both, y explicit]  table[skip first n=1,x index=1, y index=2, col sep=comma, y error index=12] {"plot2/reduction_csv/attacker-newreg-1.0-0.01.csv"};
           \addplot[red,mark = square*,error bars/.cd, y dir=both, y explicit]  table[skip first n=1,x index=1, y index=3,  col sep=comma, y error index=13] {"plot2/reduction_csv/attacker-newreg-1.0-0.1.csv"};
           \addplot[blue, mark = diamond*,error bars/.cd, y dir=both, y explicit] table[skip first n=1,x index=1, y index=3, col sep=comma, y error index=13] {"plot2/reduction_csv/attacker-newreg-1.0-0.01.csv"};
       	  \addplot[brown, mark = *,error bars/.cd, y dir=both, y explicit] table[skip first n=1,x index=1, y index=3,  col sep=comma, y error index=13] {"plot2/PP_csv/attacker-newreg-1.0-0.01.csv"};
         \end{axis}

         \begin{axis}
         [name=plot5, at=(plot4.south east), anchor=south west, xshift=0.8cm, legend style={font=\small},title= {Majority group - Adv. Labeling},xlabel={$\epsilon$},ylabel={}, ymin = 0.5, ymax =1, xtick={0,0.05,0.1,0.15,0.2}, xticklabels={0,0.05,0.1,0.15,0.2},grid = major]

           \addplot[teal,mark = asterisk,error bars/.cd, y dir=both, y explicit] table[skip first n=1,x index=1, y index=4, col sep=comma, y error index=14] {"plot2/reduction_csv/attacker-newreg-1.0-0.01.csv"};
           \addplot[red,mark = square*,error bars/.cd, y dir=both, y explicit] table[skip first n=1,x index=1, y index=6,  col sep=comma, y error index=16] {"plot2/reduction_csv/attacker-newreg-1.0-0.1.csv"};
           \addplot[blue, mark = diamond*,error bars/.cd, y dir=both, y explicit] table[skip first n=1,x index=1, y index=6, col sep=comma, y error index=16] {"plot2/reduction_csv/attacker-newreg-1.0-0.01.csv"};
      	  \addplot[brown, mark = *,error bars/.cd, y dir=both, y explicit] table[skip first n=1,x index=1, y index=6,  col sep=comma, y error index=16] {"plot2/PP_csv/attacker-newreg-1.0-0.01.csv"};

         \end{axis}
           \begin{axis}
           [name=plot6,
    at=(plot5.south east), anchor=south west, xshift=0.8cm,legend style={font=\small},title= {Minority group - Adv. Labeling},xlabel={$\epsilon$},ylabel={}, ymin = 0.5, ymax =1, xtick={0,0.05,0.1,0.15,0.2}, xticklabels={0,0.05,0.1,0.15,0.2},grid = major]

             \addplot[teal,mark = asterisk,error bars/.cd, y dir=both, y explicit] table[skip first n=1,x index=1, y index=5, col sep=comma, y error index=15] {"plot2/reduction_csv/attacker-newreg-1.0-0.01.csv"};
             \addplot[red,mark = square*,error bars/.cd, y dir=both, y explicit] table[skip first n=1,x index=1, y index=7,  col sep=comma, y error index=17] {"plot2/reduction_csv/attacker-newreg-1.0-0.1.csv"};
             \addplot[blue, mark = diamond*,error bars/.cd, y dir=both, y explicit] table[skip first n=1,x index=1, y index=7, col sep=comma, y error index=17] {"plot2/reduction_csv/attacker-newreg-1.0-0.01.csv"};
        	  \addplot[brown, mark = *,error bars/.cd, y dir=both, y explicit] table[skip first n=1,x index=1, y index=7,  col sep=comma, y error index=17] {"plot2/PP_csv/attacker-newreg-1.0-0.01.csv"};

           \end{axis}
         \end{tikzpicture}}
         \caption{Alg.~\ref{algorithm:OGD} ($\lambda=\epsilon$)}
 \end{subfigure}
 }
\vspace{-1.5em}
 \caption{Effect of fairness level $\delta$ on robustness across groups under adversarial sampling and adversarial labeling attacks -- COMPAS dataset.  The majority group (whites) contributes $61\%$ of the training data. 
 }
 \label{fig:fairness-measure}
\end{figure*}
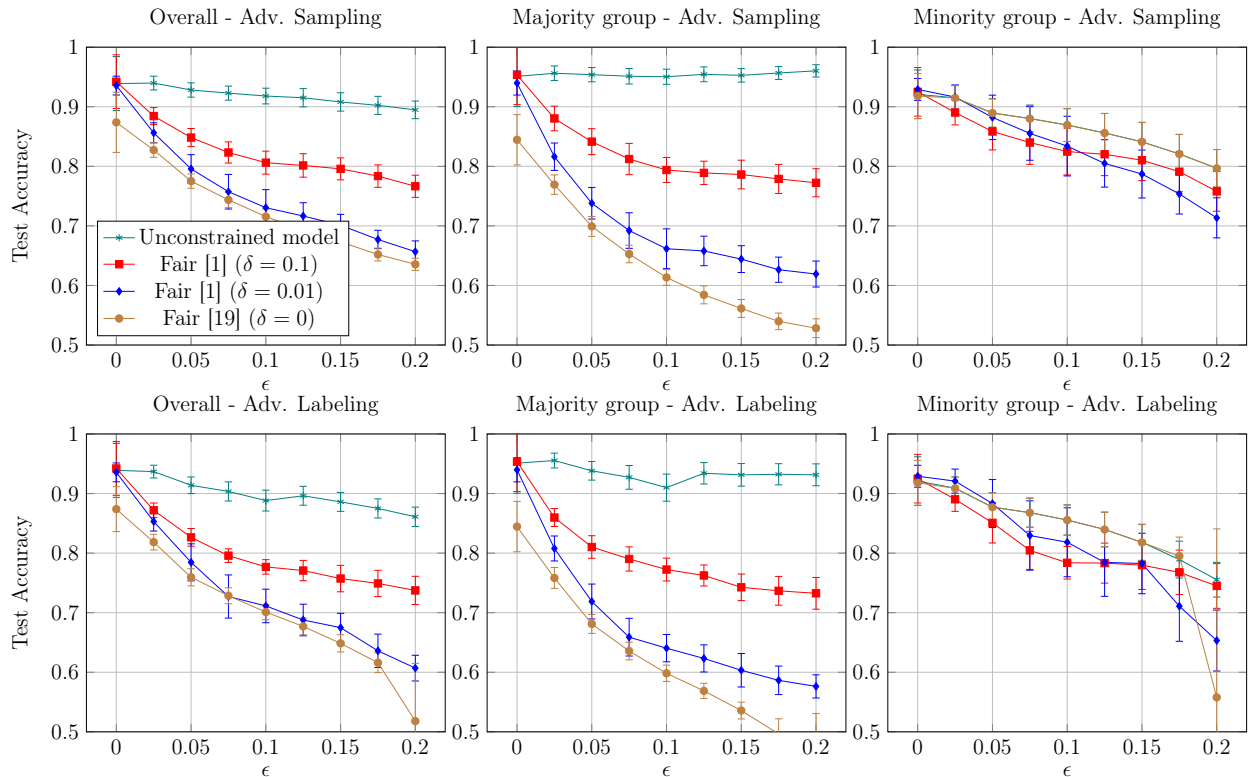
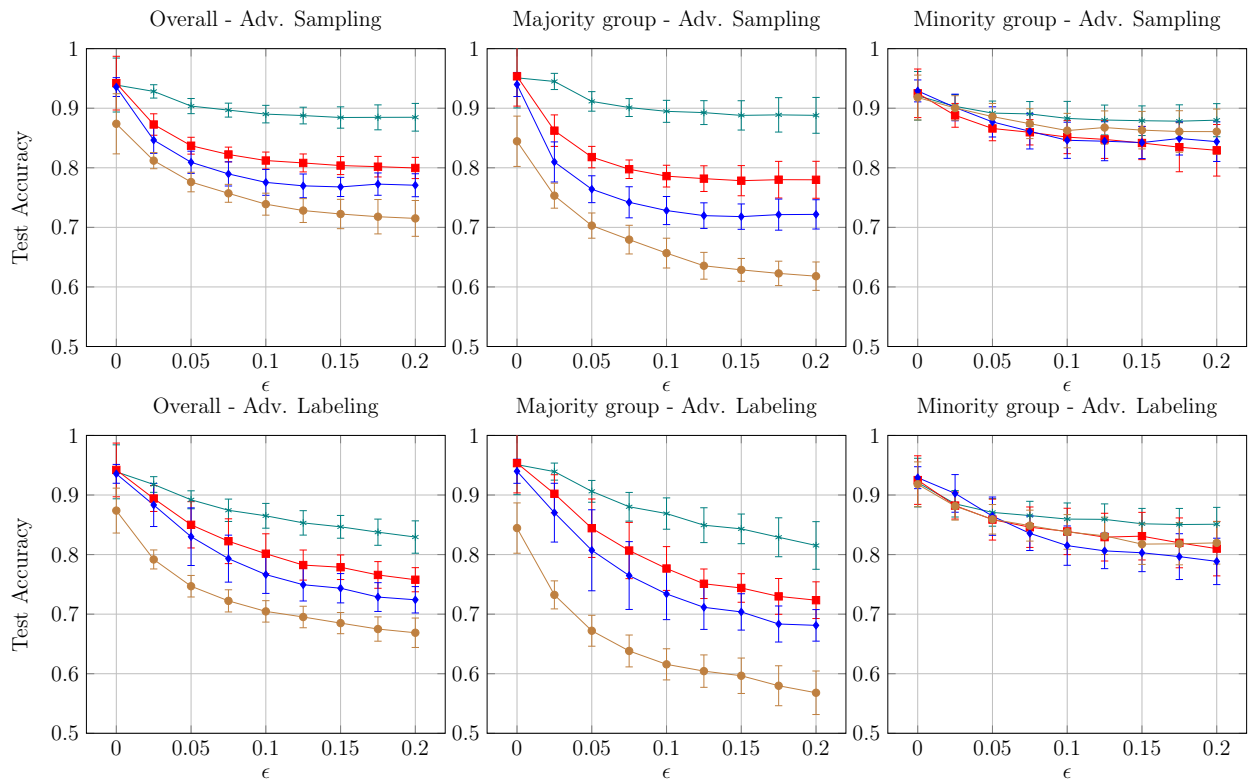

%% file: plot/fairness-measure_ADULT_labeling.tex
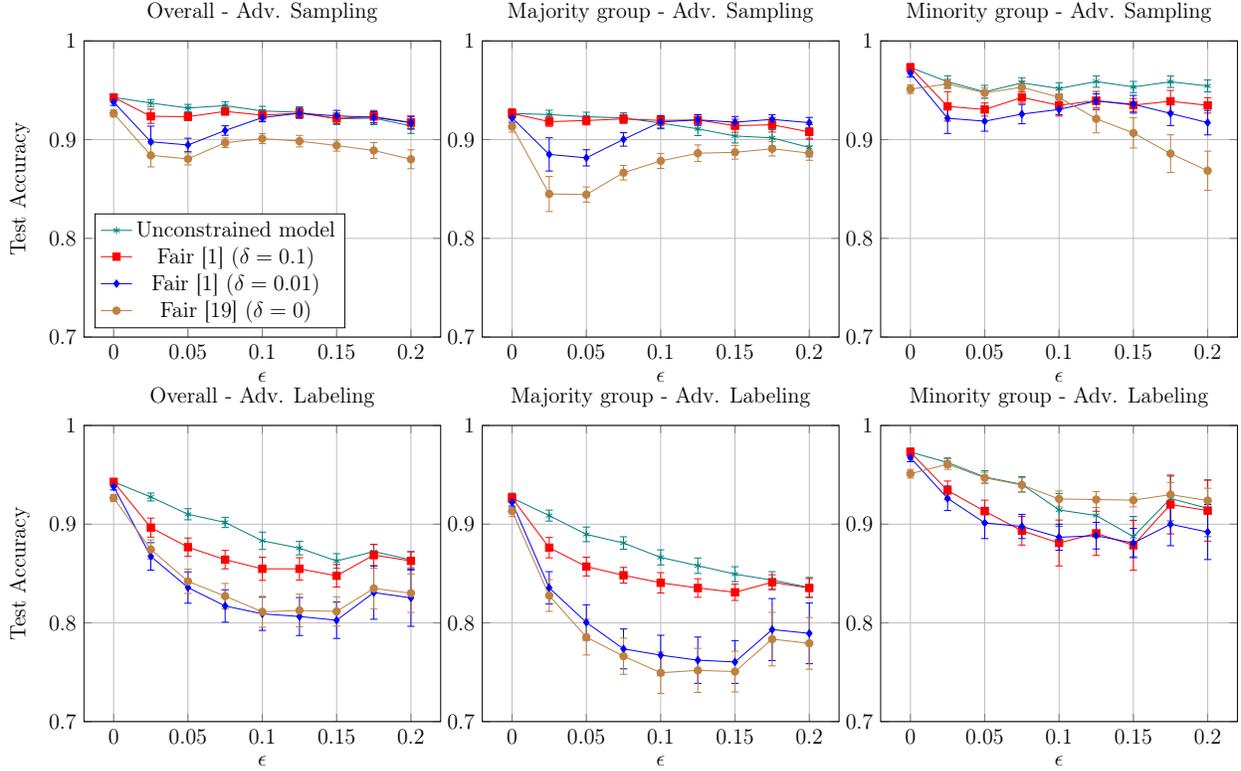
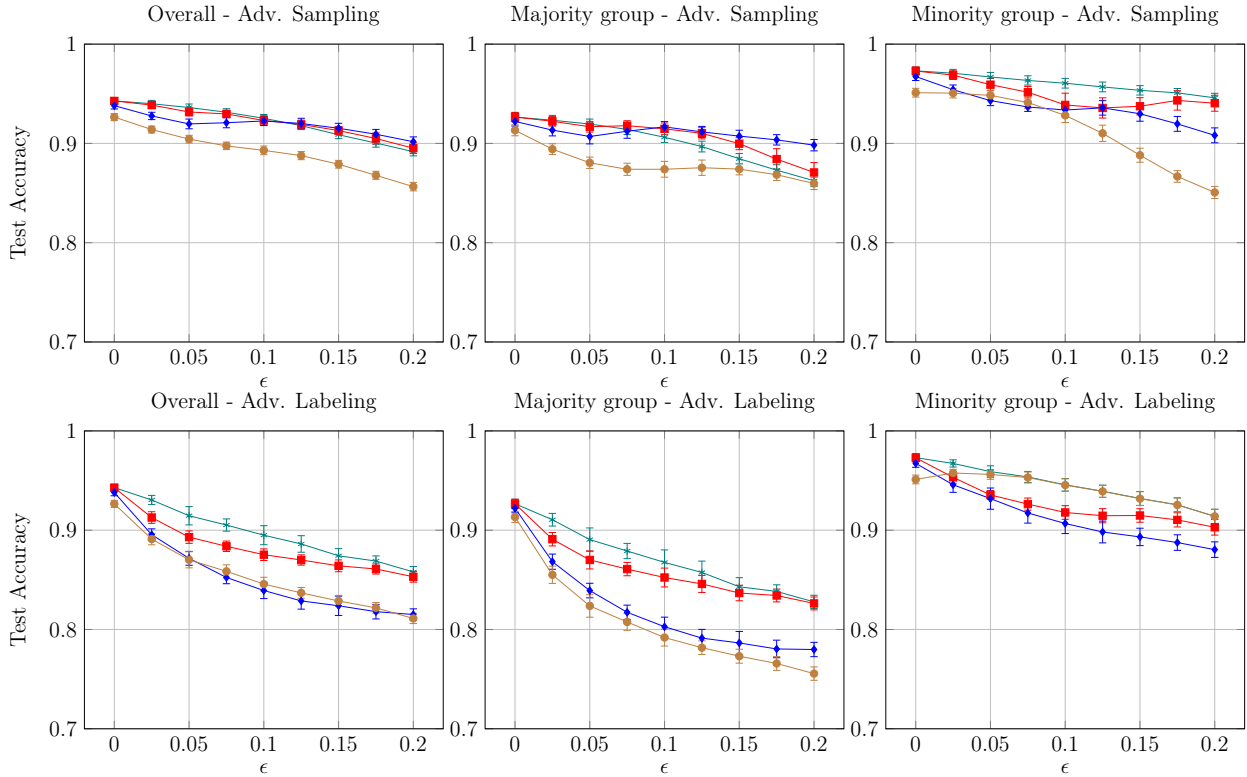
\begin{figure*}
  \centering
 {\large
\begin{subfigure}[b]{\columnwidth}
     \centering
      \resizebox{\columnwidth}{!}{%
 \begin{tikzpicture}
         \begin{axis}
         [name=plot1,title= {Overall - Adv. Sampling},xlabel={$\epsilon$},ylabel={}, ymin = 0.7, ymax =1, ylabel={Test Accuracy}, grid = major , xtick={0,0.05,0.1,0.15,0.2}, xticklabels={0,0.05,0.1,0.15,0.2},legend entries = {Unconstrained model, Fair \cite{agarwal2018reductions} ($\delta =0.1$),Fair \cite{agarwal2018reductions} ($\delta =0.01$), Fair
         	    \cite{hardt2016equality}  ($\delta = 0$)}, legend pos={south west}]
           \addplot[teal,mark = asterisk,error bars/.cd, y dir=both, y explicit]  table[skip first n=1,x index=1, y index=2, col sep=comma, y error index=12] {"plot2/adults_reduction_csv/attacker-eominmaxlossnf-0.01.csv"};
           \addplot[red,mark = square*,error bars/.cd, y dir=both, y explicit]  table[skip first n=1,x index=1, y index=3,  col sep=comma, y error index=13] {"plot2/adults_reduction_csv/attacker-eominmaxlossnf-0.1.csv"};
           \addplot[blue, mark = diamond*,error bars/.cd, y dir=both, y explicit] table[skip first n=1,x index=1, y index=3, col sep=comma, y error index=13] {"plot2/adults_reduction_csv/attacker-eominmaxlossnf-0.01.csv"};
       	  \addplot[brown, mark = *,error bars/.cd, y dir=both, y explicit] table[skip first n=1,x index=1, y index=3,  col sep=comma, y error index=13] {"plot2/adults_PP_csv/attacker-eominmaxlossnf-0.01.csv"};
         \end{axis}

         \begin{axis}
         [name=plot2, at=(plot1.south east), anchor=south west, xshift=0.8cm, legend style={font=\small},title= {Majority group - Adv. Sampling},xlabel={$\epsilon$},ylabel={}, ymin = 0.7, ymax =1, xtick={0,0.05,0.1,0.15,0.2}, xticklabels={0,0.05,0.1,0.15,0.2},grid = major]

           \addplot[teal,mark = asterisk,error bars/.cd, y dir=both, y explicit] table[skip first n=1,x index=1, y index=4, col sep=comma, y error index=14] {"plot2/adults_reduction_csv/attacker-eominmaxlossnf-0.01.csv"};
           \addplot[red,mark = square*,error bars/.cd, y dir=both, y explicit] table[skip first n=1,x index=1, y index=6,  col sep=comma, y error index=16] {"plot2/adults_reduction_csv/attacker-eominmaxlossnf-0.1.csv"};
           \addplot[blue, mark = diamond*,error bars/.cd, y dir=both, y explicit] table[skip first n=1,x index=1, y index=6, col sep=comma, y error index=16] {"plot2/adults_reduction_csv/attacker-eominmaxlossnf-0.01.csv"};
      	  \addplot[brown, mark = *,error bars/.cd, y dir=both, y explicit] table[skip first n=1,x index=1, y index=6,  col sep=comma, y error index=16] {"plot2/adults_PP_csv/attacker-eominmaxlossnf-0.01.csv"};

         \end{axis}
           \begin{axis}
           [name=plot3,
    at=(plot2.south east), anchor=south west, xshift=0.8cm,legend style={font=\small},title= {Minority group - Adv. Sampling},xlabel={$\epsilon$},ylabel={}, ymin = 0.7, ymax =1, xtick={0,0.05,0.1,0.15,0.2}, xticklabels={0,0.05,0.1,0.15,0.2},grid = major]

             \addplot[teal,mark = asterisk,error bars/.cd, y dir=both, y explicit] table[skip first n=1,x index=1, y index=5, col sep=comma, y error index=15] {"plot2/adults_reduction_csv/attacker-eominmaxlossnf-0.01.csv"};
             \addplot[red,mark = square*,error bars/.cd, y dir=both, y explicit] table[skip first n=1,x index=1, y index=7,  col sep=comma, y error index=17] {"plot2/adults_reduction_csv/attacker-eominmaxlossnf-0.1.csv"};
             \addplot[blue, mark = diamond*,error bars/.cd, y dir=both, y explicit] table[skip first n=1,x index=1, y index=7, col sep=comma, y error index=17] {"plot2/adults_reduction_csv/attacker-eominmaxlossnf-0.01.csv"};
        	  \addplot[brown, mark = *,error bars/.cd, y dir=both, y explicit] table[skip first n=1,x index=1, y index=7,  col sep=comma, y error index=17] {"plot2/adults_PP_csv/attacker-eominmaxlossnf-0.01.csv"};

           \end{axis}
         \begin{axis}
         [name=plot4,at=(plot1.south west), anchor=north west, yshift=-1.7cm,title= {Overall - Adv. Labeling},xlabel={$\epsilon$},ylabel={}, ymin = 0.7, ymax =1, ylabel={Test Accuracy}, grid = major , xtick={0,0.05,0.1,0.15,0.2}, xticklabels={0,0.05,0.1,0.15,0.2},
         	   ]
           \addplot[teal,mark = asterisk,error bars/.cd, y dir=both, y explicit]  table[skip first n=1,x index=1, y index=2, col sep=comma, y error index=12] {"plot2/adults_reduction_csv/attacker-eominmaxloss-0.01.csv"};
           \addplot[red,mark = square*,error bars/.cd, y dir=both, y explicit]  table[skip first n=1,x index=1, y index=3,  col sep=comma, y error index=13] {"plot2/adults_reduction_csv/attacker-eominmaxloss-0.1.csv"};
           \addplot[blue, mark = diamond*,error bars/.cd, y dir=both, y explicit] table[skip first n=1,x index=1, y index=3, col sep=comma, y error index=13] {"plot2/adults_reduction_csv/attacker-eominmaxloss-0.01.csv"};
       	  \addplot[brown, mark = *,error bars/.cd, y dir=both, y explicit] table[skip first n=1,x index=1, y index=3,  col sep=comma, y error index=13] {"plot2/adults_PP_csv/attacker-eominmaxloss-0.01.csv"};
         \end{axis}
         \begin{axis}
         [name=plot5, at=(plot4.south east), anchor=south west, xshift=0.8cm, legend style={font=\small},title= {Majority group - Adv. Labeling},xlabel={$\epsilon$},ylabel={}, ymin = 0.7, ymax =1, xtick={0,0.05,0.1,0.15,0.2}, xticklabels={0,0.05,0.1,0.15,0.2},grid = major]

           \addplot[teal,mark = asterisk,error bars/.cd, y dir=both, y explicit] table[skip first n=1,x index=1, y index=4, col sep=comma, y error index=14] {"plot2/adults_reduction_csv/attacker-eominmaxloss-0.01.csv"};
           \addplot[red,mark = square*,error bars/.cd, y dir=both, y explicit] table[skip first n=1,x index=1, y index=6,  col sep=comma, y error index=16] {"plot2/adults_reduction_csv/attacker-eominmaxloss-0.1.csv"};
           \addplot[blue, mark = diamond*,error bars/.cd, y dir=both, y explicit] table[skip first n=1,x index=1, y index=6, col sep=comma, y error index=16] {"plot2/adults_reduction_csv/attacker-eominmaxloss-0.01.csv"};
      	  \addplot[brown, mark = *,error bars/.cd, y dir=both, y explicit] table[skip first n=1,x index=1, y index=6,  col sep=comma, y error index=16] {"plot2/adults_PP_csv/attacker-eominmaxloss-0.01.csv"};

         \end{axis}
           \begin{axis}
           [name=plot6,
    at=(plot5.south east), anchor=south west, xshift=0.8cm,legend style={font=\small},title= {Minority group - Adv. Labeling},xlabel={$\epsilon$},ylabel={}, ymin = 0.7, ymax =1, xtick={0,0.05,0.1,0.15,0.2}, xticklabels={0,0.05,0.1,0.15,0.2},grid = major]

             \addplot[teal,mark = asterisk,error bars/.cd, y dir=both, y explicit] table[skip first n=1,x index=1, y index=5, col sep=comma, y error index=15] {"plot2/adults_reduction_csv/attacker-eominmaxloss-0.01.csv"};
             \addplot[red,mark = square*,error bars/.cd, y dir=both, y explicit] table[skip first n=1,x index=1, y index=7,  col sep=comma, y error index=17] {"plot2/adults_reduction_csv/attacker-eominmaxloss-0.1.csv"};
             \addplot[blue, mark = diamond*,error bars/.cd, y dir=both, y explicit] table[skip first n=1,x index=1, y index=7, col sep=comma, y error index=17] {"plot2/adults_reduction_csv/attacker-eominmaxloss-0.01.csv"};
        	  \addplot[brown, mark = *,error bars/.cd, y dir=both, y explicit] table[skip first n=1,x index=1, y index=7,  col sep=comma, y error index=17] {"plot2/adults_PP_csv/attacker-eominmaxloss-0.01.csv"};

           \end{axis}
         \end{tikzpicture} }
         \caption{Alg.~\ref{algorithm:minmax-fair} ($\lambda=100\epsilon$)}
         \end{subfigure}

           \vspace{0.5em}

         \begin{subfigure}[b]{\columnwidth}
     \centering
      \resizebox{\columnwidth}{!}{%
 \begin{tikzpicture}
         \begin{axis}
         [name=plot1,title= { Overall - Adv. Sampling},xlabel={$\epsilon$},ylabel={}, ymin = 0.7, ymax =1, ylabel={Test Accuracy}, grid = major , xtick={0,0.05,0.1,0.15,0.2}, xticklabels={0,0.05,0.1,0.15,0.2}]
           \addplot[teal,mark = asterisk,error bars/.cd, y dir=both, y explicit]  table[skip first n=1,x index=1, y index=2, col sep=comma, y error index=12] {"plot2/adults_reduction_csv/attacker-newregnf-0.1-0.01.csv"};
           \addplot[red,mark = square*,error bars/.cd, y dir=both, y explicit]  table[skip first n=1,x index=1, y index=3,  col sep=comma, y error index=13] {"plot2/adults_reduction_csv/attacker-newregnf-0.1-0.1.csv"};
           \addplot[blue, mark = diamond*,error bars/.cd, y dir=both, y explicit] table[skip first n=1,x index=1, y index=3, col sep=comma, y error index=13] {"plot2/adults_reduction_csv/attacker-newregnf-0.1-0.01.csv"};
       	  \addplot[brown, mark = *,error bars/.cd, y dir=both, y explicit] table[skip first n=1,x index=1, y index=3,  col sep=comma, y error index=13] {"plot2/adults_PP_csv/attacker-newregnf-0.1-0.01.csv"};
         \end{axis}

         \begin{axis}
         [name=plot2, at=(plot1.south east), anchor=south west, xshift=0.8cm, legend style={font=\small},title= { Majority group - Adv. Sampling},xlabel={$\epsilon$},ylabel={}, ymin = 0.7, ymax =1, xtick={0,0.05,0.1,0.15,0.2}, xticklabels={0,0.05,0.1,0.15,0.2},grid = major]

           \addplot[teal,mark = asterisk,error bars/.cd, y dir=both, y explicit] table[skip first n=1,x index=1, y index=4, col sep=comma, y error index=14] {"plot2/adults_reduction_csv/attacker-newregnf-0.1-0.01.csv"};
           \addplot[red,mark = square*,error bars/.cd, y dir=both, y explicit] table[skip first n=1,x index=1, y index=6,  col sep=comma, y error index=16] {"plot2/adults_reduction_csv/attacker-newregnf-0.1-0.1.csv"};
           \addplot[blue, mark = diamond*,error bars/.cd, y dir=both, y explicit] table[skip first n=1,x index=1, y index=6, col sep=comma, y error index=16] {"plot2/adults_reduction_csv/attacker-newregnf-0.1-0.01.csv"};
      	  \addplot[brown, mark = *,error bars/.cd, y dir=both, y explicit] table[skip first n=1,x index=1, y index=6,  col sep=comma, y error index=16] {"plot2/adults_PP_csv/attacker-newregnf-0.1-0.01.csv"};

         \end{axis}
           \begin{axis}
           [name=plot3,
    at=(plot2.south east), anchor=south west, xshift=0.8cm,legend style={font=\small},title= {Minority group - Adv. Sampling},xlabel={$\epsilon$},ylabel={}, ymin = 0.7, ymax =1, xtick={0,0.05,0.1,0.15,0.2}, xticklabels={0,0.05,0.1,0.15,0.2},grid = major]

             \addplot[teal,mark = asterisk,error bars/.cd, y dir=both, y explicit] table[skip first n=1,x index=1, y index=5, col sep=comma, y error index=15] {"plot2/adults_reduction_csv/attacker-newregnf-0.1-0.01.csv"};
             \addplot[red,mark = square*,error bars/.cd, y dir=both, y explicit] table[skip first n=1,x index=1, y index=7,  col sep=comma, y error index=17] {"plot2/adults_reduction_csv/attacker-newregnf-0.1-0.1.csv"};
             \addplot[blue, mark = diamond*,error bars/.cd, y dir=both, y explicit] table[skip first n=1,x index=1, y index=7, col sep=comma, y error index=17] {"plot2/adults_reduction_csv/attacker-newregnf-0.1-0.01.csv"};
        	  \addplot[brown, mark = *,error bars/.cd, y dir=both, y explicit] table[skip first n=1,x index=1, y index=7,  col sep=comma, y error index=17] {"plot2/adults_PP_csv/attacker-newregnf-0.1-0.01.csv"};

           \end{axis}
         \begin{axis}
         [name=plot4,at=(plot1.south west), anchor=north west, yshift=-1.7cm,title= { Overall - Adv. Labeling},xlabel={$\epsilon$},ylabel={}, ymin = 0.7, ymax =1, ylabel={Test Accuracy}, grid = major , xtick={0,0.05,0.1,0.15,0.2}, xticklabels={0,0.05,0.1,0.15,0.2},
         	   ]
           \addplot[teal,mark = asterisk,error bars/.cd, y dir=both, y explicit]  table[skip first n=1,x index=1, y index=2, col sep=comma, y error index=12] {"plot2/adults_reduction_csv/attacker-newreg-0.1-0.01.csv"};
           \addplot[red,mark = square*,error bars/.cd, y dir=both, y explicit]  table[skip first n=1,x index=1, y index=3,  col sep=comma, y error index=13] {"plot2/adults_reduction_csv/attacker-newreg-0.1-0.1.csv"};
           \addplot[blue, mark = diamond*,error bars/.cd, y dir=both, y explicit] table[skip first n=1,x index=1, y index=3, col sep=comma, y error index=13] {"plot2/adults_reduction_csv/attacker-newreg-0.1-0.01.csv"};
       	  \addplot[brown, mark = *,error bars/.cd, y dir=both, y explicit] table[skip first n=1,x index=1, y index=3,  col sep=comma, y error index=13] {"plot2/adults_PP_csv/attacker-newreg-0.1-0.01.csv"};
         \end{axis}
         \begin{axis}
         [name=plot5, at=(plot4.south east), anchor=south west, xshift=0.8cm, legend style={font=\small},title= {Majority group - Adv. Labeling},xlabel={$\epsilon$},ylabel={}, ymin = 0.7, ymax =1, xtick={0,0.05,0.1,0.15,0.2}, xticklabels={0,0.05,0.1,0.15,0.2},grid = major]

           \addplot[teal,mark = asterisk,error bars/.cd, y dir=both, y explicit] table[skip first n=1,x index=1, y index=4, col sep=comma, y error index=14] {"plot2/adults_reduction_csv/attacker-newreg-0.1-0.01.csv"};
           \addplot[red,mark = square*,error bars/.cd, y dir=both, y explicit] table[skip first n=1,x index=1, y index=6,  col sep=comma, y error index=16] {"plot2/adults_reduction_csv/attacker-newreg-0.1-0.1.csv"};
           \addplot[blue, mark = diamond*,error bars/.cd, y dir=both, y explicit] table[skip first n=1,x index=1, y index=6, col sep=comma, y error index=16] {"plot2/adults_reduction_csv/attacker-newreg-0.1-0.01.csv"};
      	  \addplot[brown, mark = *,error bars/.cd, y dir=both, y explicit] table[skip first n=1,x index=1, y index=6,  col sep=comma, y error index=16] {"plot2/adults_PP_csv/attacker-newreg-0.1-0.01.csv"};

         \end{axis}
           \begin{axis}
           [name=plot6,
    at=(plot5.south east), anchor=south west, xshift=0.8cm,legend style={font=\small},title= { Minority group - Adv. Labeling},xlabel={$\epsilon$},ylabel={}, ymin = 0.7, ymax =1, xtick={0,0.05,0.1,0.15,0.2}, xticklabels={0,0.05,0.1,0.15,0.2},grid = major]

             \addplot[teal,mark = asterisk,error bars/.cd, y dir=both, y explicit] table[skip first n=1,x index=1, y index=5, col sep=comma, y error index=15] {"plot2/adults_reduction_csv/attacker-newreg-0.1-0.01.csv"};
             \addplot[red,mark = square*,error bars/.cd, y dir=both, y explicit] table[skip first n=1,x index=1, y index=7,  col sep=comma, y error index=17] {"plot2/adults_reduction_csv/attacker-newreg-0.1-0.1.csv"};
             \addplot[blue, mark = diamond*,error bars/.cd, y dir=both, y explicit] table[skip first n=1,x index=1, y index=7, col sep=comma, y error index=17] {"plot2/adults_reduction_csv/attacker-newreg-0.1-0.01.csv"};
        	  \addplot[brown, mark = *,error bars/.cd, y dir=both, y explicit] table[skip first n=1,x index=1, y index=7,  col sep=comma, y error index=17] {"plot2/adults_PP_csv/attacker-newreg-0.1-0.01.csv"};

           \end{axis}
         \end{tikzpicture}}
         \caption{Alg.~\ref{algorithm:OGD} ($\lambda=0.1\epsilon$)}
 \end{subfigure}
 } 
\vspace{-1.5em}
 \caption{Effect of fairness level $\delta$ on robustness across groups under adversarial sampling and adversarial labeling  attacks -- Adult dataset. The majority group (males) contributes $66\%$ of the training data.
 }
 \label{fig:fairness-measure-adult-labeling}
\end{figure*}

%% file: sections/appdx-adult-exp.tex

\subsection{Performance of Algorithm 2 with $\lambda=100\epsilon$ on Adult dataset}\label{appdx:adult-exp}

We notice that there are accuracy fluctuations for the fair models evaluated on the poisoning data selected by Algorithm~\ref{algorithm:minmax-fair} with $\lambda = 100\epsilon$ for the Adult dataset.
Recall that, the algorithm selects poisoning data from the attack dataset \emph{without replacement}. In each iteration, it selects the data point that maximizes the classification loss plus the fairness gap (as at Line~\ref{line:SA:pick} in Algorithm~\ref{algorithm:minmax-fair}).
As shown in Figure 11 and 12, Algorithm~\ref{algorithm:minmax-fair} with $\lambda=100\epsilon$ has a significant preference to select poisoning data that would result in a large fairness gap. Thus, it chooses data that would fall into the smallest subgroup in the training set. This is shown to be very effective in the case of COMPAS dataset and can lead to a sharp decrease in the model accuracy even for small $\epsilon$ (see Figure~\ref{fig:accuracy-compas-paper}). However, this greedy algorithm in the case of small attack sets, and no repetition in the poisoning data, can result in the degradation of attack performance for larger $\epsilon$ values, as we see in Figure~\ref{fig:labeling-attacks-adult}.

In more detail, the reason behind the attack behavior on Algorithm 2 with $\lambda=100 \epsilon$ for larger $\epsilon$ on the Adult dataset is the following.

In the adversarial sampling setting, the size of the smallest subgroup ($y=+$ and $s=1$) in the attack dataset is only equivalent to $\epsilon = 7.69 \%$ poisoning data. For larger values of $\epsilon$, the attacker will choose data from other subgroups, which cannot further harm the model accuracy, thus reduces the effect of the data poisoning.

In the adversarial labeling setting, with large $\epsilon$, the number of poisoning data points is larger than the size of subgroups with positive labels ($y=+$) in $\D_\clean$; typically when $\epsilon=0.2$, $|\D_\poisoning|>3000$ whereas the number of samples with $y=+$ in $\D_\clean$ is 2943 on average. Relying on choosing points to select data points to maximize $\Delta$ results in the possibility of choosing points from any subgroups with positive labels (as shown in Figure~\ref{fig:poisoning-acc-adult-labeling}(b)), since poisoning data points can dominate any of these subgroups. In other words, the smallest subgroup of the clean training dataset is not the smallest subgroup of the training dataset.

In summary, the fluctuation in the figures is due to the significant effect of maximizing the fairness gap.
In fact, in both adversarial sampling and labeling settings, Algorithm~\ref{algorithm:minmax-fair} ($\lambda=100\epsilon$) achieves the same performances with a smaller $\epsilon$ as the other attacks with larger $\epsilon$. These results, in effect, reflects the effectiveness of the algorithm.

%% file: sections/appdx-exp-train-acc.tex

\subsection{Training accuracy of poisoned dataset} \label{appdx:exp-train-acc}

In Figure~\ref{fig:poisoning-acc} and Figure~\ref{fig:poisoning-acc-adult-labeling}, the accuracy of the unconstrained model on the poisoning data $\D_\poisoning$ is compared with the corresponding accuracy of fair models with different fairness level $\delta$ on COMPAS and Adult dataset respectively.
The poisoning data $\D_\poisoning$ is selected using Algorithm ~\ref{algorithm:minmax-fair} with $\lambda = 0$ for the unconstrained model for both adversarial labeling and adversarial sampling settings.
For fair models, we evaluate on poisoning dataset selected using Algorithm ~\ref{algorithm:minmax-fair} with $\lambda = 100\epsilon$ and $\lambda = 0.1 \epsilon$ on Adult dataset.
For the fair models trained on COMPAS dataset, we evaluate on poisoning dataset selected using Algorithm ~\ref{algorithm:minmax-fair} with $\lambda = 100\epsilon$ and $\lambda = \epsilon$.

On the COMPAS dataset, from Figure~\ref{fig:poisoning-acc}, we can observe that, for the fair model using \cite{agarwal2018reductions}, as the value of $\delta$ decreases, the accuracy of the model increases on poisoning data $\D_\poisoning$ and decreases on $\D_\clean$. This implies poisoning data reduce fair models' ability to learn from clean data.
In Figure~\ref{fig:poisoning-acc}, note that the post-processing method does not impose the fairness constraint during the training but uses the predictions of the unconstrained model trained in the standard way and makes corrections to achieve fairness. Depending on which subgroups poisoning data points belong to, fair models trained with \cite{hardt2016equality} show different behavior. For example, when poisoning data points have the sensitive attribute $s=1$ and label $y=+$, post-processing tends to make corrections for the majority group (observed in Figure~\ref{fig:fairness-measure}(a)). The accuracy of poisoning data remains similar to that of the unconstrained model, but the accuracy of clean data decreases significantly. 

On Adult dataset, from Figure~\ref{fig:poisoning-acc-adult-labeling}, we can also observe that the fair models have a higher accuracy on the poisoning data compared with the unconstrained model. In Figure~\ref{fig:poisoning-acc-adult-labeling}(a), we notice that the accuracy of the poisoning data increases when $\epsilon$ increases from 0 to 0.05. After that, the accuracy decreases for the fair models. We provide detailed explanations in Appendix~\ref{appdx:adult-exp}.

\input{plot/train_acc.tex}
\input{plot/train_acc_ADULT_labeling.tex}

%% file: plot/train_acc.tex
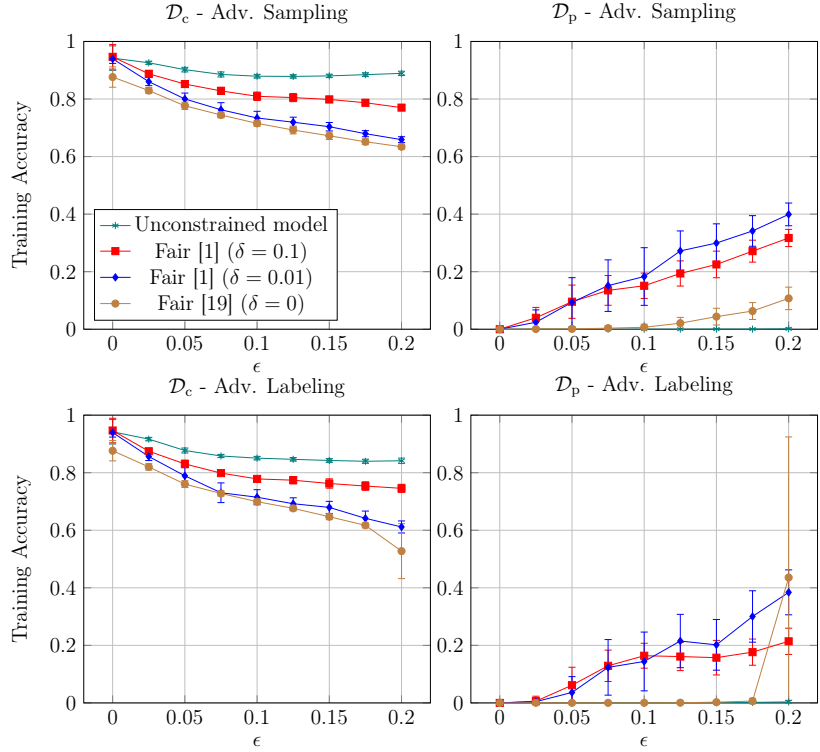
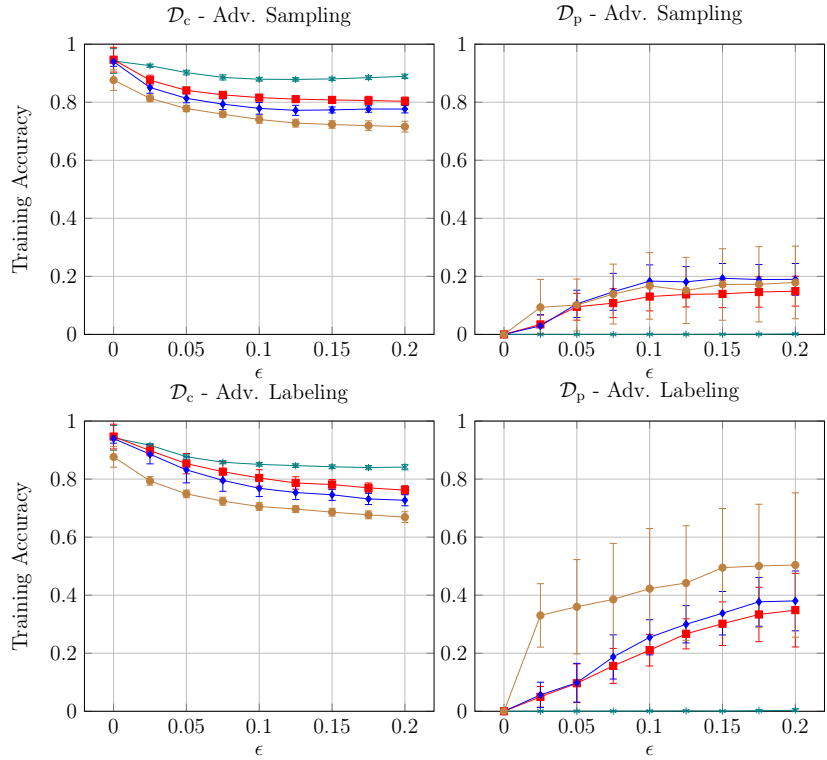
\begin{figure*}[t!]
 \centering
 {\large
     \begin{subfigure}[b]{\columnwidth}
     \centering
      \resizebox{0.66\columnwidth}{!}{%
     \begin{tikzpicture}
       \begin{axis}
       [name=plot11,title= {$\D_\clean$ - Adv. Sampling},xlabel={$\epsilon$ },ylabel={Training Accuracy}, ymin = 0, ymax =1, xtick={0,0.05,0.1,0.15,0.2}, xticklabels={0,0.05,0.1,0.15,0.2}, grid = major,legend entries = {Unconstrained model, Fair \cite{agarwal2018reductions} ($\delta =0.1$),Fair \cite{agarwal2018reductions} ($\delta =0.01$), Fair \cite{hardt2016equality} ($\delta = 0$)}, legend pos = south west]
        \addplot[teal,mark = asterisk,error bars/.cd, y dir=both, y explicit] table[skip first n=1,x index=1, y index=8,  col sep=comma, y error index=18] {"plot2/reduction_csv/attacker-nfminmax-0.01.csv"};
         \addplot[red,mark = square*,error bars/.cd, y dir=both, y explicit] table[skip first n=1,x index=1, y index=10,  col sep=comma, y error index=20] {"plot2/reduction_csv/attacker-eominmaxlossnf100-0.1.csv"};
         \addplot[blue, mark = diamond*,error bars/.cd, y dir=both, y explicit] table[skip first n=1,x index=1, y index=10, col sep=comma, y error index=20] {"plot2/reduction_csv/attacker-eominmaxlossnf100-0.01.csv"};
         \addplot[brown, mark = *,error bars/.cd, y dir=both, y explicit] table[skip first n=1,x index=1, y index=10,  col sep=comma, y error index=20] {"plot2/PP_csv/attacker-eominmaxlossnf100-0.01.csv"};

       \end{axis}
       \begin{axis}
       [name=plot12, at=(plot11.south east), anchor=south west, xshift=0.8cm,title= {$\D_\poisoning$ - Adv. Sampling},xlabel={$\epsilon$},ylabel={}, ymin = 0, ymax =1, xtick={0,0.05,0.1,0.15,0.2}, xticklabels={0,0.05,0.1,0.15,0.2}, grid = major]
        \addplot[teal,mark = asterisk,error bars/.cd, y dir=both, y explicit] table[skip first n=1,x index=1, y index=9,  col sep=comma, y error index=19] {"plot2/reduction_csv/attacker-nfminmax-0.01.csv"};
         \addplot[red,mark = square*,error bars/.cd, y dir=both, y explicit] table[skip first n=1,x index=1, y index=11,  col sep=comma, y error index=21] {"plot2/reduction_csv/attacker-eominmaxlossnf100-0.1.csv"};
         \addplot[blue, mark = diamond*,error bars/.cd, y dir=both, y explicit] table[skip first n=1,x index=1, y index=11, col sep=comma, y error index=21] {"plot2/reduction_csv/attacker-eominmaxlossnf100-0.01.csv"};
          \addplot[brown, mark = *,error bars/.cd, y dir=both, y explicit] table[skip first n=1,x index=1, y index=11,  col sep=comma, y error index=21] {"plot2/PP_csv/attacker-eominmaxlossnf100-0.01.csv"};
       \end{axis}
       \begin{axis}
       [name=plot3,at=(plot11.south west), anchor=north west, yshift=-1.7cm,title= {$\D_\clean$ - Adv. Labeling},xlabel={$\epsilon$ },ylabel={Training Accuracy}, ymin = 0, ymax =1, xtick={0,0.05,0.1,0.15,0.2}, xticklabels={0,0.05,0.1,0.15,0.2}, grid = major, 
       ]
        \addplot[teal,mark = asterisk,error bars/.cd, y dir=both, y explicit] table[skip first n=1,x index=1, y index=8,  col sep=comma, y error index=18] {"plot2/reduction_csv/attacker-minmax2-0.01.csv"};
         \addplot[red,mark = square*,error bars/.cd, y dir=both, y explicit] table[skip first n=1,x index=1, y index=10,  col sep=comma, y error index=20] {"plot2/reduction_csv/attacker-eominmaxloss-0.1.csv"};
         \addplot[blue, mark = diamond*,error bars/.cd, y dir=both, y explicit] table[skip first n=1,x index=1, y index=10, col sep=comma, y error index=20] {"plot2/reduction_csv/attacker-eominmaxloss-0.01.csv"};
         \addplot[brown, mark = *,error bars/.cd, y dir=both, y explicit] table[skip first n=1,x index=1, y index=10,  col sep=comma, y error index=20] {"plot2/PP_csv/attacker-eominmaxloss-0.01.csv"};

       \end{axis}
       \begin{axis}
       [name=plot4, at=(plot3.south east), anchor=south west, xshift=0.8cm,title= {$\D_\poisoning$ - Adv. Labeling},xlabel={$\epsilon$},ylabel={}, ymin = 0, ymax =1, xtick={0,0.05,0.1,0.15,0.2}, xticklabels={0,0.05,0.1,0.15,0.2}, grid = major]
        \addplot[teal,mark = asterisk,error bars/.cd, y dir=both, y explicit] table[skip first n=1,x index=1, y index=9,  col sep=comma, y error index=19] {"plot2/reduction_csv/attacker-minmax2-0.01.csv"};
         \addplot[red,mark = square*,error bars/.cd, y dir=both, y explicit] table[skip first n=1,x index=1, y index=11,  col sep=comma, y error index=21] {"plot2/reduction_csv/attacker-eominmaxloss-0.1.csv"};
         \addplot[blue, mark = diamond*,error bars/.cd, y dir=both, y explicit] table[skip first n=1,x index=1, y index=11, col sep=comma, y error index=21] {"plot2/reduction_csv/attacker-eominmaxloss-0.01.csv"};
          \addplot[brown, mark = *,error bars/.cd, y dir=both, y explicit] table[skip first n=1,x index=1, y index=11,  col sep=comma, y error index=21] {"plot2/PP_csv/attacker-eominmaxloss-0.01.csv"};
       \end{axis}
     \end{tikzpicture} }
     \caption{Alg.~\ref{algorithm:minmax-fair} ($\lambda=100\epsilon$)}
     \end{subfigure}

      \vspace{1em}

     \begin{subfigure}[b]{\columnwidth}
     \centering
      \resizebox{0.66\columnwidth}{!}{%
     \begin{tikzpicture}
       \begin{axis}
       [name=plot11,title= {$\D_\clean$ - Adv. Sampling},xlabel={$\epsilon$ },ylabel={Training Accuracy}, ymin = 0, ymax =1, xtick={0,0.05,0.1,0.15,0.2}, xticklabels={0,0.05,0.1,0.15,0.2}, grid = major]
        \addplot[teal,mark = asterisk,error bars/.cd, y dir=both, y explicit] table[skip first n=1,x index=1, y index=8,  col sep=comma, y error index=18] {"plot2/reduction_csv/attacker-nfminmax-0.01.csv"};
         \addplot[red,mark = square*,error bars/.cd, y dir=both, y explicit] table[skip first n=1,x index=1, y index=10,  col sep=comma, y error index=20] {"plot2/reduction_csv/attacker-newregnf-1.0-0.1.csv"};
         \addplot[blue, mark = diamond*,error bars/.cd, y dir=both, y explicit] table[skip first n=1,x index=1, y index=10, col sep=comma, y error index=20] {"plot2/reduction_csv/attacker-newregnf-1.0-0.01.csv"};
         \addplot[brown, mark = *,error bars/.cd, y dir=both, y explicit] table[skip first n=1,x index=1, y index=10,  col sep=comma, y error index=20] {"plot2/PP_csv/attacker-newregnf-1.0-0.01.csv"};

       \end{axis}
       \begin{axis}
       [name=plot12, at=(plot11.south east), anchor=south west, xshift=0.8cm,title= {$\D_\poisoning$ - Adv. Sampling},xlabel={$\epsilon$},ylabel={}, ymin = 0, ymax =1, xtick={0,0.05,0.1,0.15,0.2}, xticklabels={0,0.05,0.1,0.15,0.2}, grid = major]
        \addplot[teal,mark = asterisk,error bars/.cd, y dir=both, y explicit] table[skip first n=1,x index=1, y index=9,  col sep=comma, y error index=19] {"plot2/reduction_csv/attacker-nfminmax-0.01.csv"};
         \addplot[red,mark = square*,error bars/.cd, y dir=both, y explicit] table[skip first n=1,x index=1, y index=11,  col sep=comma, y error index=21] {"plot2/reduction_csv/attacker-newregnf-1.0-0.1.csv"};
         \addplot[blue, mark = diamond*,error bars/.cd, y dir=both, y explicit] table[skip first n=1,x index=1, y index=11, col sep=comma, y error index=21] {"plot2/reduction_csv/attacker-newregnf-1.0-0.01.csv"};
          \addplot[brown, mark = *,error bars/.cd, y dir=both, y explicit] table[skip first n=1,x index=1, y index=11,  col sep=comma, y error index=21] {"plot2/PP_csv/attacker-newregnf-1.0-0.01.csv"};
       \end{axis}
       \begin{axis}
       [name=plot3,at=(plot11.south west), anchor=north west, yshift=-1.7cm,title= {$\D_\clean$ - Adv. Labeling},xlabel={$\epsilon$ },ylabel={Training Accuracy}, ymin = 0, ymax =1, xtick={0,0.05,0.1,0.15,0.2}, xticklabels={0,0.05,0.1,0.15,0.2}, grid = major, 
       ]
        \addplot[teal,mark = asterisk,error bars/.cd, y dir=both, y explicit] table[skip first n=1,x index=1, y index=8,  col sep=comma, y error index=18] {"plot2/reduction_csv/attacker-minmax2-0.01.csv"};
         \addplot[red,mark = square*,error bars/.cd, y dir=both, y explicit] table[skip first n=1,x index=1, y index=10,  col sep=comma, y error index=20] {"plot2/reduction_csv/attacker-newreg-1.0-0.1.csv"};
         \addplot[blue, mark = diamond*,error bars/.cd, y dir=both, y explicit] table[skip first n=1,x index=1, y index=10, col sep=comma, y error index=20] {"plot2/reduction_csv/attacker-newreg-1.0-0.01.csv"};
         \addplot[brown, mark = *,error bars/.cd, y dir=both, y explicit] table[skip first n=1,x index=1, y index=10,  col sep=comma, y error index=20] {"plot2/PP_csv/attacker-newreg-1.0-0.01.csv"};

       \end{axis}
       \begin{axis}
       [name=plot4, at=(plot3.south east), anchor=south west, xshift=0.8cm,title= {$\D_\poisoning$ - Adv. Labeling},xlabel={$\epsilon$},ylabel={}, ymin = 0, ymax =1, xtick={0,0.05,0.1,0.15,0.2}, xticklabels={0,0.05,0.1,0.15,0.2}, grid = major]
        \addplot[teal,mark = asterisk,error bars/.cd, y dir=both, y explicit] table[skip first n=1,x index=1, y index=9,  col sep=comma, y error index=19] {"plot2/reduction_csv/attacker-minmax2-0.01.csv"};
         \addplot[red,mark = square*,error bars/.cd, y dir=both, y explicit] table[skip first n=1,x index=1, y index=11,  col sep=comma, y error index=21] {"plot2/reduction_csv/attacker-newreg-1.0-0.1.csv"};
         \addplot[blue, mark = diamond*,error bars/.cd, y dir=both, y explicit] table[skip first n=1,x index=1, y index=11, col sep=comma, y error index=21] {"plot2/reduction_csv/attacker-newreg-1.0-0.01.csv"};
          \addplot[brown, mark = *,error bars/.cd, y dir=both, y explicit] table[skip first n=1,x index=1, y index=11,  col sep=comma, y error index=21] {"plot2/PP_csv/attacker-newreg-1.0-0.01.csv"};
       \end{axis}
     \end{tikzpicture}}
     \caption{Alg.~\ref{algorithm:OGD} ($\lambda=\epsilon$)}
     \end{subfigure}

 }
 \caption{Accuracy of clean training data and poisoning data under adversarial sampling and adversarial labeling attacks -- COMPAS dataset.
 }
 \label{fig:poisoning-acc}
\end{figure*}

%% file: plot/train_acc_ADULT_labeling.tex
\begin{figure*}
 \centering
     {\large
    \begin{subfigure}[b]{\columnwidth}
    \centering
    \resizebox{0.66\columnwidth}{!}{%
     \begin{tikzpicture}
       \begin{axis}
       [name=plot1,title= {$\D_\clean$ - Adv. Sampling},xlabel={$\epsilon$ },ylabel={Training Accuracy}, ymin = 0, ymax =1, xtick={0,0.05,0.1,0.15,0.2}, xticklabels={0,0.05,0.1,0.15,0.2}, grid = major,legend entries = {Unconstrained model, Fair \cite{agarwal2018reductions} ($\delta =0.1$),Fair \cite{agarwal2018reductions} ($\delta =0.01$), Fair \cite{hardt2016equality} ($\delta = 0$)}, legend pos = south west]
        \addplot[teal,mark = asterisk,error bars/.cd, y dir=both, y explicit] table[skip first n=1,x index=1, y index=8,  col sep=comma, y error index=18] {"plot2/adults_reduction_csv/attacker-nfminmax-0.01.csv"};
         \addplot[red,mark = square*,error bars/.cd, y dir=both, y explicit] table[skip first n=1,x index=1, y index=10,  col sep=comma, y error index=20] {"plot2/adults_reduction_csv/attacker-eominmaxlossnf-0.1.csv"};
         \addplot[blue, mark = diamond*,error bars/.cd, y dir=both, y explicit] table[skip first n=1,x index=1, y index=10, col sep=comma, y error index=20] {"plot2/adults_reduction_csv/attacker-eominmaxlossnf-0.01.csv"};
         \addplot[brown, mark = *,error bars/.cd, y dir=both, y explicit] table[skip first n=1,x index=1, y index=10,  col sep=comma, y error index=20] {"plot2/adults_PP_csv/attacker-eominmaxlossnf-0.01.csv"};

       \end{axis}
       \begin{axis}
       [name=plot2, at=(plot1.south east), anchor=south west, xshift=0.8cm,title= {$\D_\poisoning$ - Adv. Sampling},xlabel={$\epsilon$},ylabel={}, ymin = 0, ymax =1, xtick={0,0.05,0.1,0.15,0.2}, xticklabels={0,0.05,0.1,0.15,0.2}, grid = major]
        \addplot[teal,mark = asterisk,error bars/.cd, y dir=both, y explicit] table[skip first n=1,x index=1, y index=9,  col sep=comma, y error index=19] {"plot2/adults_reduction_csv/attacker-nfminmax-0.01.csv"};
         \addplot[red,mark = square*,error bars/.cd, y dir=both, y explicit] table[skip first n=1,x index=1, y index=11,  col sep=comma, y error index=21] {"plot2/adults_reduction_csv/attacker-eominmaxlossnf-0.1.csv"};
         \addplot[blue, mark = diamond*,error bars/.cd, y dir=both, y explicit] table[skip first n=1,x index=1, y index=11, col sep=comma, y error index=21] {"plot2/adults_reduction_csv/attacker-eominmaxlossnf-0.01.csv"};
          \addplot[brown, mark = *,error bars/.cd, y dir=both, y explicit] table[skip first n=1,x index=1, y index=11,  col sep=comma, y error index=21] {"plot2/adults_PP_csv/attacker-eominmaxlossnf-0.01.csv"};
       \end{axis}
       \begin{axis}
       [name=plot3,at=(plot1.south west), anchor=north west, yshift=-1.7cm,title= {$\D_\clean$ - Adv. Labeling},xlabel={$\epsilon$ },ylabel={Training Accuracy}, ymin = 0, ymax =1, xtick={0,0.05,0.1,0.15,0.2}, xticklabels={0,0.05,0.1,0.15,0.2}, grid = major,
       ]
        \addplot[teal,mark = asterisk,error bars/.cd, y dir=both, y explicit] table[skip first n=1,x index=1, y index=8,  col sep=comma, y error index=18] {"plot2/adults_reduction_csv/attacker-minmax2-0.01.csv"};
         \addplot[red,mark = square*,error bars/.cd, y dir=both, y explicit] table[skip first n=1,x index=1, y index=10,  col sep=comma, y error index=20] {"plot2/adults_reduction_csv/attacker-eominmaxloss-0.1.csv"};
         \addplot[blue, mark = diamond*,error bars/.cd, y dir=both, y explicit] table[skip first n=1,x index=1, y index=10, col sep=comma, y error index=20] {"plot2/adults_reduction_csv/attacker-eominmaxloss-0.01.csv"};
         \addplot[brown, mark = *,error bars/.cd, y dir=both, y explicit] table[skip first n=1,x index=1, y index=10,  col sep=comma, y error index=20] {"plot2/adults_PP_csv/attacker-eominmaxloss-0.01.csv"};

       \end{axis}
       \begin{axis}
       [name=plot4, at=(plot3.south east), anchor=south west, xshift=0.8cm,title= {$\D_\poisoning$ - Adv. Labeling},xlabel={$\epsilon$},ylabel={}, ymin = 0, ymax =1, xtick={0,0.05,0.1,0.15,0.2}, xticklabels={0,0.05,0.1,0.15,0.2}, grid = major]
        \addplot[teal,mark = asterisk,error bars/.cd, y dir=both, y explicit] table[skip first n=1,x index=1, y index=9,  col sep=comma, y error index=19] {"plot2/adults_reduction_csv/attacker-minmax2-0.01.csv"};
         \addplot[red,mark = square*,error bars/.cd, y dir=both, y explicit] table[skip first n=1,x index=1, y index=11,  col sep=comma, y error index=21] {"plot2/adults_reduction_csv/attacker-eominmaxloss-0.1.csv"};
         \addplot[blue, mark = diamond*,error bars/.cd, y dir=both, y explicit] table[skip first n=1,x index=1, y index=11, col sep=comma, y error index=21] {"plot2/adults_reduction_csv/attacker-eominmaxloss-0.01.csv"};
          \addplot[brown, mark = *,error bars/.cd, y dir=both, y explicit] table[skip first n=1,x index=1, y index=11,  col sep=comma, y error index=21] {"plot2/adults_PP_csv/attacker-eominmaxloss-0.01.csv"};
       \end{axis}
     \end{tikzpicture}}
     \caption{Alg.~\ref{algorithm:minmax-fair} ($\lambda=100\epsilon$)}
      \end{subfigure}

      \vspace{1em}

     \begin{subfigure}[b]{\columnwidth}
     \centering
     \resizebox{0.66\columnwidth}{!}{%
     \begin{tikzpicture}
       \begin{axis}
       [name=plot1,title= {$\D_\clean$ - Adv. Sampling},xlabel={$\epsilon$ },ylabel={Training Accuracy}, ymin = 0, ymax =1, xtick={0,0.05,0.1,0.15,0.2}, xticklabels={0,0.05,0.1,0.15,0.2}, grid = major]
        \addplot[teal,mark = asterisk,error bars/.cd, y dir=both, y explicit] table[skip first n=1,x index=1, y index=8,  col sep=comma, y error index=18] {"plot2/adults_reduction_csv/attacker-nfminmax-0.01.csv"};
         \addplot[red,mark = square*,error bars/.cd, y dir=both, y explicit] table[skip first n=1,x index=1, y index=10,  col sep=comma, y error index=20] {"plot2/adults_reduction_csv/attacker-newregnf-0.1-0.1.csv"};
         \addplot[blue, mark = diamond*,error bars/.cd, y dir=both, y explicit] table[skip first n=1,x index=1, y index=10, col sep=comma, y error index=20] {"plot2/adults_reduction_csv/attacker-newregnf-0.1-0.01.csv"};
         \addplot[brown, mark = *,error bars/.cd, y dir=both, y explicit] table[skip first n=1,x index=1, y index=10,  col sep=comma, y error index=20] {"plot2/adults_PP_csv/attacker-newregnf-0.1-0.01.csv"};

       \end{axis}
       \begin{axis}
       [name=plot2, at=(plot1.south east), anchor=south west, xshift=0.8cm,title= {$\D_\poisoning$ - Adv. Sampling},xlabel={$\epsilon$},ylabel={}, ymin = 0, ymax =1, xtick={0,0.05,0.1,0.15,0.2}, xticklabels={0,0.05,0.1,0.15,0.2}, grid = major]
        \addplot[teal,mark = asterisk,error bars/.cd, y dir=both, y explicit] table[skip first n=1,x index=1, y index=9,  col sep=comma, y error index=19] {"plot2/adults_reduction_csv/attacker-nfminmax-0.01.csv"};
         \addplot[red,mark = square*,error bars/.cd, y dir=both, y explicit] table[skip first n=1,x index=1, y index=11,  col sep=comma, y error index=21] {"plot2/adults_reduction_csv/attacker-newregnf-0.1-0.1.csv"};
         \addplot[blue, mark = diamond*,error bars/.cd, y dir=both, y explicit] table[skip first n=1,x index=1, y index=11, col sep=comma, y error index=21] {"plot2/adults_reduction_csv/attacker-newregnf-0.1-0.01.csv"};
          \addplot[brown, mark = *,error bars/.cd, y dir=both, y explicit] table[skip first n=1,x index=1, y index=11,  col sep=comma, y error index=21] {"plot2/adults_PP_csv/attacker-newregnf-0.1-0.01.csv"};
       \end{axis}
       \begin{axis}
       [name=plot3,at=(plot1.south west), anchor=north west, yshift=-1.7cm,title= { $\D_\clean$ - Adv. Labeling},xlabel={$\epsilon$ },ylabel={Training Accuracy}, ymin = 0, ymax =1, xtick={0,0.05,0.1,0.15,0.2}, xticklabels={0,0.05,0.1,0.15,0.2}, grid = major,
       ]
        \addplot[teal,mark = asterisk,error bars/.cd, y dir=both, y explicit] table[skip first n=1,x index=1, y index=8,  col sep=comma, y error index=18] {"plot2/adults_reduction_csv/attacker-minmax2-0.01.csv"};
         \addplot[red,mark = square*,error bars/.cd, y dir=both, y explicit] table[skip first n=1,x index=1, y index=10,  col sep=comma, y error index=20] {"plot2/adults_reduction_csv/attacker-newreg-0.1-0.1.csv"};
         \addplot[blue, mark = diamond*,error bars/.cd, y dir=both, y explicit] table[skip first n=1,x index=1, y index=10, col sep=comma, y error index=20] {"plot2/adults_reduction_csv/attacker-newreg-0.1-0.01.csv"};
         \addplot[brown, mark = *,error bars/.cd, y dir=both, y explicit] table[skip first n=1,x index=1, y index=10,  col sep=comma, y error index=20] {"plot2/adults_PP_csv/attacker-newreg-0.1-0.01.csv"};

       \end{axis}
       \begin{axis}
       [name=plot4, at=(plot3.south east), anchor=south west, xshift=0.8cm,title= {$\D_\poisoning$ - Adv. Labeling},xlabel={$\epsilon$},ylabel={}, ymin = 0, ymax =1, xtick={0,0.05,0.1,0.15,0.2}, xticklabels={0,0.05,0.1,0.15,0.2}, grid = major]
        \addplot[teal,mark = asterisk,error bars/.cd, y dir=both, y explicit] table[skip first n=1,x index=1, y index=9,  col sep=comma, y error index=19] {"plot2/adults_reduction_csv/attacker-minmax2-0.01.csv"};
         \addplot[red,mark = square*,error bars/.cd, y dir=both, y explicit] table[skip first n=1,x index=1, y index=11,  col sep=comma, y error index=21] {"plot2/adults_reduction_csv/attacker-newreg-0.1-0.1.csv"};
         \addplot[blue, mark = diamond*,error bars/.cd, y dir=both, y explicit] table[skip first n=1,x index=1, y index=11, col sep=comma, y error index=21] {"plot2/adults_reduction_csv/attacker-newreg-0.1-0.01.csv"};
          \addplot[brown, mark = *,error bars/.cd, y dir=both, y explicit] table[skip first n=1,x index=1, y index=11,  col sep=comma, y error index=21] {"plot2/adults_PP_csv/attacker-newreg-0.1-0.01.csv"};
       \end{axis}
     \end{tikzpicture}}
     \caption{Alg.~\ref{algorithm:OGD} ($\lambda=0.1\epsilon$)}
     \end{subfigure}
 }
 \caption{Accuracy of clean training data and poisoning data under adversarial sampling  and adversarial labeling attacks -- Adult dataset. 
 }
 \label{fig:poisoning-acc-adult-labeling}
\end{figure*}
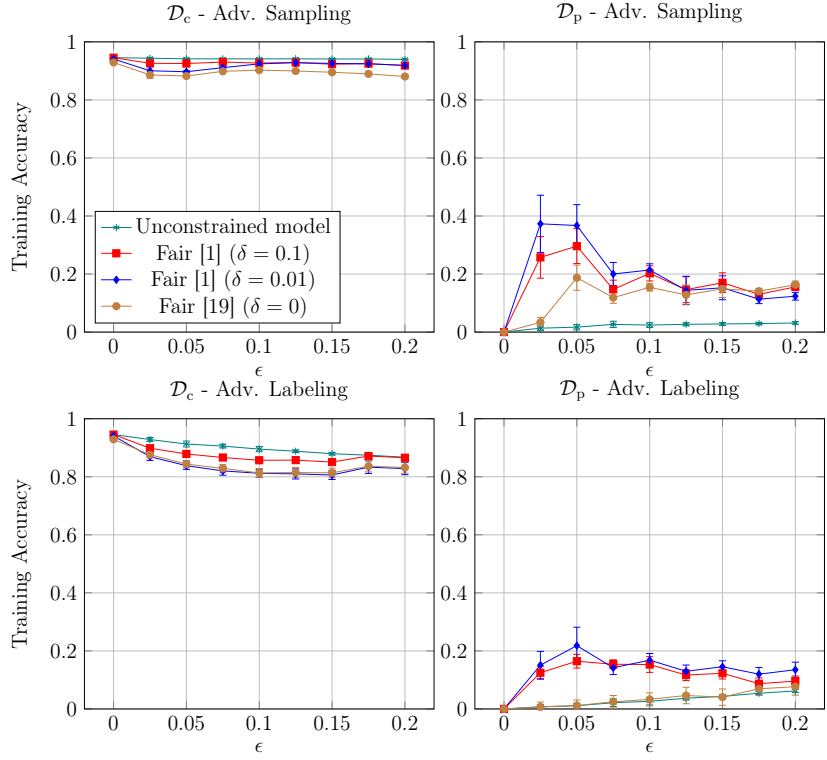
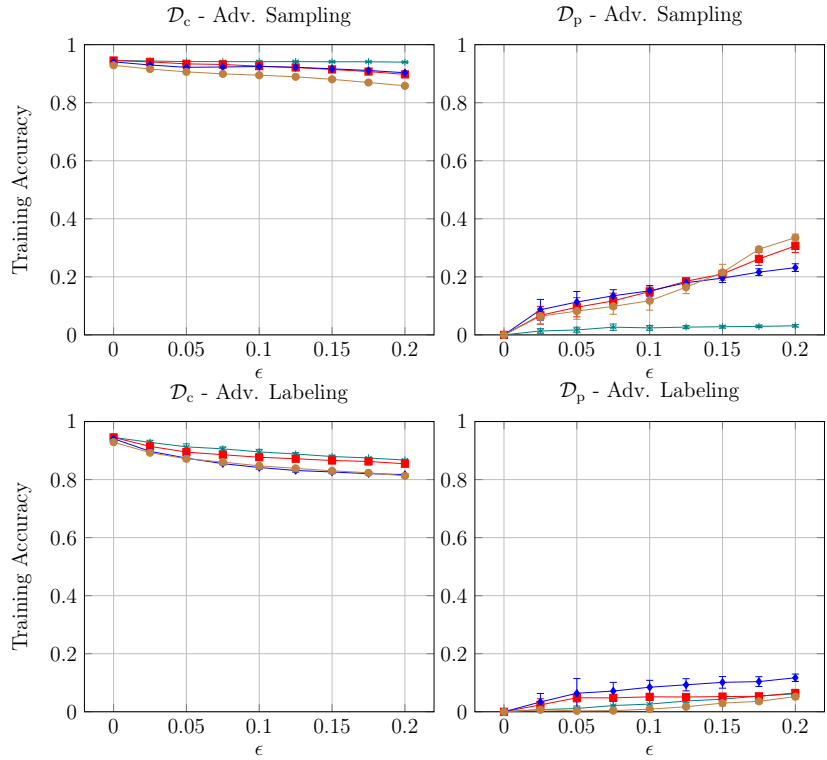

%% file: sections/appdx-exp-fairness-gap.tex

\subsection{Fairness gap of the regualar model on poisoned dataset} \label{appdx:exp-fairness-gap}

To investigate the effectiveness of our attacks, we train an unconstrained classifier without any fairness constraints and measure the fairness gap $\Delta(\theta;\D_\clean \cup \D_\poisoning)$ of the poisoned training dataset generated by different algorithms.
Figure~\ref{fig:EOgap} and Figure~\ref{fig:EOgap-adult-labeling} show the results for COMPAS and Adult dataset respectively.

On the COMPAS dataset, from Figure~\ref{fig:accuracy-compas}, we observe a correlation between attack performance and its fairness gap on the training data. For the baseline attacks (Label flipping, Random sampling, Hard examples), the slight increase in  $\Delta$ corresponds to a small accuracy drop on the test data. For our attacks,  $\Delta$ quickly increases when $\epsilon \le 0.1$, and the corresponding test accuracy also show significant declines. When $\epsilon>0.1$, for Algorithm~\ref{algorithm:OGD} and Algorithm~\ref{algorithm:minmax-fair}  with $\lambda=0$, $\Delta$ stops increasing and the test accuracy begins to level. By contrast, for Algorithm~\ref{algorithm:minmax-fair}  with $\lambda=100\epsilon$, $\Delta$  continues to rise and the attack performance becomes significantly better than all other attacks.

On the Adult dataset, in Figure~\ref{fig:EOgap-adult-labeling} adversarial sampling, the $\Delta$ quickly increases when $\epsilon \le 0.05$ for our attacks, and then decreases as we increase the $\epsilon$. A similar wave can be observed from Figure~\ref{fig:EOgap-adult-labeling} adversarial labeling.
The detailed explanations are presented in the Appendix~\ref{appdx:adult-exp}.

\input{plot/table_eogap.tex}
\input{plot/table_eogap_ADULT_labeling.tex}

%% file: plot/table_eogap.tex

\begin{figure*}[t!]
 \centering
 \resizebox{0.66\columnwidth}{!}{%

     \begin{tikzpicture}

     \begin{axis}
       [name=plot1, legend style={font=\small},title= {Unconstrained Model - Adv. Sampling},xlabel={$\epsilon$},ylabel={Fairness gap on training data}, ymin = 0, ymax =0.75,xtick={0,0.05,0.1,0.15,0.2}, xticklabels={0,0.05,0.1,0.15,0.2}, grid = major,
      ]

		\addplot[solid, blue, mark = *, mark options={solid,}]table[skip first n=1,x index=1,y index=2, col sep=comma] {"plot2/reduction_csv/attacker-nfminmax-0.01-EO-gap.csv"};

		\addplot[violet, mark = o,]table[skip first n=1,x index=1,y index=2, col sep=comma] {"plot2/reduction_csv/attacker-eominmaxlossnf100-0.01-EO-gap.csv"};

          \addplot[teal, mark = x] table[skip first n=1,x index=1, y index=2, col sep=comma] {"plot2/reduction_csv/attacker-newregnf-1.0-0.01-EO-gap.csv"};

        \addplot[solid, gray, mark = triangle, mark options={solid,}] table[skip first n=1,x index=1, y index=2, col sep=comma] {"plot2/reduction_csv/attacker-usel-0.01-EO-gap.csv"};

          \addplot[solid, brown, mark = diamond, mark options={solid,}] table[skip first n=1,x index=1, y index=2, col sep=comma] {"plot2/reduction_csv/noise-usel-0.01-EO-gap.csv"};
     \end{axis}

      \begin{axis}
       [name=plot2, , at=(plot1.south east), anchor=south west, xshift=0.8cm, title= {Unconstrained Model - Adv. Labeling},xlabel={$\epsilon$},ylabel={}, ymin = 0, ymax =0.75,xtick={0,0.05,0.1,0.15,0.2}, xticklabels={0,0.05,0.1,0.15,0.2}, grid = major, legend entries = {
       Alg.~\ref{algorithm:minmax-fair} ($\lambda = 0$) \cite{steinhardt2017certified},
       Adg.~\ref{algorithm:minmax-fair} ($\lambda = 100\epsilon$),Alg.~\ref{algorithm:OGD} ($\lambda = \epsilon$),
      Random sampling, Hard examples, Label flipping,
       }, legend style={at={(0, -0.15)}, anchor=north, draw=none, legend columns=4}]


		\addplot[solid, blue, mark = *, mark options={solid,}]table[skip first n=1,x index=1,y index=2, col sep=comma] {"plot2/reduction_csv/attacker-minmax2-0.01-EO-gap.csv"};

		\addplot[violet, mark = o,]table[skip first n=1,x index=1,y index=2, col sep=comma] {"plot2/reduction_csv/attacker-eominmaxloss-0.01-EO-gap.csv"};

          \addplot[teal, mark = x] table[skip first n=1,x index=1, y index=2, col sep=comma] {"plot2/reduction_csv/attacker-newreg-1.0-0.01-EO-gap.csv"};

        \addplot[solid, gray, mark = triangle, mark options={solid,}] table[skip first n=1,x index=1, y index=2, col sep=comma] {"plot2/reduction_csv/attacker-usel-0.01-EO-gap.csv"};

          \addplot[solid, brown, mark = diamond, mark options={solid,}] table[skip first n=1,x index=1, y index=2, col sep=comma] {"plot2/reduction_csv/noise-usel-0.01-EO-gap.csv"};
          \addplot[solid, red, mark = square, mark options={solid,}] table[skip first n=1,x index=1, y index=2, col sep=comma] {"plot2/reduction_csv/attacker-uflip-0.01-EO-gap.csv"};
     \end{axis}

     \end{tikzpicture}
}
\caption{Fairness gap on the unconstrained model with respect to the training data -- COMPAS dataset. An unconstrained model is learned on the training data that includes poisoning data generated by Alg.~\ref{algorithm:minmax-fair} ($\lambda = 100\epsilon$). The fairness gap $\Delta$ is defined in \eqref{eq:eo-fairness}. The numbers reflect how unfair this unconstrained model is with respect to the protected group on the training  data.
}

\label{fig:EOgap}
\end{figure*}
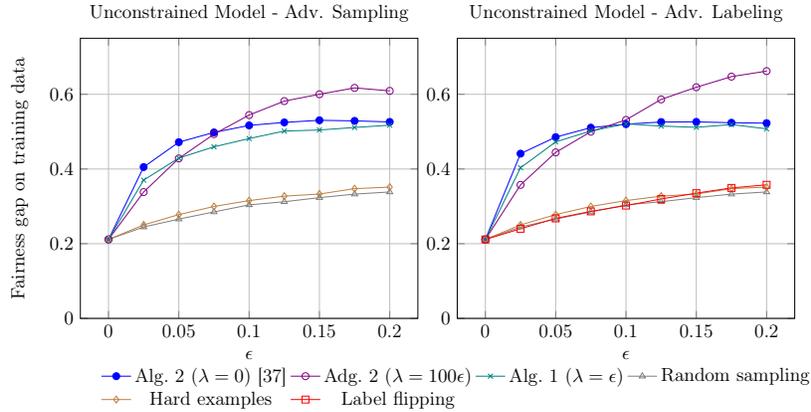

%% file: plot/table_eogap_ADULT_labeling.tex

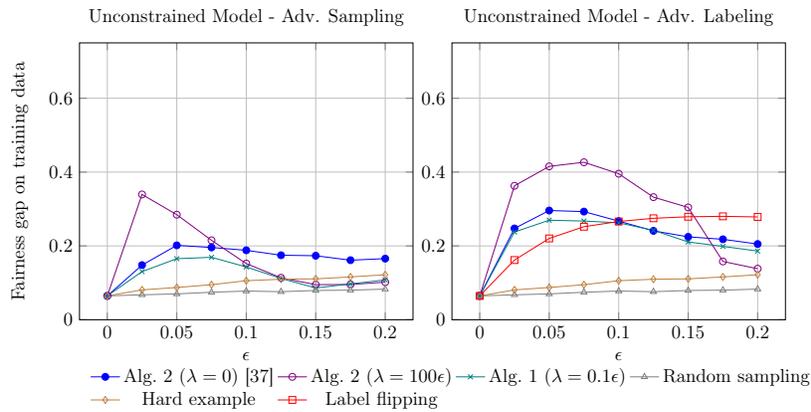
\begin{figure*}[t!]
 \centering
 \resizebox{0.66\columnwidth}{!}{%
     \begin{tikzpicture}
     \begin{axis}
       [name=plot1, title= {Unconstrained Model - Adv. Sampling},xlabel={$\epsilon$},ylabel={Fairness gap on training data}, ymin = 0, ymax =0.75,xtick={0,0.05,0.1,0.15,0.2}, xticklabels={0,0.05,0.1,0.15,0.2}, grid = major]

		\addplot[solid, blue, mark = *, mark options={solid,}]table[skip first n=1,x index=1,y index=2, col sep=comma] {"plot2/adults_reduction_csv/attacker-nfminmax-0.01-EO-gap.csv"};

		\addplot[violet, mark = o,]table[skip first n=1,x index=1,y index=2, col sep=comma] {"plot2/adults_reduction_csv/attacker-eominmaxlossnf-0.01-EO-gap.csv"};

          \addplot[teal, mark = x] table[skip first n=1,x index=1, y index=2, col sep=comma] {"plot2/adults_reduction_csv/attacker-newregnf-0.1-0.01-EO-gap.csv"};
        \addplot[solid, gray, mark = triangle, mark options={solid,}] table[skip first n=1,x index=1, y index=2, col sep=comma] {"plot2/adults_reduction_csv/attacker-usel-0.01-EO-gap.csv"};

          \addplot[solid, brown, mark = diamond, mark options={solid,}] table[skip first n=1,x index=1, y index=2, col sep=comma] {"plot2/adults_reduction_csv/noise-usel-0.01-EO-gap.csv"};

     \end{axis}

     \begin{axis}
       [name=plot2, , at=(plot1.south east), anchor=south west, xshift=0.8cm, title= {Unconstrained Model - Adv. Labeling},xlabel={$\epsilon$},ylabel={}, ymin = 0, ymax =0.75,xtick={0,0.05,0.1,0.15,0.2}, xticklabels={0,0.05,0.1,0.15,0.2}, grid = major, legend entries = {
       Alg.~\ref{algorithm:minmax-fair} ($\lambda = 0$) \cite{steinhardt2017certified},
       Alg.~\ref{algorithm:minmax-fair} ($\lambda = 100\epsilon$), Alg.~\ref{algorithm:OGD}  ($\lambda =0.1 \epsilon$),
        Random sampling, Hard example,Label flipping
       }, legend style={at={(0, -0.15)}, anchor=north, draw=none, legend columns=4}]


		\addplot[solid, blue, mark = *, mark options={solid,}]table[skip first n=1,x index=1,y index=2, col sep=comma] {"plot2/adults_reduction_csv/attacker-minmax2-0.01-EO-gap.csv"};

		\addplot[violet, mark = o,]table[skip first n=1,x index=1,y index=2, col sep=comma] {"plot2/adults_reduction_csv/attacker-eominmaxloss-0.01-EO-gap.csv"};

          \addplot[teal, mark = x] table[skip first n=1,x index=1, y index=2, col sep=comma] {"plot2/adults_reduction_csv/attacker-newreg-0.1-0.01-EO-gap.csv"};

        \addplot[solid, gray, mark = triangle, mark options={solid,}] table[skip first n=1,x index=1, y index=2, col sep=comma] {"plot2/adults_reduction_csv/attacker-usel-0.01-EO-gap.csv"};

          \addplot[solid, brown, mark = diamond, mark options={solid,}] table[skip first n=1,x index=1, y index=2, col sep=comma] {"plot2/adults_reduction_csv/noise-usel-0.01-EO-gap.csv"};
     \addplot[solid, red, mark = square, mark options={solid,}] table[skip first n=1,x index=1, y index=2, col sep=comma] {"plot2/adults_reduction_csv/attacker-uflip-0.01-EO-gap.csv"};
     \end{axis}
     \end{tikzpicture}
}
\caption{Fairness gap on the unconstrained model with respect to the training data -- Adult dataset. An unconstrained model is learned on the training data that includes poisoning data generated by Alg.~\ref{algorithm:minmax-fair} ($\lambda = 100\epsilon$). The fairness gap $\Delta$ is defined in \eqref{eq:eo-fairness}. The numbers reflect how unfair this unconstrained model is with respect to the protected group on the training  data.
}
\label{fig:EOgap-adult-labeling}
\end{figure*}

%% file: sections/appdx-exp-subgroup.tex

\subsection{Distribution of the poisoning data} \label{appdx:exp-subgroups}
In Figure~\ref{fig:subgroups} and Figure~\ref{fig:subgroups-adult}, we show group membership based on the protected attribute and labels of the data which are generated via different attack strategies on COMPAS and Adult dataset respectively.

Note that, for COMPAS dataset, we use the race as the protected attribute ($s = 0$ represents ``White'' and $s=1$ represents ``Black''). The number of samples with $s=1,y=+$ is the smallest among the four combinations of labels and the protected attribute. As shown in Figure~\ref{fig:accuracy-compas}, the attack algorithms in the first two rows are more effective compared with baselines in the second row. As shown in sub-figures (a)-(f), in more effective attacks,  most poisoning data points are from the smallest subgroup (positive labeled points from the minority).

On Adult dataset, we observe the similar phenomena, and we use the gender as the protected attribute where $s = 0$ represents ``Male'' and $s=1$ represents ``Female''. It is also important to pay attention to the fact that on this dataset, the number of samples with $s=1$ is relatively small (33.2\%) and those with $s=1,y=+$ only account for 3.6\% of the dataset. Due to this, finding influential data points from this subgroup is not always possible. Instead, as shown in Figure~\ref{fig:subgroups-adult}, our attacks mainly select data with $ y = +$. Algorithm~\ref{algorithm:minmax-fair} with $\lambda=100\epsilon$ finds more point with $s=1,y=+$  and as shown in Figure~\ref{fig:labeling-attacks-adult} has a marginally better performance.

\input{plot/subgroups.tex}
\input{plot/subgroups-adults.tex}

%% file: plot/subgroups.tex

\begin{figure*}
 \centering
 \resizebox{\columnwidth}{!}{%
\large{
     \begin{tikzpicture}

       \begin{axis}
         [name=plot1, legend style={font=\small},title={(a) Adv. sampling~(Alg.~\ref{algorithm:minmax-fair}, $\lambda = 0$) \cite{steinhardt2017certified}},xlabel={$\epsilon$}, ymin = 0, ymax =100, xtick={0,0.05,0.1,0.15,0.2}, xticklabels={0,0.05,0.1,0.15,0.2},ylabel={Percentage}, stack plots=y, area style, legend entries = {\text{$s= 0, y = -$},\text{$s= 0, y = +$}, \text{$s= 1, y = -$}, \text{$s= 1, y = +$}}, legend pos=south west]
         \addplot table[skip first n=1,x index=1, y index=2, col sep=comma] {"plot2/subgroups/nfminmax-subgroups.csv"}
         \closedcycle;

           \addplot table[skip first n=1,x index=1, y index=3, col sep=comma] {"plot2/subgroups/nfminmax-subgroups.csv"}
           \closedcycle;

           \addplot table[skip first n=1,x index=1, y index=4, col sep=comma] {"plot2/subgroups/nfminmax-subgroups.csv"}
           \closedcycle;

           \addplot table[skip first n=1,x index=1, y index=5, col sep=comma] {"plot2/subgroups/nfminmax-subgroups.csv"}
           \closedcycle;

       \end{axis}

     \begin{axis}
       [name=plot2, at=(plot1.south east), anchor=south west, xshift=0.8cm, legend style={font=\small},title= {(b) Adv. sampling~(Alg.~\ref{algorithm:minmax-fair}, $\lambda = 100\epsilon$)},xlabel={$\epsilon$}, ymin = 0, ymax =100, xtick={0,0.05,0.1,0.15,0.2}, xticklabels={0,0.05,0.1,0.15,0.2}, stack plots=y, area style,
       ]
       \addplot table[skip first n=1,x index=1, y index=2, col sep=comma] {"plot2/subgroups/eominmaxlossnf100-subgroups.csv"}
       \closedcycle;

         \addplot table[skip first n=1,x index=1, y index=3, col sep=comma] {"plot2/subgroups/eominmaxlossnf100-subgroups.csv"}
         \closedcycle;

         \addplot table[skip first n=1,x index=1, y index=4, col sep=comma] {"plot2/subgroups/eominmaxlossnf100-subgroups.csv"}
         \closedcycle;

         \addplot table[skip first n=1,x index=1, y index=5, col sep=comma] {"plot2/subgroups/eominmaxlossnf100-subgroups.csv"}
         \closedcycle;

     \end{axis}

   \begin{axis}
     [name=plot3, at=(plot2.south east), anchor=south west, xshift=0.8cm,title= {(c) Adv. sampling~(Alg.~\ref{algorithm:OGD}, $\lambda = \epsilon$)},xlabel={$\epsilon$},  ymin = 0, ymax =100, xtick={0,0.05,0.1,0.15,0.2}, xticklabels={0,0.05,0.1,0.15,0.2}, stack plots=y, area style
     ]
     \addplot table[skip first n=1,x index=1, y index=2, col sep=comma] {"plot2/subgroups/newregnf-1.0-subgroups.csv"}
     \closedcycle;

       \addplot table[skip first n=1,x index=1, y index=3, col sep=comma] {"plot2/subgroups/newregnf-1.0-subgroups.csv"}
       \closedcycle;

       \addplot table[skip first n=1,x index=1, y index=4, col sep=comma] {"plot2/subgroups/newregnf-1.0-subgroups.csv"}
       \closedcycle;

       \addplot table[skip first n=1,x index=1, y index=5, col sep=comma] {"plot2/subgroups/newregnf-1.0-subgroups.csv"}
       \closedcycle;

   \end{axis}

     \begin{axis}
       [name=plot4, at=(plot1.south west), anchor=north west, yshift=-1.7cm,title={(d) Adv. labeling~(Alg.~\ref{algorithm:minmax-fair}, $\lambda = 0$) \cite{steinhardt2017certified}},xlabel={$\epsilon$}, ylabel={Percentage},  ymin = 0, ymax =100, xtick={0,0.05,0.1,0.15,0.2}, xticklabels={0,0.05,0.1,0.15,0.2},, stack plots=y, area style]
       \addplot table[skip first n=1,x index=1, y index=2, col sep=comma] {"plot2/subgroups/minmax2-subgroups.csv"}
       \closedcycle;

         \addplot table[skip first n=1,x index=1, y index=3, col sep=comma] {"plot2/subgroups/minmax2-subgroups.csv"}
         \closedcycle;

         \addplot table[skip first n=1,x index=1, y index=4, col sep=comma] {"plot2/subgroups/minmax2-subgroups.csv"}
         \closedcycle;

         \addplot table[skip first n=1,x index=1, y index=5, col sep=comma] {"plot2/subgroups/minmax2-subgroups.csv"}
         \closedcycle;

     \end{axis}

     \begin{axis}
       [name=plot5, at=(plot4.south east), anchor=south west, xshift=0.8cm, legend style={font=\small},title= {(e) Adv. labeling~(Alg.~\ref{algorithm:minmax-fair}, $\lambda = 100\epsilon$)},xlabel={$\epsilon$}, ymin = 0, ymax =100, xtick={0,0.05,0.1,0.15,0.2}, xticklabels={0,0.05,0.1,0.15,0.2}, stack plots=y, area style,
       ]
       \addplot table[skip first n=1,x index=1, y index=2, col sep=comma] {"plot2/subgroups/eominmaxloss-subgroups.csv"}
       \closedcycle;

         \addplot table[skip first n=1,x index=1, y index=3, col sep=comma] {"plot2/subgroups/eominmaxloss-subgroups.csv"}
         \closedcycle;

         \addplot table[skip first n=1,x index=1, y index=4, col sep=comma] {"plot2/subgroups/eominmaxloss-subgroups.csv"}
         \closedcycle;

         \addplot table[skip first n=1,x index=1, y index=5, col sep=comma] {"plot2/subgroups/eominmaxloss-subgroups.csv"}
         \closedcycle;

     \end{axis}

     \begin{axis}
           [name=plot6, at=(plot5.south east), anchor=south west, xshift=0.8cm, legend style={font=\small},title= {(f) Adv. labeling~(Alg.~\ref{algorithm:OGD}, $\lambda = \epsilon$)},xlabel={$\epsilon$},ymin = 0, ymax =100, xtick={0,0.05,0.1,0.15,0.2}, xticklabels={0,0.05,0.1,0.15,0.2}, stack plots=y, area style
           ]
           \addplot table[skip first n=1,x index=1, y index=2, col sep=comma] {"plot2/subgroups/newreg-1.0-subgroups.csv"}
           \closedcycle;

             \addplot table[skip first n=1,x index=1, y index=3, col sep=comma] {"plot2/subgroups/newreg-1.0-subgroups.csv"}
             \closedcycle;

             \addplot table[skip first n=1,x index=1, y index=4, col sep=comma] {"plot2/subgroups/newreg-1.0-subgroups.csv"}
             \closedcycle;

             \addplot table[skip first n=1,x index=1, y index=5, col sep=comma] {"plot2/subgroups/newreg-1.0-subgroups.csv"}
             \closedcycle;

         \end{axis}

     \begin{axis}
       [name=plot7, at=(plot4.south west), anchor=north west, yshift=-1.7cm, legend style={font=\small},title= {(g) Random sampling},ymin = 0, ymax =100, xtick={0,0.05,0.1,0.15,0.2}, ylabel={Percentage},xlabel={$\epsilon$},xticklabels={0,0.05,0.1,0.15,0.2},, stack plots=y, area style]
       \addplot table[skip first n=1,x index=1, y index=2, col sep=comma] {"plot2/subgroups/usel-subgroups.csv"}
       \closedcycle;

         \addplot table[skip first n=1,x index=1, y index=3, col sep=comma] {"plot2/subgroups/usel-subgroups.csv"}
         \closedcycle;

         \addplot table[skip first n=1,x index=1, y index=4, col sep=comma] {"plot2/subgroups/usel-subgroups.csv"}
         \closedcycle;

         \addplot table[skip first n=1,x index=1, y index=5, col sep=comma] {"plot2/subgroups/usel-subgroups.csv"}
         \closedcycle;

     \end{axis}

     \begin{axis}
       [name=plot8, at=(plot7.south east), anchor=south west, xshift=0.8cm,legend style={font=\small},title= {(h) Hard examples}, ymin = 0, ymax =100, xtick={0,0.05,0.1,0.15,0.2}, ,xlabel={$\epsilon$}, xticklabels={0,0.05,0.1,0.15,0.2},, stack plots=y, area style,
       ]
       \addplot table[skip first n=1,x index=1, y index=2, col sep=comma] {"plot2/subgroups/usel-noise-subgroups.csv"}
       \closedcycle;

         \addplot table[skip first n=1,x index=1, y index=3, col sep=comma] {"plot2/subgroups/usel-noise-subgroups.csv"}
         \closedcycle;

         \addplot table[skip first n=1,x index=1, y index=4, col sep=comma] {"plot2/subgroups/usel-noise-subgroups.csv"}
         \closedcycle;

         \addplot table[skip first n=1,x index=1, y index=5, col sep=comma] {"plot2/subgroups/usel-noise-subgroups.csv"}
         \closedcycle;

     \end{axis}

     \begin{axis}
       [name=plot9, at=(plot8.south east), anchor=south west, xshift=0.8cm, legend style={font=\small},title= {(i) Label flipping},  ymin = 0, ymax =100, xtick={0,0.05,0.1,0.15,0.2}, ,xlabel={$\epsilon$}, xticklabels={0,0.05,0.1,0.15,0.2},, stack plots=y, area style,
       ]
       \addplot table[skip first n=1,x index=1, y index=2, col sep=comma] {"plot2/subgroups/uflip-subgroups.csv"}
       \closedcycle;

         \addplot table[skip first n=1,x index=1, y index=3, col sep=comma] {"plot2/subgroups/uflip-subgroups.csv"}
         \closedcycle;

         \addplot table[skip first n=1,x index=1, y index=4, col sep=comma] {"plot2/subgroups/uflip-subgroups.csv"}
         \closedcycle;

         \addplot table[skip first n=1,x index=1, y index=5, col sep=comma] {"plot2/subgroups/uflip-subgroups.csv"}
         \closedcycle;

     \end{axis}

     \end{tikzpicture}
     }

}
\caption{Distribution of the poisoning data under adversarial attacks -- COMPAS dataset. We report the protected attribute ($s=0$ for whites and $s=1$ for blacks) and labels of the poisoning data for various $\epsilon$. For every value of $\epsilon$, the number for each combination of the protected attribute and label reflects the percentage of points with this combination in the poisoning data.
}
\label{fig:subgroups}
\end{figure*}
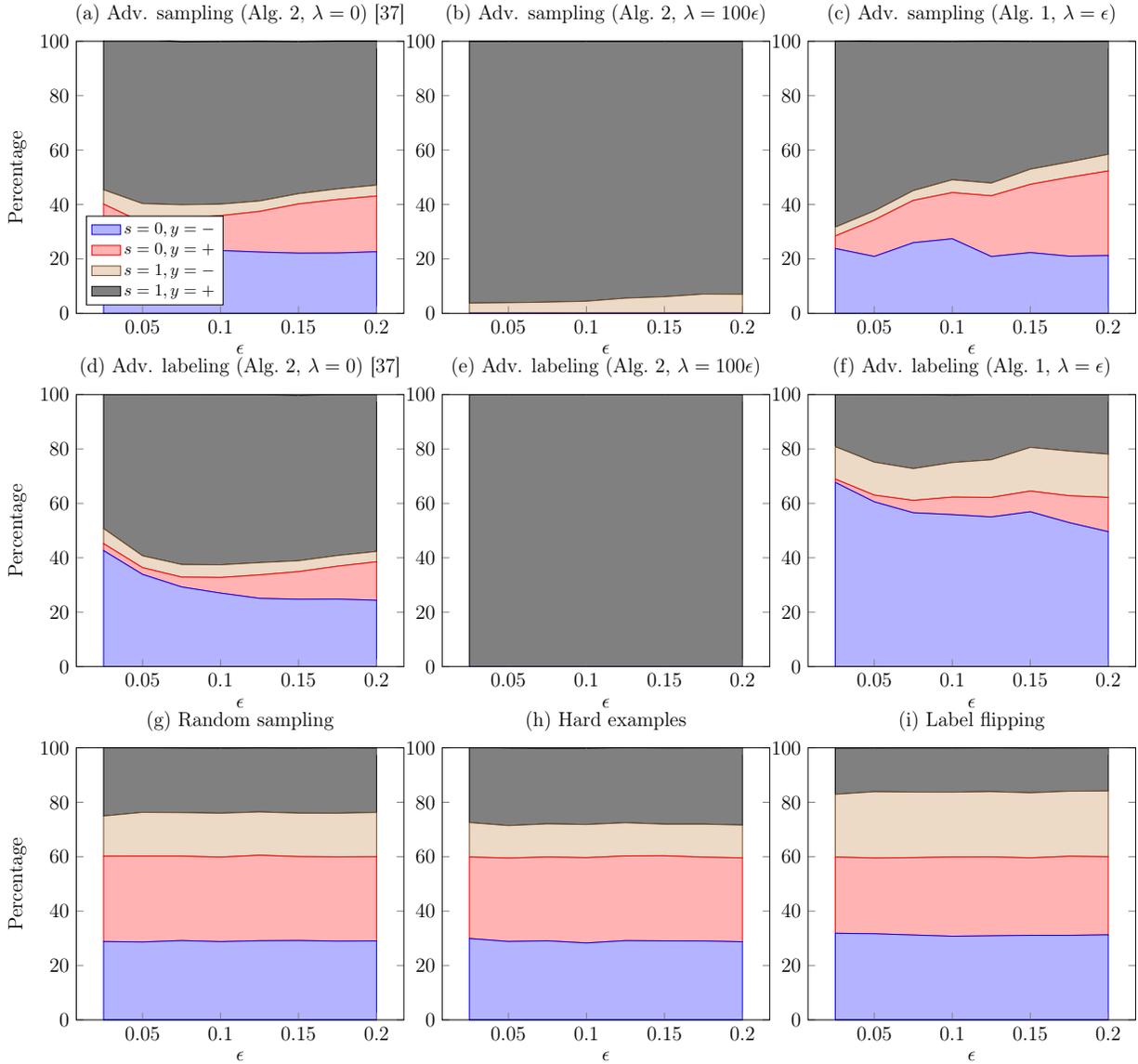

%% file: plot/subgroups-adults.tex

\begin{figure*}[ht]
 \centering
 \resizebox{\columnwidth}{!}{%
\large{
     \begin{tikzpicture}
       \begin{axis}
         [name=plot1,title= {(a) Adv. sampling~(Alg.~\ref{algorithm:minmax-fair}, $\lambda = 0$) \cite{steinhardt2017certified}},xlabel={$\epsilon$}, ymin = 0, ymax =100, xtick={0,0.05,0.1,0.15,0.2}, ylabel={Percentage}, xticklabels={0,0.05,0.1,0.15,0.2},, stack plots=y, area style,legend entries = {\text{$s= 0, y = -$},\text{$s= 0, y = +$}, \text{$s= 1, y = -$}, \text{$s= 1, y = +$}}, legend pos=south west]
         \addplot table[skip first n=1,x index=1, y index=2, col sep=comma] {"plot2/subgroups_adults/nfminmax-subgroups-adults.csv"}
         \closedcycle;

           \addplot table[skip first n=1,x index=1, y index=3, col sep=comma] {"plot2/subgroups_adults/nfminmax-subgroups-adults.csv"}
           \closedcycle;

           \addplot table[skip first n=1,x index=1, y index=4, col sep=comma] {"plot2/subgroups_adults/nfminmax-subgroups-adults.csv"}
           \closedcycle;

           \addplot table[skip first n=1,x index=1, y index=5, col sep=comma] {"plot2/subgroups_adults/nfminmax-subgroups-adults.csv"}
           \closedcycle;

       \end{axis}

     \begin{axis}
       [name=plot2, at=(plot1.south east), anchor=south west, xshift=0.8cm, legend style={font=\small},title= {(b) Adv. sampling~(Alg.~\ref{algorithm:minmax-fair}, $\lambda = 100\epsilon$)},xlabel={$\epsilon$}, ymin = 0, ymax =100, xtick={0,0.05,0.1,0.15,0.2}, xticklabels={0,0.05,0.1,0.15,0.2}, stack plots=y, area style,
       ]
       \addplot table[skip first n=1,x index=1, y index=2, col sep=comma] {"plot2/subgroups_adults/eominmaxlossnf-subgroups-adults.csv"}
       \closedcycle;

         \addplot table[skip first n=1,x index=1, y index=3, col sep=comma] {"plot2/subgroups_adults/eominmaxlossnf-subgroups-adults.csv"}
         \closedcycle;

         \addplot table[skip first n=1,x index=1, y index=4, col sep=comma] {"plot2/subgroups_adults/eominmaxlossnf-subgroups-adults.csv"}
         \closedcycle;

         \addplot table[skip first n=1,x index=1, y index=5, col sep=comma] {"plot2/subgroups_adults/eominmaxlossnf-subgroups-adults.csv"}
         \closedcycle;

     \end{axis}

     \begin{axis}
       [name=plot3, at=(plot2.south east), anchor=south west, xshift=0.8cm, legend style={font=\small},title= {(c) Adv. sampling~(Alg.~\ref{algorithm:OGD}), $\lambda = 0.1 \epsilon$},xlabel={$\epsilon$},  ymin = 0, ymax =100, xtick={0,0.05,0.1,0.15,0.2}, xticklabels={0,0.05,0.1,0.15,0.2},, stack plots=y, area style
       ]
       \addplot table[skip first n=1,x index=1, y index=2, col sep=comma] {"plot2/subgroups_adults/newregnf-0.1-subgroups-adults.csv"}
       \closedcycle;

         \addplot table[skip first n=1,x index=1, y index=3, col sep=comma] {"plot2/subgroups_adults/newregnf-0.1-subgroups-adults.csv"}
         \closedcycle;

         \addplot table[skip first n=1,x index=1, y index=4, col sep=comma] {"plot2/subgroups_adults/newregnf-0.1-subgroups-adults.csv"}
         \closedcycle;

         \addplot table[skip first n=1,x index=1, y index=5, col sep=comma] {"plot2/subgroups_adults/newregnf-0.1-subgroups-adults.csv"}
         \closedcycle;

     \end{axis}

     \begin{axis}
       [name=plot4, at=(plot1.south west), anchor=north west, yshift=-1.7cm,title= {(d) Adv. labeling~(Alg.~\ref{algorithm:minmax-fair}, $\lambda = 0$) \cite{steinhardt2017certified}},xlabel={$\epsilon$}, ymin = 0, ymax =100, xtick={0,0.05,0.1,0.15,0.2}, ylabel={Percentage}, xticklabels={0,0.05,0.1,0.15,0.2},, stack plots=y, area style]
       \addplot table[skip first n=1,x index=1, y index=2, col sep=comma] {"plot2/subgroups_adults/minmax2-subgroups-adults.csv"}
       \closedcycle;

         \addplot table[skip first n=1,x index=1, y index=3, col sep=comma] {"plot2/subgroups_adults/minmax2-subgroups-adults.csv"}
         \closedcycle;

         \addplot table[skip first n=1,x index=1, y index=4, col sep=comma] {"plot2/subgroups_adults/minmax2-subgroups-adults.csv"}
         \closedcycle;

         \addplot table[skip first n=1,x index=1, y index=5, col sep=comma] {"plot2/subgroups_adults/minmax2-subgroups-adults.csv"}
         \closedcycle;

     \end{axis}

     \begin{axis}
       [name=plot5, at=(plot4.south east), anchor=south west, xshift=0.8cm, legend style={font=\small},title= {(e) Adv. labeling~(Alg.~\ref{algorithm:minmax-fair}, $\lambda = 100\epsilon$)},xlabel={$\epsilon$}, ymin = 0, ymax =100, xtick={0,0.05,0.1,0.15,0.2}, xticklabels={0,0.05,0.1,0.15,0.2}, stack plots=y, area style,
       ]
       \addplot table[skip first n=1,x index=1, y index=2, col sep=comma] {"plot2/subgroups_adults/eominmaxloss-subgroups-adults.csv"}
       \closedcycle;

         \addplot table[skip first n=1,x index=1, y index=3, col sep=comma] {"plot2/subgroups_adults/eominmaxloss-subgroups-adults.csv"}
         \closedcycle;

         \addplot table[skip first n=1,x index=1, y index=4, col sep=comma] {"plot2/subgroups_adults/eominmaxloss-subgroups-adults.csv"}
         \closedcycle;

         \addplot table[skip first n=1,x index=1, y index=5, col sep=comma] {"plot2/subgroups_adults/eominmaxloss-subgroups-adults.csv"}
         \closedcycle;

     \end{axis}

     \begin{axis}
      [name=plot6, at=(plot5.south east), anchor=south west, xshift=0.8cm, legend style={font=\small},title= {(f) Adv. labeling~(Alg.~\ref{algorithm:OGD}), $\lambda = 0.1 \epsilon$},xlabel={$\epsilon$},  ymin = 0, ymax =100, xtick={0,0.05,0.1,0.15,0.2}, xticklabels={0,0.05,0.1,0.15,0.2},, stack plots=y, area style
      ]
      \addplot table[skip first n=1,x index=1, y index=2, col sep=comma] {"plot2/subgroups_adults/newreg-0.1-subgroups-adults.csv"}
      \closedcycle;

        \addplot table[skip first n=1,x index=1, y index=3, col sep=comma] {"plot2/subgroups_adults/newreg-0.1-subgroups-adults.csv"}
        \closedcycle;

        \addplot table[skip first n=1,x index=1, y index=4, col sep=comma] {"plot2/subgroups_adults/newreg-0.1-subgroups-adults.csv"}
        \closedcycle;

        \addplot table[skip first n=1,x index=1, y index=5, col sep=comma] {"plot2/subgroups_adults/newreg-0.1-subgroups-adults.csv"}
        \closedcycle;

    \end{axis}

     \begin{axis}
       [name=plot7, at=(plot4.south west), anchor=north west, yshift=-1.7cm, legend style={font=\small},title= {(g) Random sampling},ymin = 0, ymax =100, xtick={0,0.05,0.1,0.15,0.2},ylabel={Percentage} ,xlabel={$\epsilon$},xticklabels={0,0.05,0.1,0.15,0.2},, stack plots=y, area style]
       \addplot table[skip first n=1,x index=1, y index=2, col sep=comma] {"plot2/subgroups_adults/usel-subgroups-adults.csv"}
       \closedcycle;

         \addplot table[skip first n=1,x index=1, y index=3, col sep=comma] {"plot2/subgroups_adults/usel-subgroups-adults.csv"}
         \closedcycle;

         \addplot table[skip first n=1,x index=1, y index=4, col sep=comma] {"plot2/subgroups_adults/usel-subgroups-adults.csv"}
         \closedcycle;

         \addplot table[skip first n=1,x index=1, y index=5, col sep=comma] {"plot2/subgroups_adults/usel-subgroups-adults.csv"}
         \closedcycle;

     \end{axis}

      \begin{axis}
       [name=plot8, at=(plot7.south east), anchor=south west, xshift=0.8cm,  legend style={font=\small},title= {(h) Hard examples},ymin = 0, ymax =100, xtick={0,0.05,0.1,0.15,0.2}, ,xlabel={$\epsilon$},xticklabels={0,0.05,0.1,0.15,0.2},, stack plots=y, area style]
       \addplot table[skip first n=1,x index=1, y index=2, col sep=comma] {"plot2/subgroups_adults/usel-subgroups-adults.csv"}
       \closedcycle;

         \addplot table[skip first n=1,x index=1, y index=3, col sep=comma] {"plot2/subgroups_adults/usel-noise-subgroups-adults.csv"}
         \closedcycle;

         \addplot table[skip first n=1,x index=1, y index=4, col sep=comma] {"plot2/subgroups_adults/usel-noise-subgroups-adults.csv"}
         \closedcycle;

         \addplot table[skip first n=1,x index=1, y index=5, col sep=comma] {"plot2/subgroups_adults/usel-noise-subgroups-adults.csv"}
         \closedcycle;

     \end{axis}
     \begin{axis}
       [name=plot9, at=(plot8.south east), anchor=south west, xshift=0.8cm, legend style={font=\small},title= {(i) Label flipping}, ymin = 0, ymax =100, xtick={0,0.05,0.1,0.15,0.2}, ,xlabel={$\epsilon$}, xticklabels={0,0.05,0.1,0.15,0.2},, stack plots=y, area style,
       ]
       \addplot table[skip first n=1,x index=1, y index=2, col sep=comma] {"plot2/subgroups_adults/uflip-subgroups-adults.csv"}
       \closedcycle;

         \addplot table[skip first n=1,x index=1, y index=3, col sep=comma] {"plot2/subgroups_adults/uflip-subgroups-adults.csv"}
         \closedcycle;

         \addplot table[skip first n=1,x index=1, y index=4, col sep=comma] {"plot2/subgroups_adults/uflip-subgroups-adults.csv"}
         \closedcycle;

         \addplot table[skip first n=1,x index=1, y index=5, col sep=comma] {"plot2/subgroups_adults/uflip-subgroups-adults.csv"}
         \closedcycle;

     \end{axis}

     \end{tikzpicture}
}}
\caption{\small{Distribution of the poisoning data under adversarial attacks -- Adult dataset. We report the protected attribute ($s=0$ for males and $s=1$ for females) and labels of the poisoning data for various $\epsilon$. For every value of $\epsilon$, the number for each combination of the protected attribute and label reflects the percentage of points with this combination in the poisoning data.
}}
\label{fig:subgroups-adult}
\end{figure*}
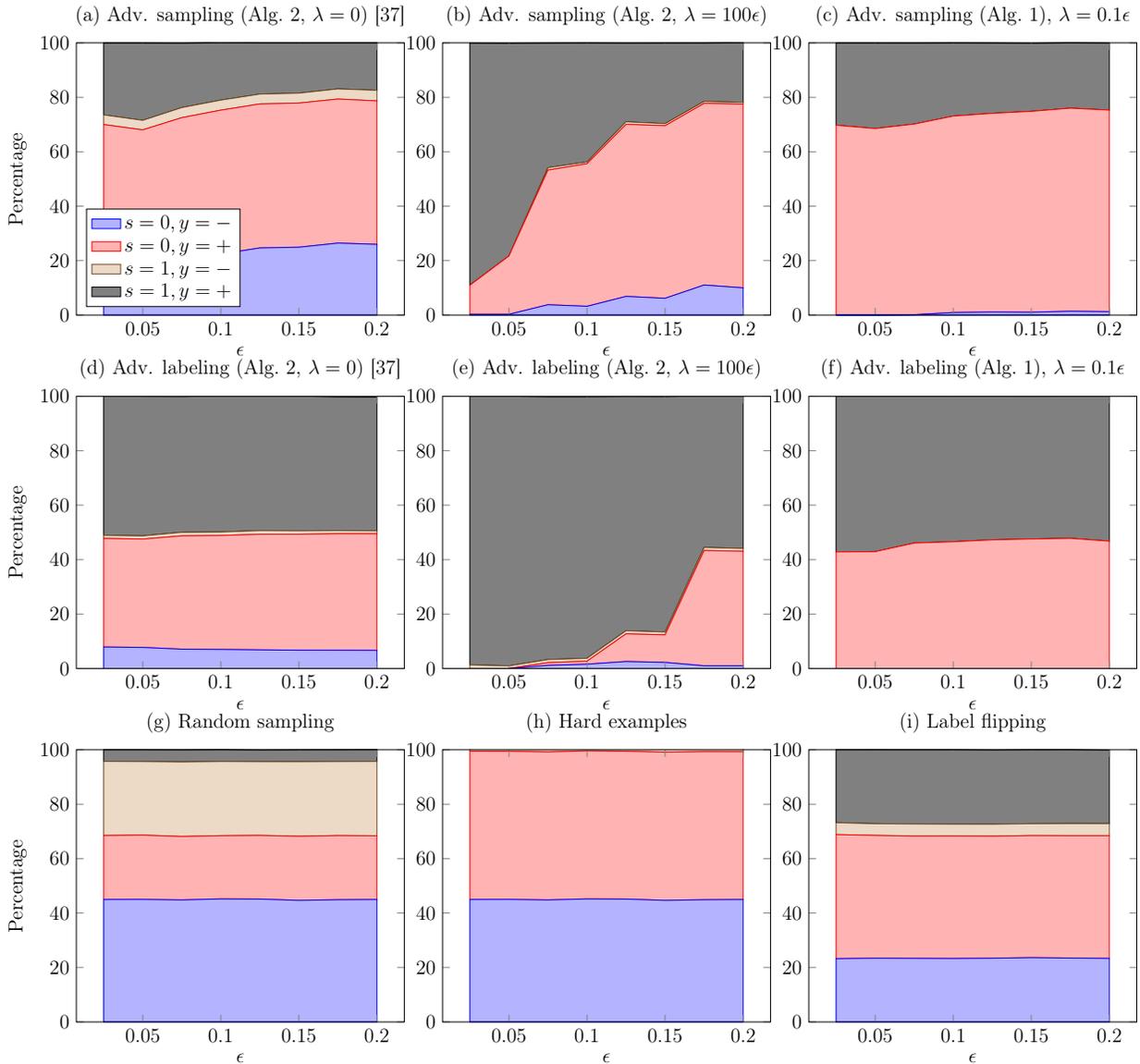